\documentclass[twoside,11pt]{article}

\usepackage{jmlr2e} 
\usepackage{amsmath,epsfig} 
\usepackage{amsfonts}
\usepackage{xcolor}
\usepackage{amssymb}
\usepackage{subfigure}
\usepackage[algo2e]{algorithm2e}
\usepackage{paralist}
\usepackage{booktabs}
\usepackage{subfigure}
\usepackage{color}
\DeclareGraphicsExtensions{.pdf,.jpeg,.png,.epbibdesks}
\usepackage{epstopdf}
\usepackage[T1]{fontenc}
\usepackage{graphicx}
\usepackage{algpseudocode}
\usepackage{url}
\usepackage{multicol}

\usepackage{amssymb}
\usepackage{amsfonts}
\usepackage{mathrsfs}
\usepackage{xspace}
\usepackage{bm}
\usepackage{upgreek}

\newcommand{\safemath}[2]{\newcommand{#1}{\ensuremath{#2}\xspace}}

\safemath{\bma}{\mathbf{a}}
\safemath{\bmb}{\mathbf{b}}
\safemath{\bmc}{\mathbf{c}}
\safemath{\bmd}{\mathbf{d}}
\safemath{\bme}{\mathbf{e}}
\safemath{\bmf}{\mathbf{f}}
\safemath{\bmg}{\mathbf{g}}
\safemath{\bmh}{\mathbf{h}}
\safemath{\bmi}{\mathbf{i}}
\safemath{\bmj}{\mathbf{j}}
\safemath{\bmk}{\mathbf{k}}
\safemath{\bml}{\mathbf{l}}
\safemath{\bmm}{\mathbf{m}}
\safemath{\bmn}{\mathbf{n}}
\safemath{\bmo}{\mathbf{o}}
\safemath{\bmp}{\mathbf{p}}
\safemath{\bmq}{\mathbf{q}}
\safemath{\bmr}{\mathbf{r}}
\safemath{\bms}{\mathbf{s}}
\safemath{\bmt}{\mathbf{t}}
\safemath{\bmu}{\mathbf{u}}
\safemath{\bmv}{\mathbf{v}}
\safemath{\bmw}{\mathbf{w}}
\safemath{\bmx}{\mathbf{x}}
\safemath{\bmy}{\mathbf{y}}
\safemath{\bmz}{\mathbf{z}}
\safemath{\bmzero}{\mathbf{0}}
\safemath{\bmone}{\mathbf{1}}

\bmdefine{\biad}{a}
\bmdefine{\bibd}{b}
\bmdefine{\bicd}{c}
\bmdefine{\bidd}{d}
\bmdefine{\bied}{e}
\bmdefine{\bifd}{f}
\bmdefine{\bigd}{g}
\bmdefine{\bihd}{h}
\bmdefine{\biid}{i}
\bmdefine{\bijd}{j}
\bmdefine{\bikd}{k}
\bmdefine{\bild}{l}
\bmdefine{\bimd}{m}
\bmdefine{\bind}{n}
\bmdefine{\biod}{o}
\bmdefine{\bipd}{p}
\bmdefine{\biqd}{q}
\bmdefine{\bird}{r}
\bmdefine{\bisd}{s}
\bmdefine{\bitd}{t}
\bmdefine{\biud}{u}
\bmdefine{\bivd}{v}
\bmdefine{\biwd}{w}
\bmdefine{\bixd}{x}
\bmdefine{\biyd}{y}
\bmdefine{\bizd}{z}

\bmdefine{\bixid}{\xi}
\bmdefine{\bilambdad}{\lambda}
\bmdefine{\bimud}{\mu}
\bmdefine{\bithetad}{\theta}
\bmdefine{\biphid}{\phi}
\bmdefine{\bideltad}{\delta}

\safemath{\bmia}{\biad}
\safemath{\bmib}{\bibd}
\safemath{\bmic}{\bicd}
\safemath{\bmid}{\bidd}
\safemath{\bmie}{\bied}
\safemath{\bmif}{\bifd}
\safemath{\bmig}{\bigd}
\safemath{\bmih}{\bihd}
\safemath{\bmii}{\biid}
\safemath{\bmij}{\bijd}
\safemath{\bmik}{\bikd}
\safemath{\bmil}{\bild}
\safemath{\bmim}{\bimd}
\safemath{\bmin}{\bind}
\safemath{\bmio}{\biod}
\safemath{\bmip}{\bipd}
\safemath{\bmiq}{\biqd}
\safemath{\bmir}{\bird}
\safemath{\bmis}{\bisd}
\safemath{\bmit}{\bitd}
\safemath{\bmiu}{\biud}
\safemath{\bmiv}{\bivd}
\safemath{\bmiw}{\biwd}
\safemath{\bmix}{\bixd}
\safemath{\bmiy}{\biyd}
\safemath{\bmiz}{\bizd}

\safemath{\bmxi}{\bixid}
\safemath{\bmlambda}{\bilambdad}
\safemath{\bmmu}{\bimud}
\safemath{\bmtheta}{\bithetad}
\safemath{\bmphi}{\biphid}
\safemath{\bmdelta}{\bideltad}

\safemath{\bA}{\mathbf{A}}
\safemath{\bB}{\mathbf{B}}
\safemath{\bC}{\mathbf{C}}
\safemath{\bD}{\mathbf{D}}
\safemath{\bE}{\mathbf{E}}
\safemath{\bF}{\mathbf{F}}
\safemath{\bG}{\mathbf{G}}
\safemath{\bH}{\mathbf{H}}
\safemath{\bI}{\mathbf{I}}
\safemath{\bJ}{\mathbf{J}}
\safemath{\bK}{\mathbf{K}}
\safemath{\bL}{\mathbf{L}}
\safemath{\bM}{\mathbf{M}}
\safemath{\bN}{\mathbf{N}}
\safemath{\bO}{\mathbf{O}}
\safemath{\bP}{\mathbf{P}}
\safemath{\bQ}{\mathbf{Q}}
\safemath{\bR}{\mathbf{R}}
\safemath{\bS}{\mathbf{S}}
\safemath{\bT}{\mathbf{T}}
\safemath{\bU}{\mathbf{U}}
\safemath{\bV}{\mathbf{V}}
\safemath{\bW}{\mathbf{W}}
\safemath{\bX}{\mathbf{X}}
\safemath{\bY}{\mathbf{Y}}
\safemath{\bZ}{\mathbf{Z}}

\safemath{\bZero}{\mathbf{0}}
\safemath{\bOne}{\mathbf{1}}
\safemath{\bDelta}{\mathbf{\Delta}}
\safemath{\bLambda}{\mathbf{\UpLambda}}
\safemath{\bPhi}{\mathbf{\Upphi}}
\safemath{\bSigma}{\mathbf{\Upsigma}}
\safemath{\bOmega}{\mathbf{\Upomega}}
\safemath{\bTheta}{\mathbf{\Uptheta}}

\bmdefine{\biAd}{A}
\bmdefine{\biBd}{B}
\bmdefine{\biCd}{C}
\bmdefine{\biDd}{D}
\bmdefine{\biEd}{E}
\bmdefine{\biFd}{F}
\bmdefine{\biGd}{G}
\bmdefine{\biHd}{H}
\bmdefine{\biId}{I}
\bmdefine{\biJd}{J}
\bmdefine{\biKd}{K}
\bmdefine{\biLd}{L}
\bmdefine{\biMd}{M}
\bmdefine{\biNd}{N}
\bmdefine{\biOd}{O}
\bmdefine{\biPd}{P}
\bmdefine{\biQd}{Q}
\bmdefine{\biRd}{R}
\bmdefine{\biSd}{S}
\bmdefine{\biTd}{T}
\bmdefine{\biUd}{U}
\bmdefine{\biVd}{V}
\bmdefine{\biWd}{W}
\bmdefine{\biXd}{X}
\bmdefine{\biYd}{Y}
\bmdefine{\biZd}{Z}

\bmdefine{\biDelta}{\Delta}
\bmdefine{\biLambda}{\Lambda}
\bmdefine{\biPhi}{\Phi}
\bmdefine{\biSigma}{\Sigma}
\bmdefine{\biOmega}{\Omega}
\bmdefine{\biTheta}{\Theta}

\safemath{\bimA}{\biAd}
\safemath{\bimB}{\biBd}
\safemath{\bimC}{\biCd}
\safemath{\bimD}{\biDd}
\safemath{\bimE}{\biEd}
\safemath{\bimF}{\biFd}
\safemath{\bimG}{\biGd}
\safemath{\bimH}{\biHd}
\safemath{\bimI}{\biId}
\safemath{\bimJ}{\biJd}
\safemath{\bimK}{\biKd}
\safemath{\bimL}{\biLd}
\safemath{\bimM}{\biMd}
\safemath{\bimN}{\biNd}
\safemath{\bimO}{\biOd}
\safemath{\bimP}{\biPd}
\safemath{\bimQ}{\biQd}
\safemath{\bimR}{\biRd}
\safemath{\bimS}{\biSd}
\safemath{\bimT}{\biTd}
\safemath{\bimU}{\biUd}
\safemath{\bimV}{\biVd}
\safemath{\bimW}{\biWd}
\safemath{\bimX}{\biXd}
\safemath{\bimY}{\biYd}
\safemath{\bimZ}{\biZd}

\safemath{\bimDelta}{\biDelta}
\safemath{\bimLambda}{\biLambda}
\safemath{\bimPhi}{\biPhi}
\safemath{\bimSigma}{\biSigma}
\safemath{\bimOmega}{\biOmega}
\safemath{\bimTheta}{\biTheta}

\safemath{\setA}{\mathcal{A}}
\safemath{\setB}{\mathcal{B}}
\safemath{\setC}{\mathcal{C}}
\safemath{\setD}{\mathcal{D}}
\safemath{\setE}{\mathcal{E}}
\safemath{\setF}{\mathcal{F}}
\safemath{\setG}{\mathcal{G}}
\safemath{\setH}{\mathcal{H}}
\safemath{\setI}{\mathcal{I}}
\safemath{\setJ}{\mathcal{J}}
\safemath{\setK}{\mathcal{K}}
\safemath{\setL}{\mathcal{L}}
\safemath{\setM}{\mathcal{M}}
\safemath{\setN}{\mathcal{N}}
\safemath{\setO}{\mathcal{O}}
\safemath{\setP}{\mathcal{P}}
\safemath{\setQ}{\mathcal{Q}}
\safemath{\setR}{\mathcal{R}}
\safemath{\setS}{\mathcal{S}}
\safemath{\setT}{\mathcal{T}}
\safemath{\setU}{\mathcal{U}}
\safemath{\setV}{\mathcal{V}}
\safemath{\setW}{\mathcal{W}}
\safemath{\setX}{\mathcal{X}}
\safemath{\setY}{\mathcal{Y}}
\safemath{\setZ}{\mathcal{Z}}
\safemath{\emptySet}{\varnothing}

\safemath{\colA}{\mathscr{A}}
\safemath{\colB}{\mathscr{B}}
\safemath{\colC}{\mathscr{C}}
\safemath{\colD}{\mathscr{D}}
\safemath{\colE}{\mathscr{E}}
\safemath{\colF}{\mathscr{F}}
\safemath{\colG}{\mathscr{G}}
\safemath{\colH}{\mathscr{H}}
\safemath{\colI}{\mathscr{I}}
\safemath{\colJ}{\mathscr{J}}
\safemath{\colK}{\mathscr{K}}
\safemath{\colL}{\mathscr{L}}
\safemath{\colM}{\mathscr{M}}
\safemath{\colN}{\mathscr{N}}
\safemath{\colO}{\mathscr{O}}
\safemath{\colP}{\mathscr{P}}
\safemath{\colQ}{\mathscr{Q}}
\safemath{\colR}{\mathscr{R}}
\safemath{\colS}{\mathscr{S}}
\safemath{\colT}{\mathscr{T}}
\safemath{\colU}{\mathscr{U}}
\safemath{\colV}{\mathscr{V}}
\safemath{\colW}{\mathscr{W}}
\safemath{\colX}{\mathscr{X}}
\safemath{\colY}{\mathscr{Y}}
\safemath{\colZ}{\mathscr{Z}}

\safemath{\opA}{\mathbb{A}}
\safemath{\opB}{\mathbb{B}}
\safemath{\opC}{\mathbb{C}}
\safemath{\opD}{\mathbb{D}}
\safemath{\opE}{\mathbb{E}}
\safemath{\opF}{\mathbb{F}}
\safemath{\opG}{\mathbb{G}}
\safemath{\opH}{\mathbb{H}}
\safemath{\opI}{\mathbb{I}}
\safemath{\opJ}{\mathbb{J}}
\safemath{\opK}{\mathbb{K}}
\safemath{\opL}{\mathbb{L}}
\safemath{\opM}{\mathbb{M}}
\safemath{\opN}{\mathbb{N}}
\safemath{\opO}{\mathbb{O}}
\safemath{\opP}{\mathbb{P}}
\safemath{\opQ}{\mathbb{Q}}
\safemath{\opR}{\mathbb{R}}
\safemath{\opS}{\mathbb{S}}
\safemath{\opT}{\mathbb{T}}
\safemath{\opU}{\mathbb{U}}
\safemath{\opV}{\mathbb{V}}
\safemath{\opW}{\mathbb{W}}
\safemath{\opX}{\mathbb{X}}
\safemath{\opY}{\mathbb{Y}}
\safemath{\opZ}{\mathbb{Z}}
\safemath{\opZero}{\mathbb{O}}
\safemath{\identityop}{\opI}

\safemath{\veca}{\bma}
\safemath{\vecb}{\bmb}
\safemath{\vecc}{\bmc}
\safemath{\vecd}{\bmd}
\safemath{\vece}{\bme}
\safemath{\vecf}{\bmf}
\safemath{\vecg}{\bmg}
\safemath{\vech}{\bmh}
\safemath{\veci}{\bmi}
\safemath{\vecj}{\bmj}
\safemath{\veck}{\bmk}
\safemath{\vecl}{\bml}
\safemath{\vecm}{\bmm}
\safemath{\vecn}{\bmn}
\safemath{\veco}{\bmo}
\safemath{\vecp}{\bmp}
\safemath{\vecq}{\bmq}
\safemath{\vecr}{\bmr}
\safemath{\vecs}{\bms}
\safemath{\vect}{\bmt}
\safemath{\vecu}{\bmu}
\safemath{\vecv}{\bmv}
\safemath{\vecw}{\bmw}
\safemath{\vecx}{\bmx}
\safemath{\vecy}{\bmy}
\safemath{\vecz}{\bmz}

\safemath{\veczero}{\bmzero}
\safemath{\vecone}{\bmone}
\safemath{\vecxi}{\bmxi}
\safemath{\veclambda}{\bmlambda}
\safemath{\vecmu}{\bmmu}
\safemath{\vectheta}{\bmtheta}
\safemath{\vecphi}{\bmphi}
\safemath{\vecdelta}{\bmdelta}

\safemath{\matA}{\bA}
\safemath{\matB}{\bB}
\safemath{\matC}{\bC}
\safemath{\matD}{\bD}
\safemath{\matE}{\bE}
\safemath{\matF}{\bF}
\safemath{\matG}{\bG}
\safemath{\matH}{\bH}
\safemath{\matI}{\bI}
\safemath{\matJ}{\bJ}
\safemath{\matK}{\bK}
\safemath{\matL}{\bL}
\safemath{\matM}{\bM}
\safemath{\matN}{\bN}
\safemath{\matO}{\bO}
\safemath{\matP}{\bP}
\safemath{\matQ}{\bQ}
\safemath{\matR}{\bR}
\safemath{\matS}{\bS}
\safemath{\matT}{\bT}
\safemath{\matU}{\bU}
\safemath{\matV}{\bV}
\safemath{\matW}{\bW}
\safemath{\matX}{\bX}
\safemath{\matY}{\bY}
\safemath{\matZ}{\bZ}
\safemath{\matzero}{\bmzero}

\safemath{\matDelta}{\bDelta}
\safemath{\matLambda}{\bLambda}
\safemath{\matPhi}{\bPhi}
\safemath{\matSigma}{\bSigma}
\safemath{\matOmega}{\bOmega}
\safemath{\matTheta}{\bTheta}

\safemath{\matidentity}{\matI}
\safemath{\matone}{\matO}

\safemath{\rnda}{A}
\safemath{\rndb}{B}
\safemath{\rndc}{C}
\safemath{\rndd}{D}
\safemath{\rnde}{E}
\safemath{\rndf}{F}
\safemath{\rndg}{G}
\safemath{\rndh}{H}
\safemath{\rndi}{I}
\safemath{\rndj}{J}
\safemath{\rndk}{K}
\safemath{\rndl}{L}
\safemath{\rndm}{M}
\safemath{\rndn}{N}
\safemath{\rndo}{O}
\safemath{\rndp}{P}
\safemath{\rndq}{Q}
\safemath{\rndr}{R}
\safemath{\rnds}{S}
\safemath{\rndt}{T}
\safemath{\rndu}{U}
\safemath{\rndv}{V}
\safemath{\rndw}{W}
\safemath{\rndx}{X}
\safemath{\rndy}{Y}
\safemath{\rndz}{Z}

\safemath{\rveca}{\bimA}
\safemath{\rvecb}{\bimB}
\safemath{\rvecc}{\bimC}
\safemath{\rvecd}{\bimD}
\safemath{\rvece}{\bimE}
\safemath{\rvecf}{\bimF}
\safemath{\rvecg}{\bimG}
\safemath{\rvech}{\bimH}
\safemath{\rveci}{\bimI}
\safemath{\rvecj}{\bimJ}
\safemath{\rveck}{\bimK}
\safemath{\rvecl}{\bimL}
\safemath{\rvecm}{\bimM}
\safemath{\rvecn}{\bimN}
\safemath{\rveco}{\bomO}
\safemath{\rvecp}{\bimP}
\safemath{\rvecq}{\bimQ}
\safemath{\rvecr}{\bimR}
\safemath{\rvecs}{\bimS}
\safemath{\rvect}{\bimT}
\safemath{\rvecu}{\bimU}
\safemath{\rvecv}{\bimV}
\safemath{\rvecw}{\bimW}
\safemath{\rvecx}{\bimX}
\safemath{\rvecy}{\bimY}
\safemath{\rvecz}{\bimZ}

\safemath{\rvecxi}{\bmxi}
\safemath{\rveclambda}{\bmlambda}
\safemath{\rvecmu}{\bmmu}
\safemath{\rvectheta}{\bmtheta}
\safemath{\rvecphi}{\bmphi}

\safemath{\rmatA}{\bimA}
\safemath{\rmatB}{\bimB}
\safemath{\rmatC}{\bimC}
\safemath{\rmatD}{\bimD}
\safemath{\rmatE}{\bimE}
\safemath{\rmatF}{\bimF}
\safemath{\rmatG}{\bimG}
\safemath{\rmatH}{\bimH}
\safemath{\rmatI}{\bimI}
\safemath{\rmatJ}{\bimJ}
\safemath{\rmatK}{\bimK}
\safemath{\rmatL}{\bimL}
\safemath{\rmatM}{\bimM}
\safemath{\rmatN}{\bimN}
\safemath{\rmatO}{\bimO}
\safemath{\rmatP}{\bimP}
\safemath{\rmatQ}{\bimQ}
\safemath{\rmatR}{\bimR}
\safemath{\rmatS}{\bimS}
\safemath{\rmatT}{\bimT}
\safemath{\rmatU}{\bimU}
\safemath{\rmatV}{\bimV}
\safemath{\rmatW}{\bimW}
\safemath{\rmatX}{\bimX}
\safemath{\rmatY}{\bimY}
\safemath{\rmatZ}{\bimZ}

\safemath{\rmatDelta}{\bimDelta}
\safemath{\rmatLambda}{\bimLambda}
\safemath{\rmatPhi}{\bimPhi}
\safemath{\rmatSigma}{\bimSigma}
\safemath{\rmatOmega}{\bimOmega}
\safemath{\rmatTheta}{\bimTheta}

\usepackage{amssymb}
\usepackage{amsfonts}
\usepackage{mathrsfs}
\usepackage{xspace}
\usepackage{bm}
\usepackage{fancyref}
\usepackage{textcomp}

\usepackage{multirow}
\usepackage{stmaryrd}

\newenvironment{textbmatrix}{	\setlength{\arraycolsep}{2.5pt}%
								\big[\begin{matrix}}{\end{matrix}\big]%
								\raisebox{0.08ex}{\vphantom{M}}}

\def\be{\begin{equation}}
\def\ee{\end{equation}}
\def\een{\nonumber \end{equation}}
\def\mat{\begin{bmatrix}}
\def\emat{\end{bmatrix}}
\def\btm{\begin{textbmatrix}}
\def\etm{\end{textbmatrix}}

\def\ba#1\ea{\begin{align}#1\end{align}}
\def\bas#1\eas{\begin{align*}#1\end{align*}}
\def\bs#1\es{\begin{split}#1\end{split}} 
\def\bg#1\eg{\begin{gather}#1\end{gather}}
\def\bml#1\eml{\begin{multline}#1\end{multline}}
\def\bi#1\ei{\begin{itemize}#1\end{itemize}}

\newcommand{\lefto}{\mathopen{}\left}

\DeclareMathOperator{\sign}{sign}			

\newcommand{\abs}[1]{\lefto\lvert#1\right\rvert}		

\newcommand{\vecnorm}[1]{\lefto\lVert#1\right\rVert}		

\safemath{\dirac}{\delta}					
\safemath{\krond}{\dirac}					

\safemath{\upto}{\uparrow}
\safemath{\downto}{\downarrow}
\safemath{\iu}{j}							
\safemath{\ev}{\lambda}						
\safemath{\hilseqspace}{l^{2}}				
\newcommand{\banachfunspace}[1]{\setL^{#1}}	
\safemath{\hilfunspace}{\banachfunspace{2}}	

\safemath{\SNR}{\text{\sc snr}} 				
\safemath{\No}{N_0}							
\safemath{\Es}{E_s}							
\safemath{\Eb}{E_b}							
\safemath{\EbNo}{\frac{\Eb}{\No}}
\safemath{\EsNo}{\frac{\Es}{\No}}

\DeclareMathOperator{\CHop}{\ensuremath{\opH}} 
\safemath{\tvir}{\rndh_{\CHop}}				
\safemath{\tvtf}{\rndl_{\CHop}}				
\safemath{\spf}{\rnds_{\CHop}}				
\safemath{\bff}{H_{\CHop}}					

\safemath{\ircf}{r_{h}}						
\safemath{\tftvcf}{r_{s}}					
\safemath{\tfcf}{r_{l}}						
\safemath{\bfcf}{r_{H}}						

\safemath{\tcorr}{c_h}						
\safemath{\scf}{c_{s}}						
\safemath{\tfcorr}{c_{l}}					
\safemath{\fcorr}{c_{H}}						

\safemath{\mi}{I}							
\safemath{\capacity}{C}						

\safemath{\normal}{\mathcal{N}}			
\safemath{\jpg}{\mathcal{CN}}			
\safemath{\mchain}{\leftrightarrow}		

\safemath{\dB}{\,\mathrm{dB}}
\safemath{\dBm}{\,\mathrm{dBm}}
\safemath{\Hz}{\,\mathrm{Hz}}
\safemath{\kHz}{\,\mathrm{kHz}}
\safemath{\MHz}{\,\mathrm{MHz}}
\safemath{\GHz}{\,\mathrm{GHz}}
\safemath{\s}{\,\mathrm{s}}
\safemath{\ms}{\,\mathrm{ms}}
\safemath{\mus}{\,\mathrm{\text{\textmu}s}}
\safemath{\ns}{\,\mathrm{ns}}
\safemath{\ps}{\,\mathrm{ps}}
\safemath{\meter}{\,\mathrm{m}}
\safemath{\mm}{\,\mathrm{mm}}
\safemath{\cm}{\,\mathrm{cm}}
\safemath{\m}{\,\mathrm{m}}
\safemath{\W}{\,\mathrm{W}}
\safemath{\mW}{\, \mathrm{mW}}
\safemath{\J}{\,\mathrm{J}}
\safemath{\K}{\,\mathrm{K}}
\safemath{\bit}{\,\mathrm{bit}}
\safemath{\nat}{\,\mathrm{nat}}

\safemath{\define}{\triangleq}			

\safemath{\equivalent}{\sim}
\safemath{\distas}{\sim}					
\safemath{\sdiff}{\Delta}				

\safemath{\reals}{\mathbb{R}}
\safemath{\positivereals}{\reals_{+}}
\safemath{\integers}{\mathbb{Z}}
\safemath{\posint}{\integers_{+}}
\safemath{\naturals}{\mathbb{N}}
\safemath{\posnaturals}{\naturals_{+}}
\safemath{\complexset}{\mathbb{C}}
\safemath{\rationals}{\mathbb{Q}}

\newcommand*{\fancyrefapplabelprefix}{app}		
\newcommand*{\fancyrefthmlabelprefix}{thm}		
\newcommand*{\fancyreflemlabelprefix}{lem}		
\newcommand*{\fancyrefcorlabelprefix}{cor}		
\newcommand*{\fancyrefdeflabelprefix}{def}		
\newcommand*{\fancyrefalglabelprefix}{alg}		
\newcommand*{\fancyrefproplabelprefix}{prop}		
\newcommand*{\fancyrefexmpllabelprefix}{exmpl}
\newcommand*{\fancyreftbllabelprefix}{tbl}
\frefformat{vario}{\fancyrefseclabelprefix}{Section~#1}
\frefformat{vario}{\fancyrefthmlabelprefix}{Theorem~#1}
\frefformat{vario}{\fancyreflemlabelprefix}{Lemma~#1}
\frefformat{vario}{\fancyrefcorlabelprefix}{Corrolary~#1}
\frefformat{vario}{\fancyrefdeflabelprefix}{Definition~#1}
\frefformat{vario}{\fancyrefalglabelprefix}{Algorithm~#1}
\frefformat{vario}{\fancyreffiglabelprefix}{Figure~#1}
\frefformat{vario}{\fancyrefapplabelprefix}{Appendix~#1} 
\frefformat{vario}{\fancyrefeqlabelprefix}{(#1)}
\frefformat{vario}{\fancyrefproplabelprefix}{Proposition~#1}
\frefformat{vario}{\fancyrefexmpllabelprefix}{Example~#1}
\frefformat{vario}{\fancyreftbllabelprefix}{Table~#1}

 \newtheorem{thm}{Theorem}

 \newtheorem{lem}[thm]{Lemma}
 
\safemath{\dictab}{[\,\dicta\,\,\dictb\,]}

\safemath{\ysig}{\bmy}
\safemath{\ysighat}{\hat{\ysig}}
\safemath{\ysigdim}{M}
\safemath{\xsig}{\bmx}
\safemath{\xsigdim}{N}
\safemath{\nx}{n_x}
\safemath{\zsig}{\bmz}
\safemath{\zsigdim}{\ysigdim}
\safemath{\rsig}{\bmr}
\safemath{\Adict}{\bA}
\safemath{\Adicttilde}{\widetilde{\Adict}}
\safemath{\Adictdim}{\outputdim\times\xsigdim}
\safemath{\avec}{\bma}
\safemath{\avectilde}{\tilde{\avec}}
\safemath{\Bdict}{\bB}
\safemath{\Bdicttilde}{\widetilde{\Bdict}}
\safemath{\Cdict}{\bC}
\safemath{\cvec}{\bmc}
\safemath{\Ddict}{\bD}
\safemath{\Ddictdim}{\ysigdim\times\xsigdim}
\safemath{\dvec}{\bmd}
\safemath{\Ddicttilde}{\widetilde{\bD}}
\safemath{\Bonb}{\bB}
\safemath{\bvec}{\bmb}
\safemath{\Bonbdim}{\ysigdim\times\ysigdim}
\safemath{\noise}{\bmn}
\safemath{\noisedim}{\ysigim}
\safemath{\err}{\bme}
\safemath{\errdim}{\ysigdim}
\safemath{\errset}{\setE}
\safemath{\nerr}{n_e}
\safemath{\delop}{\bP_\errset}
\safemath{\delopc}{\bP_{{\errset}^c}}

\safemath{\cplxi}{\imath}
\safemath{\cplxj}{\jmath}

\safemath{\dict}{\matD}
\safemath{\inputdim}{N}		
\safemath{\outputdim}{M}		
\safemath{\sparsity}{S}	
\safemath{\inputdimA}{{N_a}}	
\safemath{\inputdimB}{{N_b}}	
\safemath{\elemA}{{n_a}}	
\safemath{\elemB}{{n_b}}	
\safemath{\resA}{\matR_a}	
\safemath{\resB}{\matR_b}	
\safemath{\subD}{\matS} 
\safemath{\subA}{\matS_a} 
\safemath{\subB}{\matS_b} 
\safemath{\dicta}{\matA} 	
\safemath{\dictb}{\matB} 	
\safemath{\hollowS}{H}
\safemath{\hollowA}{H_a}
\safemath{\hollowB}{H_b}
\safemath{\cross}{Z}
\safemath{\coh}{\mu_d}			
\safemath{\coha}{\mu_a}			
\safemath{\cohb}{\mu_b}			
\safemath{\mubs}{\nu}	
\safemath{\cohm}{\mu_m} 
\safemath{\dictset}{\setD}	
\safemath{\dictsetp}{\dictset(\coh,\coha,\cohb)}	
\safemath{\dictsetgen}{\dictset_\text{gen}}
\safemath{\dictsetgenp}{\dictsetgen(\coh)}
\safemath{\dictsetonb}{\dictset_\text{onb}}
\safemath{\dictsetonbp}{\dictsetonb(\coh)}

\safemath{\leftside}{U}
\safemath{\rightsideA}{R_a}
\safemath{\rightsideB}{R_b}

\safemath{\indexS}{\setI_S} 

\safemath{\na}{n_a}			
\safemath{\nb}{n_b}			
\safemath{\coeffa}{p_i}	
\safemath{\coeffb}{q_j}	
\safemath{\seta}{\setP}		
\safemath{\setb}{\setQ}     
\safemath{\setw}{\setW}	
\safemath{\setz}{\setZ}	
\safemath{\cola}{\veca}		
\safemath{\colb}{\vecb}		
\safemath{\cold}{\vecd}		
\safemath{\inputvec}{\vecx} 	
\safemath{\error}{\vece}	
\safemath{\noiseout}{\vecz} 	
\safemath{\inputvecel}{x}
\safemath{\inputveca}{\vecx_a}
\safemath{\inputvecb}{\vecx_b}
\safemath{\outputvec}{\vecy}	
\safemath{\lambdamin}{\lambda_{\mathrm{min}}}

\newcommand{\normtwo}[1]{\vecnorm{#1}_2}
\newcommand{\normone}[1]{\vecnorm{#1}_1}

\newcommand{\normfro}[1]{\vecnorm{#1}_\text{F}}

\safemath{\elltwo}{\ell_2}
\safemath{\ellone}{\ell_1}
\safemath{\ellzero}{\ell_0}
\safemath{\ellinf}{\ell_\infty}
\safemath{\licard}{Z(\coh,\coha,\cohb)}
\safemath{\xsol}{\hat{x}}
\safemath{\xbord}{x_b}		
\safemath{\xstat}{x_s}		
\safemath{\xstatLone}{\tilde{x}_s}
\safemath{\order}{\mathcal{O}} 
\safemath{\scales}{\Theta} 
\safemath{\ones}{\mathbf{1}} 
\safemath{\zeroes}{\mathbf{0}} 
\safemath{\thlone}{\kappa(\coh,\cohb)} 
\safemath{\constoneA}{\delta} 
\safemath{\constoneB}{\epsilon} 
\safemath{\nlarge}{L}				   
\safemath{\sumlarge}{S_\nlarge}
\safemath{\maxlarger}{P_\nlarge}	   
\safemath{\Pzero}{\textrm{P0}}	
\safemath{\Pone}{\textrm{P1}}
\safemath{\vecfir}{\vecw}			 
\safemath{\vecsec}{\vecz}
\safemath{\elvecfir}{w}              
\safemath{\elvecsec}{z}				 
\safemath{\nlargefir}{n}
\safemath{\normout}{\gamma}
\safemath{\auxfun}{h}
\safemath{\supp}{\textrm{supp}}

\safemath{\indexa}{\ell}
\safemath{\indexb}{r}
\safemath{\indexc}{i}
\safemath{\indexd}{j}

\safemath{\project}{P}
\renewcommand{\vecw}{\bar{\mathbf{w}}}

\jmlrheading{}{}{}{10/12}{--}{A.~S. Lan, A.~E.~Waters, C.~Studer, and R.~G. Baraniuk\,---\,first authorship determined by coin toss}

\ShortHeadings{Sparse Factor Analysis for Learning and Content Analytics}{Lan, Waters, Studer, and Baraniuk}
\firstpageno{1}

\begin{document}

\title{Sparse Factor Analysis for Learning and Content Analytics}

\author{\name Andrew\ S.\ Lan \email mr.lan@sparfa.com \\
       \name Andrew E. Waters \email waters@sparfa.com \\
	\name Christoph Studer \email studer@sparfa.com \\
	\name Richard G. Baraniuk \email richb@sparfa.com \\
       \addr Dept.~Electrical and Computer Engineering\\
       Rice University\\
       Houston, TX 77005, USA}

\editor{TBD}

\maketitle

\begin{abstract}
We develop a new model and algorithms for machine learning-based {\em{learning analytics}}, which estimate a learner's knowledge of the concepts underlying a domain, and {\em{content analytics}}, which estimate the relationships among a collection of questions and those concepts.
Our model represents the probability that a learner provides the correct response to a question in terms of three factors: their understanding of a set of underlying concepts, the concepts involved in each question, and each question's intrinsic difficulty.
We estimate these factors given the graded responses to a collection of questions.  
The underlying estimation problem is ill-posed in general, especially when only a subset of the questions are answered. 
The key observation that enables a well-posed solution is the fact that typical educational domains of interest involve only a small number of key concepts.  
Leveraging this observation, we develop both a bi-convex maximum-likelihood-based solution and a Bayesian solution to the resulting {\em SPARse Factor Analysis} (SPARFA) problem. 
We also incorporate user-defined tags on questions to facilitate the interpretability of the estimated factors.
Experiments with synthetic and real-world data demonstrate the efficacy of our approach.  
Finally, we make a connection between SPARFA and noisy, binary-valued~(\mbox{1-bit}) dictionary learning that is of independent interest.
\end{abstract}


\section{Introduction} \label{sec:intro}

Textbooks, lectures, and homework assignments were the answer to the main educational challenges of the 19th century, but they are the main bottleneck of the 21st century. 
Today's textbooks are static, linearly organized, time-consuming to develop, soon out-of-date, and expensive.  
Lectures remain a primarily passive experience of copying down what an instructor says and writes on a board (or projects on a screen).  
Homework assignments that are not graded for weeks provide poor feedback to learners (e.g., students) on their learning progress.  
Even more importantly, today's courses provide only a ``one-size-fits-all'' learning experience that does not cater to the background, interests, and goals of individual learners.  

\subsection{The Promise of Personalized Learning}

We envision a world where access to high-quality, personally tailored educational experiences is affordable to all of the world's learners. 
The key is to integrate textbooks, lectures, and homework assignments into a {\em personalized learning system}~(PLS) that closes the learning feedback loop by 
\begin{inparaenum}[(i)] 
\item continuously monitoring and analyzing learner interactions with learning resources in order to assess their learning progress and 
\item providing timely remediation, enrichment, or practice based on that analysis.
 \end{inparaenum}
See \cite{adaptivetest}, \cite{lookaheaddecision}, \cite{stampersm}, \cite{pomdplearn}, \cite{cohen},  and \cite{knewton} for various visions and examples. 

Some progress has been made over the past few decades on personalized learning; see, for example, the sizable literature on \emph{intelligent tutoring systems} discussed in~\cite{itslesson}. 
To date, the lionshare of fielded, intelligent tutors have been rule-based systems that are hard-coded by domain experts to give learners feedback for pre-defined scenarios (e.g., \cite{rule}, \cite{pittcmu}, \cite{andes2005}, and \cite{butzregina}).  
The specificity of such systems is counterbalanced by their high development cost in terms of both time and money, which has limited their scalability and impact in practice.  

In a fresh direction, recent progress has been made on applying \emph{machine learning} algorithms to mine learner interaction data and educational content (see the overview articles by \cite{edm} and \cite{edm2009}). 
In contrast to rule-based approaches, machine learning-based PLSs promise to be rapid and inexpensive to deploy, which will enhance their scalability and impact.  Indeed, the dawning age of ``big data'' provides new opportunities to build PLSs based on data rather than rules.
We conceptualize the architecture of a generic machine learning-based PLS to have three interlocking components:
\begin{itemize}
\item \emph{Learning analytics}: Algorithms that estimate what each learner does and does not understand based on data obtained from tracking their interactions with learning content. 
\item \emph{Content analytics}: Algorithms that organize learning content such as text, video, simulations, questions, and feedback hints.
\item \emph{Scheduling}: Algorithms that use the results of learning and content analytics to suggest to each learner at each moment what they should be doing in order to maximize their learning outcomes, in effect closing the learning feedback loop.

\end{itemize}

\subsection{Sparse Factor Analysis (SPARFA)}
\label{sec:SPARFASectionref}

In this paper, we develop a new model and a suite of algorithms for joint machine learning-based {\em learning analytics} and  {\em content analytics}.  
Our model (developed in \fref{sec:model}) represents the probability that a learner provides the correct response to a given question in terms of three factors: their knowledge of the underlying concepts, the concepts involved in each question, and each question's intrinsic difficulty.

Figure \ref{fig:graphex} provides a graphical depiction of our approach.
As shown in Figure \ref{fig:graphex}(a), we are provided with data relating to the correctness of the learners' responses to a collection of questions.
We encode these graded responses in a ``gradebook,'' a source of information commonly used in the context of classical test theory (\cite{ctt}).
Specifically, the ``gradebook'' is a matrix with entry $Y_{i,j} = 1$ or 0 depending on whether learner $j$ answers question~$i$ correctly or incorrectly, respectively.
Question marks correspond to incomplete data due to unanswered or unassigned questions. 
Working left-to-right in Figure \ref{fig:graphex}(b), we assume that the collection of questions (rectangles) is related to a small number of abstract concepts (circles) by a bipartite graph, where the edge weight $W_{i,k}$ indicates the degree to which question $i$ involves concept~$k$. 
We also assume that question $i$ has intrinsic difficulty~$\mu_i$.
Denoting learner $j$'s knowledge of concept $k$ by $C_{k,j}$, we calculate the probabilities that  the learners answer the questions correctly in terms of $\bW\bC+\bM$, where $\bW$ and $\bC$ are matrix versions of $W_{i,k}$ and $C_{k,j}$, respectively, and $\bM$ is a matrix containing  the intrinsic question difficulty~$\mu_i$ on row $i$.
We transform the probability of a correct answer to an actual $1/0$ correctness via a standard probit or logit link function (see \mbox{\cite{gpml}}).  

\begin{figure}[t]
\vspace{-1.0cm}
\centering
\subfigure[Graded learner--question responses.]{\includegraphics[scale=1.25]{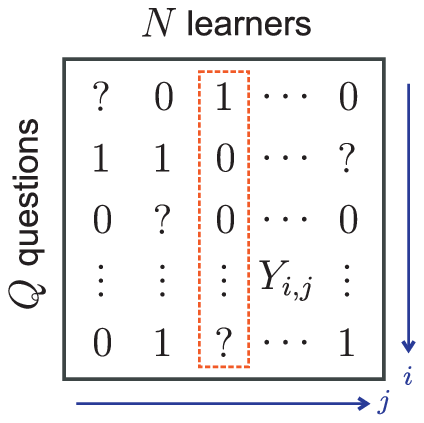}\label{fig:graphex_a}}\hspace{-0.2cm}
\subfigure[Inferred question--concept association graph.]{\includegraphics[scale=1.25]{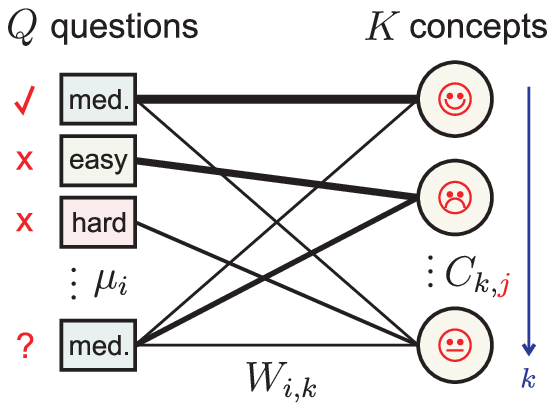}\label{fig:graphex2_b}}
\vspace{-0.2cm}
\caption{(a) The SPARFA framework processes a (potentially incomplete) binary-valued dataset of graded learner--question responses to (b) estimate the underlying questions-concept association graph and the abstract conceptual knowledge of each learner (illustrated here by smiley faces for learner $j=3$, the column in (a) selected by the red dashed box).}
\label{fig:graphex}
\vspace{-0.7cm}
\end{figure}

Armed with this model and given incomplete observations of the graded learner--question responses $Y_{i,j}$, our goal is to estimate the factors~$\bW$,~$\bC$, and $\bM$.
Such a {\em factor-analysis problem} is ill-posed in general, especially when each learner answers only a small subset of the collection of questions (see \cite{factoran} for a factor analysis overview).
Our first key observation that enables a well-posed solution is the fact that typical educational domains of interest involve only a small number of key concepts (i.e., we have $K \ll N,Q$ in \fref{fig:graphex}).  
Consequently, $\bW$ becomes a tall, narrow $Q \times K$ matrix that relates the questions to a small set of abstract concepts, while $\bC$ becomes a short, wide $K \times N$ matrix that relates learner knowledge to that same small set of abstract concepts.  
Note that the concepts are ``abstract'' in that they will be estimated from the data rather than dictated by a subject matter expert. 
Our second key observation is that each question involves only a small subset of the abstract concepts.
Consequently, the matrix $\bW$ is sparsely populated. 
Our third observation is that the entries of $\bW$ should be non-negative, since we postulate that having strong concept knowledge should never hurt a learner's chances to answer questions correctly. This constraint on $\bW$ ensures that large positive values in~$\bC$ represent strong knowledge of the associated abstract concepts, which is crucial for a PLS to generate human-interpretable feedback to learners on their strengths and weaknesses.

Leveraging these observations, we propose below a suite of new algorithms for solving the {\em SPARse Factor Analysis} (SPARFA) problem.
Section \ref{sec:matrix} develops SPARFA-M, which uses an efficient bi-convex optimization approach to produce point estimates of the factors.
Section \ref{sec:bayes} develops SPARFA-B, which uses Bayesian factor analysis to produce posterior distributions of the factors.  
Since the concepts are abstract mathematical quantities estimated by the SPARFA algorithms, we develop a  \emph{post-processing} step in \fref{sec:taganalysis} to facilitate interpretation of the estimated latent concepts by associating user-defined tags for each question with each abstract concept.

In Section \ref{sec:experiments}, we report on a range of experiments with a variety of synthetic and real-world data that demonstrate the wealth of information provided by the estimates of $\bW$, $\bC$ and $\bM$. 
As an example, \fref{fig:stemg14j} provides the results for a dataset collected from learners using \cite{stemwebsite}, a science curriculum platform.  The dataset consists of 145 Grade~8 learners from a single school district answering a manually tagged set of 80 questions on Earth science; only 13.5\% of all graded learner--question responses were observed.  
We applied the \mbox{SPARFA-B} algorithm to retrieve the factors $\bW$, $\bC$, and $\bM$ using 5 latent concepts. 
The resulting sparse matrix $\bW$ is displayed as a bipartite graph in \fref{fig:stems_a}; circles denote the abstract concepts and boxes denote questions. 
Each question box is labeled with its estimated intrinsic difficulty $\mu_i$, with large positive values denoting easy questions. Links between the concept and question nodes represent the active (non-zero) entries of $\bW$, with thicker links denoting larger values $W_{i,k}$. Unconnected questions are those for which no concept explained the learners' answer pattern; such questions typically have either very low or very high intrinsic difficulty, resulting in nearly all learners answering them correctly or incorrectly. The tags provided in \fref{fig:stems_b} enable human-readable interpretability of the estimated abstract concepts.

\begin{figure}[tp]
\vspace{-1.2cm}
\subfigure[Inferred question--concept association graph.]{
\centering
\includegraphics[width=0.65\columnwidth]{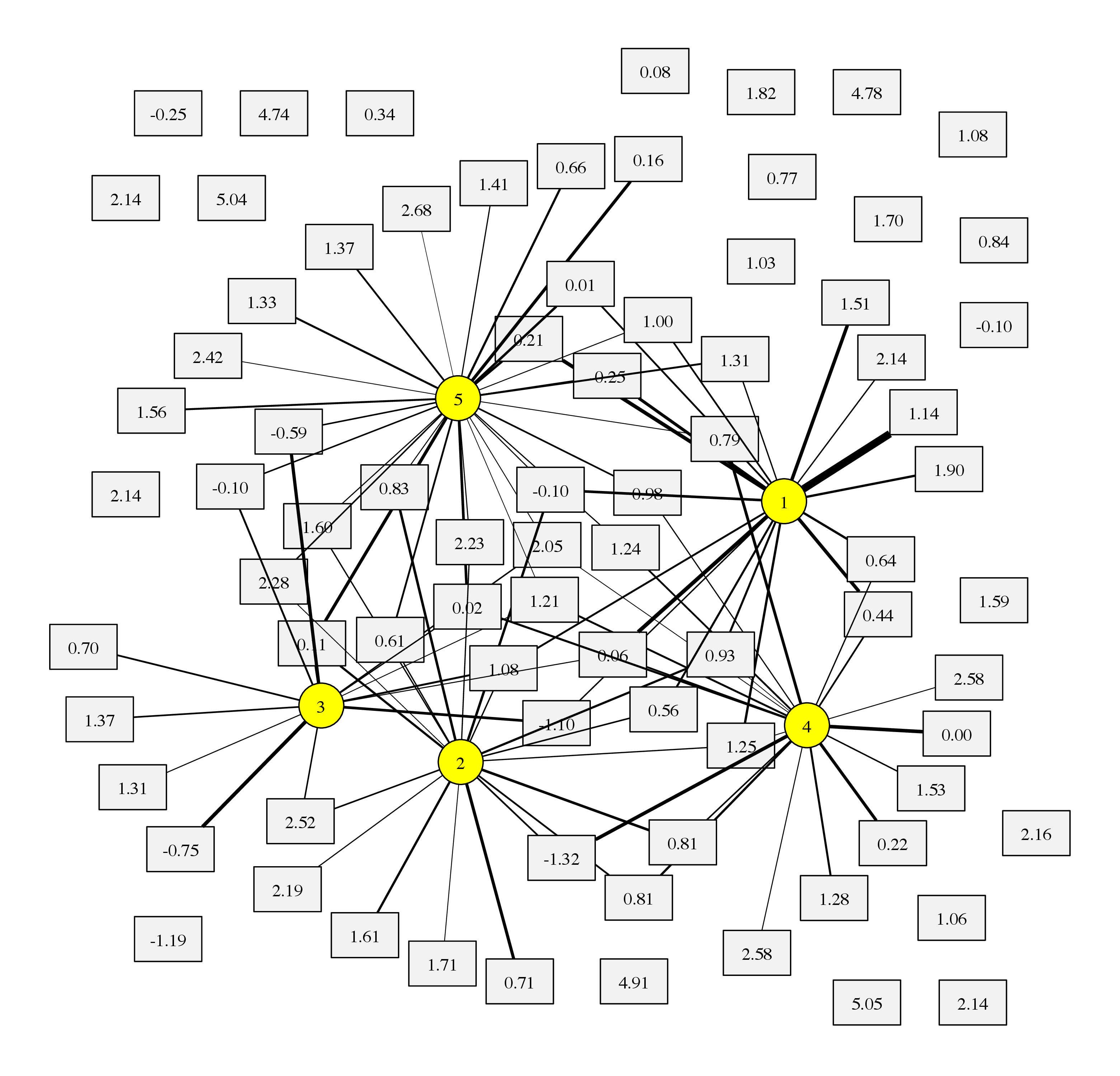}
\label{fig:stems_a}
}\\[0.3cm]
\subfigure[Most important tags and relative weights for the estimated concepts.]{
\scalebox{.84}{%
\begin{tabular}{llllll}
\toprule[0.2em]
Concept 1 && Concept 2 && Concept 3\\
\midrule[0.1em]
Changes to land & (45\%) &  Evidence of the past & (74\%) & Alternative energy & (76\%) \\
Properties of soil & (28\%) & Mixtures and solutions & (14\%) & Environmental changes & (19\%)\\ 
Uses of energy & (27\%) & Environmental changes & (12\%) & Changes from heat & (5\%)\\
\midrule[0.2em]
 Concept 4 && Concept 5\\ 
\midrule[0.1em]
Properties of soil & (77\%) &  Formulation of fossil fuels & (54\%)   \\
Environmental changes & (17\%) & Mixtures and solutions & (28\%)   &\\ 
Classifying matter & (6\%)  & Uses of energy & (18\%)\\ 
\bottomrule[0.2em]\\
\end{tabular}}
\label{fig:stems_b}
}
\vspace{-0.3cm}
\caption{\subref{fig:stems_a} Sparse question--concept association graph and \subref{fig:stems_b} most important tags associated with each concept for Grade 8 Earth science with $N=135$ learners answering $Q=80$ questions. Only 13.5\% of all graded learner--question responses were observed.}
\label{fig:stemg14j}
\vspace{-1.0cm}
\end{figure}

We envision a range of potential learning and content analytics applications for the SPARFA framework that go far beyond the standard practice of merely forming column sums of the ``gradebook'' matrix (with entries $Y_{i,j}$) to arrive at a final scalar numerical score for each learner (which is then often further quantized to a letter grade on a 5-point scale).   
Each column of the estimated $\bC$ matrix can be interpreted as a measure of the corresponding learner's knowledge about the abstract concepts.  
Low values indicate concepts ripe for remediation, while high values indicate concepts ripe for enrichment.  
The sparse graph stemming from the estimated $\bW$ matrix automatically groups questions into similar types based on their concept association; this graph makes it straightforward to find a set of questions similar to a given target question.
Finally, the estimated $\bM$ matrix (with entries~$\mu_i$ on each row) provides an estimate of each question's intrinsic difficulty. 
This property enables an instructor to assign questions in an orderly fashion as well as to prune out potentially problematic questions that are either too hard, too easy, too confusing, or unrelated to the concepts underlying the collection of questions.  

In \fref{sec:rw}, we provide an overview of related work on machine learning-based personalized learning, and we conclude in \fref{sec:conclusions}. All proofs are relegated to three Appendices.


\section{Statistical Model for Learning and Content Analytics} 
\label{sec:model}

Our approach to learning and content analytics is based on a new statistical model that encodes the probability that a learner will answer a given question correctly in terms of three factors:  
\begin{inparaenum}[(i)]
\item the learner's knowledge of a set of latent, abstract concepts,
\item how the question is related to each concept, and 
\item the intrinsic difficulty of the question.  
\end{inparaenum}

\subsection{Model for Graded Learner Response Data}
\label{sec:models}

Let $N$ denote the total number of learners, $Q$ the total number of questions, and $K$ the number of latent abstract concepts.  
We define $C_{k,j}$ as the \textit{concept knowledge} of learner~$j$ on concept $k$, with large positive values of $C_{k,j}$ corresponding to a better chance of success on questions related to concept $k$. 
Stack these values into the column vector $\bmc_j \in \mathbb{R}^{K}$, $j \in \{1,\ldots,N\}$ and the $K \times N$ matrix $\bC=[\,\vecc_1, \ldots, \vecc_N\,]$.
We further define $W_{i,k}$ as the \textit{question--concept association} of question $i$ with respect to concept $k$, with larger values denoting stronger involvement of the concept. 
Stack these values into the column vector
$\vecw_i \in \mathbb{R}^{K}$, $i \in \{1,\ldots,Q\}$ and the $Q \times K$ matrix $\bW = [\,\vecw_1, \ldots, \vecw_Q\,]^T$.
Finally, we define the scalar $\mu_i \in \mathbb{R}$ as the \emph{intrinsic difficulty} of question $i$, with larger values representing easier questions.
Stack these values into the column vector $\boldsymbol{\mu}$ and form the $Q \times N$ matrix $\bM=\boldsymbol{\mu}\,\bOne_{1\times N}$ as the product of $\boldsymbol{\mu}=[\,\mu_1, \ldots, \mu_Q\,]^T$ with the $N$-dimensional all-ones row vector~$\bOne_{1\times N}$. 

Given these definitions, we propose the following model for the binary-valued graded response variable $Y_{i,j} \in \{0,1\}$ for learner $j$ on question $i$, with $1$ representing a correct response and~$0$ an incorrect response:
\begin{align} 
Z_{i,j} & = \vecw_i^T \vecc_j + \mu_i, \!\quad \quad \forall i,j, \notag \\
Y_{i,j} & \sim \textit{Ber}(\Phi(Z_{i,j})), \quad (i,j)\in\Omega_\text{obs}. \label{eq:qa}
\end{align}
Here, $\textit{Ber}(z)$ designates a Bernoulli distribution with success probability $z$, and $\Phi(z)$ denotes an {\em inverse link function}\footnote{Inverse link functions are often called \emph{response functions} in the generalized linear models literature (see, e.g., \cite{responseglm}).} that maps a real value $z$ to the success probability of a binary random variable. Thus, the slack variable $\Phi(Z_{i,j}) \in [0,1]$ governs the probability of learner~$j$ answering question $i$ correctly. 

The set $\Omega_\text{obs}\subseteq\{1,\ldots,Q\}\times\{1,\ldots,N\}$ in \fref{eq:qa} contains the indices associated with the observed graded learner response data. Hence, our framework is able to handle the case of incomplete or missing data (e.g., when the learners do not answer all of the questions).\footnote{Two common situations lead to missing learner response data.
First, a learner might not attempt a question because it was not assigned or available to them.  
In this case, we simply exclude their response from $\Omega_\text{obs}$. 
Second, a learner might not attempt a question because it was assigned to them but was too difficult.
In this case, we treat their response as incorrect, as is typical in standard testing settings.}
Stack the values $Y_{i,j}$ and $Z_{i,j}$ into the $Q \times N$ matrices~$\bY$ and~$\bZ$, respectively.  
We can conveniently rewrite \fref{eq:qa} in matrix form as
\begin{align} \label{eq:qam}
Y_{i,j} \sim \textit{Ber}(\Phi(Z_{i,j})), \; (i,j)\in\Omega_\text{obs} \quad {\rm with} \quad \bZ = \bW \bC + \bM.  
\end{align}

In this paper, we focus on the two most commonly used link functions in the machine learning literature.
The \emph{inverse probit} function is defined as 
\begin{align} \label{eq:probitlink}
\Phi_\text{pro}(x) = \int_{-\infty}^x \! \mathcal{N}(t) \, \mathrm{d}t = \frac{1}{\sqrt{2 \pi}}\int_{-\infty}^x \! e^{-t^2 / 2} \, \mathrm{d}t,
\end{align} 
where $\mathcal{N}(t) = \frac{1}{\sqrt{2 \pi}} e^{-t^2/2}$ is the probability density function (PDF) of the standard normal distribution (with mean zero and variance one).  The \emph{inverse logit} link function is defined as 
\begin{align} \label{eq:logitlink}
\Phi_\text{log}(x) = \frac{1}{1+e^{-x}}.
\end{align}

As we noted in the Introduction, $\bC$, $\bW$, and $\boldsymbol{\mu}$ (or equivalently, $\bM$) have natural interpretations in real education settings.  Column $j$ of $\bC$ can be interpreted as a measure of learner $j$'s knowledge about the abstract concepts, with larger $C_{k,j}$ values implying more knowledge.  
The non-zero entries in $\bW$ can be used to visualize the connectivity between concepts and questions (see \fref{fig:graphex2_b} for an example), with larger $W_{i,k}$ values implying stronger ties between question~$i$ and concept $k$.
The values of $\boldsymbol{\mu}$ contains estimates of each question's intrinsic difficulty.

\subsection{Joint Estimation of Concept Knowledge and Question--Concept Association}
\label{sec:assumptions}

Given a (possibly partially observed) matrix of graded learner response data $\bY$, we aim to estimate the learner concept knowledge matrix $\bC$, the question--concept association matrix~$\bW$, and the question intrinsic difficulty vector $\boldsymbol{\mu}$.  In practice, the latent factors~$\bW$ and~$\bC$, and the vector $\boldsymbol{\mu}$ will contain many more unknowns than we have observations in~$\bY$; hence, estimating~$\bW$,~$\bC$, and $\boldsymbol{\mu}$ is, in general, an ill-posed inverse problem. 
The situation is further exacerbated if many entries in~$\bY$ are unobserved. 

To regularize this inverse problem, prevent over-fitting, improve identifiability,\footnote{If $\bZ =  \bW \bC$, then for any orthonormal matrix $\bH$ with $\bH^T \bH = \bI$, we have $\bZ = \bW \bH^T \bH \bC=\widetilde{\bW} \widetilde\bC$. Hence, the estimation of $\bW$ and $\bC$ is, in general, non-unique up to a unitary matrix rotation.} and enhance interpretability of the entries in $\bC$ and $\bW$, we appeal to the following three observations regarding education that are reasonable for typical exam, homework, and practice questions at all levels.  We will exploit these observations extensively in the sequel as fundamental assumptions:
\begin{itemize}
\item[(A1)] \emph{Low-dimensionality}: The number of latent, abstract concepts $K$ is small relative to both the number of learners $N$ and the number of questions $Q$. This implies that the questions are redundant and that the learners' graded responses live in a low-dimensional space. The parameter~$K$ dictates the concept \emph{granularity}. Small $K$ extracts just a few general, broad concepts, whereas large $K$ extracts more specific and detailed concepts.\footnote{Standard techniques like cross-validation (\cite{tibsbook}) can be used to select $K$. We provide the corresponding details in \fref{sec:realpred}.}

\item[(A2)] \emph{Sparsity}: Each question should be associated with only a small subset of the concepts in the domain of the course/assessment.
In other words, we assume that the matrix~$\bW$ is sparsely populated, i.e., contains mostly zero entries.

\item[(A3)] \emph{Non-negativity}: A learner's knowledge of a given concept does not negatively affect their probability of correctly answering a given question, i.e., knowledge of a concept is not ``harmful.''  In other words, the entries of $\bW$ are non-negative, which provides a natural interpretation for the entries in $\bC$:  Large values $C_{k,j}$ indicate strong knowledge of the corresponding concept, whereas negative values indicate weak knowledge.

\end{itemize}
In practice, $N$ can be larger than $Q$ and vice versa, and hence, we do not impose any additional assumptions on their values. 
Assumptions (A2) and (A3) impose sparsity and non-negativity constraints on $\bW$.  Since these assumptions are likely to be violated under arbitrary unitary transforms of the factors, they help alleviate several well-known identifiability problems that arise in factor analysis.

We will refer to the problem of estimating $\bW$, $\bC$, and $\boldsymbol{\mu}$, given the observations $\bY$, under the assumptions (A1)--(A3) as the {\em SPARse Factor Analysis} (SPARFA) problem.  We now develop two complementary algorithms to solve the SPARFA problem.  In Section \ref{sec:matrix}, we introduce SPARFA-M, a computationally efficient matrix-factorization approach that produces point estimates of the quantities of interest, in contrast to the principal component analysis based approach in~\cite{binopca}. In Section \ref{sec:bayes}, we introduce SPARFA-B, a Bayesian approach that produces full posterior estimates of the quantities of interest.

\section{SPARFA-M: Maximum Likelihood-based Sparse Factor Analysis} 
\label{sec:matrix}

Our first algorithm, SPARFA-M, solves the SPARFA problem using maximum-likelihood-based probit or logistic regression. 

\subsection{Problem Formulation}

To estimate $\bW$, $\bC$, and $\boldsymbol{\mu}$, we maximize the likelihood of the observed data  $Y_{i,j}$, $(i,j)\in\Omega_\text{obs}$
\begin{align*} 
p(Y_{i,j}|\vecw_i, \vecc_j) = \Phi\big(\vecw_i^T \vecc_j\big)^{Y_{i,j}} \; \big(1-\Phi(\vecw_i^T \vecc_j)\big)^{1-Y_{i,j}}
\end{align*}
given $\bW$, $\bC$, and $\boldsymbol{\mu}$ and subject to the assumptions (A1), (A2), and (A3) from Section \ref{sec:assumptions}. This likelihood yields the following optimization problem:
\begin{align*} 
(\text{P}^*) \,\, \left\{\begin{array}{ll}
 \underset{\bW,\bC}{\text{maximize}}\,\,\,  \sum_{(i,j)\in\Omega_\text{obs}} \log p(Y_{i,j}|\vecw_i, \vecc_j)   \\[0.3cm]
 \text{subject to}\,\,\, \|\vecw_i \|_0 \le s \; \forall i, \,\,\, \|\vecw_i \|_2 \le \kappa \; \forall i, \,\,\,  W_{i,k} \geq 0 \; \forall i,k, \,\,\, \normfro{\bC}=\xi.\end{array}\right.
\end{align*}
Let us take a quick tour of the problem $(\text{P}^*)$ and its constraints.  
The intrinsic difficulty vector~$\boldsymbol{\mu}$ is incorporated as an additional column of $\bW$, and $\bC$ is augmented with an all-ones row accordingly.
We impose sparsity on each vector $\vecw_i$ to comply with (A2) by limiting its maximum number of nonzero coefficients using the constraint $\|\vecw_i \|_0 \le s$; here $\|\veca\|_0$ counts the number of non-zero entries in the vector~$\veca$.
The $\elltwo$-norm constraint on each vector $\vecw_i$ with $\kappa>0$ is required for our convergence proof below.  
We enforce non-negativity on each entry $W_{i,k}$ to comply with (A3). 
Finally, we normalize the Frobenius norm of the concept knowledge matrix~$\bC$ to a given $\xi>0$ to suppress arbitrary scalings between the entries in both matrices $\bW$ and $\bC$.

Unfortunately, optimizing over the sparsity constraints $\|\vecw_i \|_0 \le s$ requires a combinatorial search over all $K$-dimensional support sets having no more than~$s$ non-zero entries. Hence, $(\text{P}^*)$ cannot be solved efficiently in practice for the typically large problem sizes of interest. 
In order to arrive at an optimization problem that can be solved with a reasonable computational complexity, we {\em relax} the sparsity constraints $ \|\vecw_i \|_0 \le s$ in $(\text{P}^*)$ to $\ellone$-norm constraints as in \cite{bpdn} and move them, the $\elltwo$-norm constraints, and the Frobenius norm constraint, into the objective function via Lagrange multipliers:
\begin{align*}
(\text{P}) \,\,\, \underset{\bW,\bC \colon \! W_{i,k} \geq 0\; \forall i,k}{\text{minimize}} \textstyle  \sum_{(i,j)\in\Omega_\text{obs}} \!\!-\log p(Y_{i,j}|\vecw_i , \vecc_j) + \lambda \sum_{i} \| \vecw_i \|_1 + \frac{\mu}{2} \sum_{i} \| \vecw_i \|_2^2 + \frac{\gamma}{2} \normfro{\bC}^2.
\end{align*}
The first regularization term $\lambda\sum_i\normone{\vecw_i}$ induces sparsity on each vector $\vecw_i$, with the single parameter  $\lambda>0$ controlling the sparsity level. 
Since one can arbitrarily increase the scale of the vectors $\vecw_i$ while decreasing the scale of the vectors $\vecc_j$ accordingly (and vice versa) without changing the likelihood, we gauge these vectors using the second and third regularization terms $\frac{\mu}{2}\sum_i\normtwo{\vecw_i}^2$ and $\frac{\gamma}{2}\normfro{\bC}^2$ with the regularization parameters $\mu>0$ and $\gamma>0$, respectively.\footnote{The first $\ellone$-norm regularization term in (RR$_1^+$) already gauges the norm of the $\vecw_i$. The $\elltwo$-norm regularizer $\frac{\mu}{2}\sum_i\normtwo{\vecw_i}^2$ is included only to aid in establishing the convergence results for SPARFA-M as detailed in \fref{sec:convm}.}
We emphasize that since $\normfro{\bC}^2=\sum_j\normtwo{\vecc_j}^2$, we can impose a regularizer on each column rather than the entire matrix $\bC$, which facilitates the development of the efficient algorithm detailed below.
%


\subsection{The SPARFA-M Algorithm} 
\label{sec:regularizedlogisticregression}

Since the first negative log-likelihood term in the objective function of $(\text{P})$ is convex in the product $\bW \bC$ for both the probit and the logit functions (see, e.g., \cite{tibsbook}), and since the rest of the regularization terms are convex in either $\bW$ or $\bC$ while the non-negativity constraints on $W_{i,k}$ are with respect to a convex set, the problem~$(\text{P})$ is \emph{biconvex} in the individual factors~$\bW$ and~$\bC$. 
More importantly, with respect to blocks of variables~$\vecw_i$,~$\vecc_j$, the problem $(\text{P})$ is \emph{block multi-convex} in the sense of \cite{wotao}.

SPARFA-M is an alternating optimization approach to (approximately) solving $(\text{P})$ that proceeds as follows.  We initialize $\bW$ and $\bC$ with random entries and then iteratively optimize the objective function of $(\text{P})$ for both factors in an alternating fashion. 
Each outer iteration involves solving two kinds of inner subproblems. In the first subproblem, we hold~$\bW$ constant and separately optimize each block of variables in~$\vecc_j$; in the second subproblem, we hold~$\bC$ constant and separately optimize each block of variables~$\vecw_i$. 
Each subproblem is solved using an iterative method; see \fref{sec:RPR} for the respective algorithms.
The outer loop is terminated whenever a maximum number of outer iterations $I_\text{max}$ is reached, or if the decrease in the objective function of $(\text{P})$ is smaller than a certain threshold.

The two subproblems constituting the inner iterations of SPARFA-M correspond to the following convex $\ell_1/\ell_2$-norm and $\ell_2$-norm regularized regression (RR) problems: 
\begin{align*}
&(\text{RR}_1^+) \quad 
\underset{\vecw_i\colon\!W_{i,k}\geq0\;\forall k}{\text{minimize}}\,\,\, \textstyle \!\sum\nolimits_{j\colon\! (i,j)\in\Omega_\text{obs}} \!\!-\log p(Y_{i,j}|\vecw_i ,\vecc_j) + \lambda\normone{\vecw_i} + \frac{\mu}{2}\normtwo{\vecw_i}^2, \\
&(\text{RR}_2) \quad \,\,\,\hspace{0.08cm}
\underset{\vecc_j}{\text{minimize}} \hspace{0.37cm} \textstyle \sum\nolimits_{i\colon\! (i,j)\in\Omega_\text{obs}} \!\!-\log p(Y_{i,j}|\vecw_i ,\vecc_j) + \frac{\gamma}{2}  \normtwo{\vecc_j}^2.
\end{align*}
We develop two novel first-order methods that efficiently solve $(\text{RR}_1^+)$ and $(\text{RR}_2)$ for both probit and logistic regression. These methods scale well to high-dimensional problems, in contrast to existing second-order methods. In addition, the probit link function makes the explicit computation of the Hessian difficult, which is only required for second-order methods. Therefore, we build our algorithm on the fast iterative soft-thresholding algorithm (FISTA) framework developed in \cite{fista}, which enables the development of efficient first-order methods with accelerated convergence.

\subsection{Accelerated First-Order Methods for Regularized Probit/Logistic Regression} 
\label{sec:RPR}

The FISTA framework (\cite{fista}) iteratively solves optimization problems whose objective function is given by $f(\cdot) + g(\cdot)$, where $f(\cdot)$ is a continuously differentiable convex function and $g(\cdot)$ is convex but potentially non-smooth. 
This approach is particularly well-suited to the inner subproblem $(\text{RR}_1^+)$ due to the presence of the non-smooth $\ellone$-norm regularizer and the non-negativity constraint.
Concretely, we associate the log-likelihood function plus the $\elltwo$-norm regularizer $\frac{\mu}{2}\normtwo{\vecw_i}^2$ with $f(\cdot)$ and the $\ellone$-norm regularization term with $g(\cdot)$.
For the inner subproblem $(\text{RR}_2)$, we associate the log-likelihood function with~$f(\cdot)$ and the $\elltwo$-norm regularization term with $g(\cdot)$.\footnote{Of course, both $f(\cdot)$ and $g(\cdot)$ are smooth for $(\text{RR}_2)$. Hence,  we could also apply an accelerated gradient-descent approach instead, e.g., as described in \cite{nest}.}

Each FISTA iteration consists of two steps: 
\begin{inparaenum}[(i)]
\item a gradient-descent step in~$f(\cdot)$ and
\item a shrinkage step determined by $g(\cdot)$.
\end{inparaenum}
For simplicity of exposition, we consider the case where all entries in $\bY$ are observed, i.e., $\Omega_\text{obs} = \{1, \ldots, Q\} \times \{1, \ldots, N\}$; the extension to the case with missing entries in $\bY$ is straightforward. 
We will derive the algorithm for the case of probit regression first and then point out the departures for logistic regression.

For $(\text{RR}_1^+)$, the gradients of $f(\vecw_i)$ with respect to the $i^{\rm th}$ block of regression coefficients~$\vecw_i$ are given by
\begin{align} 
\nabla f_\text{pro}^i  =  \nabla_{\vecw_i}^\text{pro} ( -\textstyle\sum\nolimits_{j} \log p_\text{pro}(Y_{i,j}|\vecw_i, \vecc_j) + \frac{\mu}{2}\normtwo{\vecw_i }^2 ) =  -\bC \bD^i (\bar\bmy^i  - \vecp_\text{pro}^i)  +  \mu \vecw_i, \label{eq:gradpro} 
\end{align}

where $\bar\bmy^i$ is an $N \times 1$ column vector corresponding to the transpose of the $i^\text{th}$ row of $\bY$. 
$\vecp_\text{pro}^i$ is an $N \times 1$ vector whose $j^\text{th}$ element equals the probability of $Y_{i,j}$ being~$1$; that is, $p_\text{pro}(Y_{i,j}=1|\vecw_i , \vecc_j) = \Phi_\text{pro} (\vecw_i^T \vecc_j)$.  The entries of the $N \times N$ diagonal  matrix are given by
\begin{align*}
D_{j,j}^i = \frac{\mathcal{N}(\vecw_i^T \vecc_j)}{\Phi_\text{pro} (\vecw_i^T \vecc_j)(1-\Phi_\text{pro} (\vecw_i^T \vecc_j))}.
\end{align*} 
The gradient step in each FISTA iteration~$\ell=1,2,\ldots$ corresponds to
\begin{align} \label{eq:fistagradient}
\hat{\vecw}_i^{\ell+1} \gets \vecw_i^\ell - t_\ell \, \nabla f_\text{pro}^i,
\end{align}
where $t_\ell$ is a suitable step-size.
To comply with (A3), the shrinkage step in $(\text{RR}_1^+)$ corresponds to a non-negative soft-thresholding operation
\begin{align}\label{eq:shrink2} 
\vecw_i^{\ell+1} \gets \max\{\hat{\vecw}_i^{\ell+1}-\lambda t_\ell,0\}.
\end{align}
For $(\text{RR}_2)$, the gradient step becomes
\[
\hat{\vecc}_j^{\ell+1} \gets \vecc_j^\ell - t_\ell \, \nabla f_\text{pro}^i,
\]
which is the same as \fref{eq:gradpro} and \fref{eq:fistagradient} after replacing $\bC$ with $\bW^T$ and $\mu$ with $\gamma$.  The shrinkage step for  $(\text{RR}_2)$ is the simple re-scaling
\begin{align} \label{eq:shrink1} 
\vecc_j^{\ell+1} \gets\frac{1}{1+\gamma t_{\ell}}\hat{\vecc}_j^{\ell+1}.
\end{align}

In the logistic regression case, the steps \fref{eq:fistagradient}, \fref{eq:shrink2}, and \fref{eq:shrink1} remain the same but the gradient changes to 
\begin{align} 
\nabla f_\text{log}^i &= \nabla_{\vecw_i}^\text{log}  (-\textstyle\sum\nolimits_{j} \log p_\text{log}(Y_{i,j}|\vecw_i, \vecc_j) +   \frac{\mu}{2}\normtwo{\vecw_i  }^2 ) =  -\bC (\bar\bmy^i  -  \vecp_\text{log}^i)  +  \mu \vecw_i, \label{eq:gradlog}
\end{align}
where the $N \times 1$ vector $\vecp_\text{log}^i$  has elements $\vecp_\text{log}(Y_{i,j}=1|\vecw_i, \vecc_j) = \Phi_\text{log} (\vecw_i^T \vecc_j)$.

The above steps require a suitable step-size~$t_\ell$ to ensure convergence to the optimal solution. 
A common approach that guarantees convergence is to set~$t_\ell = 1/{L}$, where $L$ is the Lipschitz constant of $f(\cdot)$ (see \cite{fista} for the details).  
The Lipschitz constants for both the probit and logit cases are analyzed in \fref{thm:rrconv} below.
Alternatively, one can also perform backtracking, which---under certain circumstances---can be more efficient;  see \cite[p.~194]{fista} for more details.

\subsection{Convergence Analysis of SPARFA-M}
\label{sec:convm}

While the SPARFA-M objective function is guaranteed to be non-increasing over the outer iterations (\cite{boydbook}), the factors $\bW$ and $\bC$ do not necessarily converge to a global or local optimum due to its biconvex (or more generally, block multi-convex) nature. It is difficult, in general, to develop rigorous statements for the convergence behavior of block multi-convex problems. 
Nevertheless, we can establish the global convergence of SPARFA-M from any starting point to a critical point of the objective function using recent results developed in \cite{wotao}. The convergence results below appear to be novel for both sparse matrix factorization as well as dictionary learning. 

\subsubsection{Convergence Analysis of Regularized Regression using FISTA}

In order to establish the SPARFA-M convergence result, we first adapt the convergence results for FISTA in \cite{fista} to prove convergence on the two subproblems~$(\text{RR}_1^+)$ and $(\text{RR}_2)$.  
The following theorem is a consequence of \cite[Thm.~4.4]{fista} combined with Lemmata \ref{lem:scalarlipschitz} and \ref{lem:vectorlips} in \fref{app:convRR}.
If back-tracking is used to select step-size $t_\ell$ \cite[p.~194]{fista}, then let~$\alpha$ correspond to the back-tracking parameter.
Otherwise set $\alpha=1$ and for $(\text{RR}_1^+)$ let $t_\ell=1/L_1$ and for $(\text{RR}_2)$ let $t_\ell=1/L_2$.  
In \fref{lem:vectorlips}, we compute that $L_1 = \sigma_\text{max}^2(\bC)+\mu$ and $L_2 = \sigma_\text{max}^2(\bW)$ for the probit case, and $L_1 = \frac{1}{4}\sigma_\text{max}^2(\bC) + \mu$ and $L_2 = \frac{1}{4}\sigma_\text{max}^2(\bW)$ for the logit case.
\begin{thm}[Linear convergence of RR using FISTA] \label{thm:rrconv} 
Given $i$ and $j$, let 
\begin{align*}
& F_1(\vecw_i) = \sum\nolimits_{j\colon\! (i,j)\in\Omega_{\text{obs}}} -\log p(Y_{i,j}|\vecw_i ,\vecc_j) + \lambda\normone{\vecw_i} + \frac{\mu}{2}\normtwo{\vecw_i }^2, \quad  W_{i,k} \geq 0\; \forall k,\\
& F_2(\vecc_j) = \sum\nolimits_{i\colon\! (i,j)\in\Omega_\text{obs}}-\log p(Y_{i,j}|\vecw_i ,\vecc_j) + \frac{\gamma}{2}\normtwo{\vecc_j}^2
\end{align*}
be the cost functions of $(\text{RR}_1^+)$ and $(\text{RR}_2)$, respectively. Then, we have
\begin{align*}
 F_1(\vecw_i^\ell) - F_1(\vecw_i^*) & \leq \frac{2\alpha L_1 \|\vecw_i^0 - \vecw_i^*\|^2}{(\ell +1)^2}, \\
 F_2(\vecc_j^\ell) - F_2(\vecc_j^*) & \leq \frac{2\alpha L_2 \|\vecc_j^0 - \vecc_j^*\|^2}{(\ell +1)^2},
\end{align*}
where $\vecw_i^0$ and $\vecc_j^0$ are the initialization points of $(\text{RR}_1^+)$ and $(\text{RR}_2)$, $\vecw_i^\ell$ and $\vecc_j^\ell$ designate the solution estimates at the $\ell^\text{th}$ inner iteration, and $\vecw_i^*$ and $\vecc_j^*$ denote the optimal solutions. 
\end{thm}

In addition to establishing convergence, \fref{thm:rrconv} reveals that the difference between the cost functions at the current estimates and the optimal solution points, $F_1(\vecw_i^\ell) - F_1(\vecw_i^*)$ and $F_2(\vecc_j^\ell) - F_2(\vecc_j^*)$, decrease as $O(\ell^{-2})$.

\subsubsection{Convergence Analysis of SPARFA-M}
\label{sec:convanalusissparfam}

We are now ready to establish global convergence of SPARFA-M to a critical point. To this end, we first define 
$\vecx = [\vecw_1^T, \ldots, \vecw_Q^T, \vecc_1^T, \ldots, \vecc_N^T]^T \in \mathbb{R}^{(N+Q) K}$
and rewrite the objective function $(\text{P})$ of SPARFA-M as follows:
\begin{align*}
 F(\vecx) =  \!\!\!\!\!\! \sum_{(i,j)\in\Omega_\text{obs}} \!\! \!\!\!\! - \log p(Y_{i,j}|\vecw_i , \vecc_j)  +  \frac{\mu}{2} \sum_{i} \| \vecw_i \|_2^2 + \lambda \sum_{i} \| \vecw_i \|_1 
 + \sum_{i,k} \delta(W_{i,k}\! <\!0) + \frac{\gamma}{2}  \sum_{j} \| \vecc_j \|_2^2
\end{align*}
with the indicator function $\delta(z<0)=\infty$ if $z<0$ and $0$ otherwise. Note that we have re-formulated the non-negativity constraint as a set indicator function and added it to the objective function of~$(\text{P})$.
Since minimizing $F(\vecx)$ is equivalent to solving~$(\text{P})$, we can now use the results developed in \cite{wotao} to establish the following convergence result for the SPARFA-M algorithm. The proof can be found in \fref{app:globproof}.
\begin{thm} [Global convergence of SPARFA-M] \label{thm:sparfamconv}
From any starting point $\vecx^0$, let $\{ \vecx^t \}$ be the sequence of estimates generated by the SPARFA-M algorithm with $t=1,2,\ldots$ as the outer iteration number. Then, the sequence $\{ \vecx^t \}$ converges to the finite limit point $\hat{\vecx}$, which is a critical point of (P).
Moreover, if the starting point $\vecx^0$ is within a close neighborhood of a global optimum of (P), then SPARFA-M converges to this global optimum.
\end{thm}
Since the problem $(\text{P})$ is bi-convex in nature, we cannot guarantee that SPARFA-M always converges to a \emph{global optimum} from an \emph{arbitrary} starting point. 
Nevertheless, the use of multiple randomized initialization points can be used to increase the chance of being in the close vicinity of a global optimum, which improves the (empirical) performance of \mbox{SPARFA-M} (see \fref{sec:addtechniques} for details).
Note that we do not provide the convergence rate of SPARFA-M, since the associated parameters in \cite[Thm.~2.9]{wotao} are difficult to determine for the model at hand; a detailed analysis of the convergence rate for SPARFA-M is part of ongoing work.

\subsection{Algorithmic Details and Improvements for SPARFA-M}
\label{sec:addtechniques}

In this section, we outline a toolbox of techniques that improve the empirical performance of SPARFA-M and provide guidelines for choosing the key algorithm parameters.

\subsubsection{Reducing Computational Complexity in Practice}

To reduce the computational complexity of SPARFA-M in practice, we can improve the convergence rates of $(\text{RR}_1^+)$ and $(\text{RR}_2)$.  In particular, the regularizer $ \frac{\mu}{2}\normtwo{\vecw_i}^2$ in $(\text{RR}_1^+)$ has been added to $(\text{P})$ to facilitate the proof for \fref{thm:sparfamconv}. 
This term, however, typically slows down the (empirical) convergence of FISTA, especially for large values of~$\mu$. We therefore set $\mu$ to a small positive value (e.g., $\mu=10^{-4}$), which leads to fast convergence of $(\text{RR}_1^+)$ while still guaranteeing convergence of SPARFA-M.

Selecting the appropriate (i.e., preferably large) step-sizes $t_\ell$ in \fref{eq:fistagradient}, \fref{eq:shrink2}, and \fref{eq:shrink1} is also crucial for fast convergence. In Lemmata \ref{lem:scalarlipschitz} and \ref{lem:vectorlips}, we derive the Lipschitz constants $L$ for~$(\text{RR}_1^+)$ and $(\text{RR}_2)$, which enables us to set the step-sizes $t_\ell$ to the constant value $t = 1/L$. In all of our experiments below, we exclusively use constant step-sizes, since we observed that backtracking (\cite[p.~194]{fista})  provided no advantage in terms of computational complexity for SPARFA-M.

To further reduce the computational complexity of SPARFA-M without degrading its empirical performance noticeably, we have found that instead of running the large number of inner iterations it typically takes to converge, we can run just a few (e.g., $10$) inner iterations per outer iteration.

\subsubsection{Reducing the Chance of Getting Stuck in Local Minima}

The performance of SPARFA-M strongly depends on the initialization of $\bW$ and~$\bC$, due to the bi-convex nature of ({\rm P}). 
We have found that running SPARFA-M multiple times with different starting points and picking the solution with the smallest overall objective function delivers excellent performance.
In addition, we can deploy the standard heuristics used in the dictionary-learning literature \cite[Section~IV-E]{ksvd}  to further improve the convergence towards a global optimum.
For example, every few outer iterations, we can evaluate the current $\bW$ and $\bC$. If two rows of $\bC$ are similar (as measured by the absolute value of the inner product between them), then we re-initialize one of them as an i.i.d.~Gaussian vector. Moreover, if some columns in $\bW$ contain only zero entries, then we re-initialize them with i.i.d.~Gaussian vectors.  Note that the convergence proof in \fref{sec:convm} does not apply to implementations employing such trickery.

\subsubsection{Parameter Selection}

The input parameters to SPARFA-M include the number of concepts $K$ and the regularization parameters $\gamma$  and $\lambda$.  The number of concepts~$K$ is a user-specified value. In practice, cross-validation could be used to select $K$ if the task is to predict missing entries of $\bY$, (see \fref{sec:realpred}).
The sparsity parameter $\lambda$ and the $\elltwo$-norm penalty parameter $\gamma$ strongly affect the output of SPARFA-M; they can be selected using any of a number of criteria, including the Bayesian information criterion (BIC) or cross-validation, as detailed in \cite{tibsbook}. Both criteria resulted in similar performance in all of the experiments reported in sec:experiments.


\subsection{Related Work on Maximum Likelihood-based Sparse Factor Analysis}

Sparse logistic factor analysis has previously been studied in \cite{binopca} in the principal components analysis context. 
There are three major differences with the SPARFA framework.
First, \cite{binopca} do not impose the non-negativity constraint on $\bW$ that is critical for the interpretation of the estimated factors.
Second, they impose an orthonormality constraint on $\bC$ that does not make sense in educational scenarios.
Third, they optimize an upper bound on the negative log-likelihood function in each outer iteration, in contrast to
SPARFA-M, which optimizes the exact cost functions in~(RR$_1^+$) and (RR$_2$).

The problem $(\text{P})$ shares some similarities with the method for missing data imputation outlined
in \cite[Eq.~7]{l1vsbayes}. However, the problem~$(\text{P})$ studied here includes an additional
non-negativity constraint on $\bW$ and the regularization term $\frac{\mu}{2}\sum_i\normtwo{\vecw_i}^2$ that are important for the interpretation of the estimated factors and the convergence analysis.
Moreover, SPARFA-M utilizes the accelerated FISTA framework as opposed to the more straightforward
but less efficient gradient descent method in \cite{l1vsbayes}.

SPARFA-M is capable of handling both the inverse logit and inverse probit link functions.
For the inverse logit link function, one could solve $(\text{RR}_1^+)$ and $(\text{RR}_2)$ using an iteratively
reweighted second-order algorithm as in~\cite{tibsbook}, \cite{minka}, \cite{ng}, \cite{parkhastie}, or an interior-point method as in \cite{kohboyd}. 
However, none of these techniques extend naturally to the inverse probit link function, which is essential for some applications, e.g., in noisy compressive sensing recovery from 1-bit measurements (e.g., \cite{jason1bit} or \cite{plan1bit}).
Moreover, second-order techniques typically do not scale well to high-dimensional problems due to the necessary computation of the Hessian. 
In contrast, SPARFA-M scales favorably thanks to its accelerated first-order FISTA optimization, which
avoids the computation of the Hessian.

\section{SPARFA-B: Bayesian Sparse Factor Analysis} 
\label{sec:bayes}

\sloppy

Our second algorithm, SPARFA-B, solves the SPARFA problem using a Bayesian method based on Markov chain Monte-Carlo (MCMC) sampling.  In contrast to SPARFA-M, which computes point estimates for each of the parameters of interest, \mbox{SPARFA-B} computes full posterior distributions for~$\bW, \bC$, and $\bimud$.

\fussy

While SPARFA-B has a higher computational complexity than SPARFA-M, it has several notable benefits in the context of learning and content analytics.  First, the full posterior distributions enable the computation of informative quantities such as credible intervals and posterior modes for all parameters of interest.  Second, since MCMC methods explore the full posterior space, they are not subject to being trapped indefinitely in local minima, which is possible with SPARFA-M. Third, the hyperparameters used in Bayesian methods generally have  intuitive meanings, in contrary to the regularization parameters of optimization-based methods like SPARFA-M. These hyperparameters can also be specially chosen to incorporate additional prior information about the problem.

\subsection{Problem Formulation} \label{priors}

As discussed in \fref{sec:assumptions}, we require the matrix $\bW$ to be both sparse (A2) and non-negative (A3). We enforce these assumptions through the following prior distributions that are a variant of the well-studied spike-slab model \citep{west2003bayesian,ishwaran2005spike} adapted for non-negative factor loadings:
\begin{align}
W_{i,k} &\sim r_k \, \textit{Exp}(\lambda_k) + (1-r_k) \, \delta_0, \,\,\,
\lambda_k \sim \textit{Ga}(\alpha,\beta), \,\,\, \text{and} \,\,\,
r_k \sim \textit{Beta}(e,f). \label{eq:priors1}
\end{align}
Here, $\textit{Exp}(x|\lambda) \sim \lambda e^{-\lambda x}$, $x \geq 0$, and $\textit{Ga}(x|\alpha,\beta) \sim \frac{\beta^\alpha x^{\alpha-1}e^{-\beta x}}{\Gamma(\alpha)}$, $x \geq 0$, $\delta_0$ is the Dirac delta function, and $\alpha$, $\beta$, $e$, $f$ are hyperparameters. The model \eqref{eq:priors1} uses the latent random variable $r_k$ to control the sparsity via the hyperparameters $e$ and~$f$. This set of priors induces a conjugate form on the posterior that enables efficient sampling. We note that both the exponential rate parameters $\lambda_k$ as well as the inclusion probabilities $r_k$ are grouped per factor. The remaining priors used in the proposed Bayesian model are summarized as 
\begin{align}
& \bmc_j  \sim  \mathcal{N}(0,\bV), \,\,\,
\bV \sim \textit{IW}(\bV_0,h), \,\,\, \text{and} \,\,\,
\mu_i \sim  \mathcal{N}(\mu_0,v_\bimud), \label{eq:priors2}
\end{align}
where $\bV_0$, $h$, $\mu_0$, and $v_{\boldsymbol{\mu}}$ are hyperparameters. 

\subsection{The SPARFA-B Algorithm}
\label{posteriors}

We obtain posterior distribution estimates for the parameters of interest through an MCMC method based on the Gibbs' sampler. To implement this, we must derive the conditional posteriors for each of the parameters of interest.  We note again that the graded learner-response matrix~$\bY$ will not be fully observed, in general. Thus, our sampling method must be equipped to handle missing data.  

The majority of the posterior distributions follow from standard results in Bayesian analysis and will not be derived in detail here. The exception is the posterior distribution of $W_{i,k}\; \forall i,k$. The spike-slab model that enforces sparsity in $\bW$ requires first sampling~\mbox{$W_{i,k} \neq 0 |  \bZ, \bC, \bimud$} and then sampling $W_{i,k} | \bZ, \bC, \bimud$, for all $W_{i,k} \neq 0$. 
These posterior distributions differ from previous results in the literature due to our assumption of an exponential (rather than a normal) prior on  $W_{i,k}$.  We next derive these two results in detail.

\subsubsection{Derivation of Posterior Distribution of $W_{i,k}$}

We seek both the probability that an entry $W_{i,k}$ is active (non-zero) and the distribution of $W_{i,k}$ when active given our observations. The following theorem states the final sampling results; the proof is given in \fref{app:sparfabd}.

\begin{thm}[Posterior distributions for $\bW$] \label{thm:wpost} 
For all $i = 1, \ldots, Q$ and all $k = 1, \ldots, K$, the posterior sampling results for $W_{i,k} = 0 |  \bZ, \bC, \bimud$ and $W_{i,k} | \bZ, \bC, \bimud, W_{i,k} \neq 0$ are given by
\begin{align*}
& \textstyle \widehat{R}_{i,k} = p(W_{i,k} = 0 | \bZ, \bC, \boldsymbol{\mu}) =  \frac{\frac{ \mathcal{N}^r\!(0|\widehat{M}_{i,k},\widehat{S}_{i,k},\lambda_k)}{\textit{Exp}(0|\lambda_k)} (1-r_k)}{\frac{ \mathcal{N}^r\!(0|\widehat{M}_{i,k},\widehat{S}_{i,k},\lambda_k)}{\textit{Exp}(0|\lambda_k)} (1-r_k) + r_k}, \notag\\
& W_{i,k} | \bZ, \bC, \boldsymbol{\mu}, W_{i,k} \neq 0 \sim  \mathcal{N}^r(\widehat{M}_{i,k},\widehat{S}_{i,k},\lambda_k), \notag \\
&  \widehat{M}_{i,k} = \frac{\sum_{\{j : (i,j) \in \Omega_{\text{obs}} \} } \bigl( (Z_{i,j} - \mu_i)  - \sum_{k^\prime \neq k} W_{i, k^\prime} C_{k^\prime, j} \bigr) C_{k, j}}{\sum_{\{j : (i,j) \in \Omega_{\text{obs}} \} } C_{k, j}^2}, \notag \\ 
&\widehat{S}_{i,k} = \left({\textstyle \sum_{\{j : (i,j) \in \Omega_{\text{obs}} \} } C_{k, j}^2}\right)^{-1}, 
\end{align*}
where $ \mathcal{N}^{r}(x|m,s,\lambda) =  \frac{e^{\lambda m - \lambda^2 s / 2}}{\sqrt{2 \pi s} \Phi\bigl(\frac{m-\lambda s}{\sqrt{s}} \bigr)} e^{-(x-m)^2/2s - \lambda m}$ represents a rectified normal distribution (see~\cite{schmidt2009bayesian}). 

\end{thm}

\subsubsection{Sampling Methodology}
\label{sec:smeth}

SPARFA-B carries out the following MCMC steps to compute posterior distributions for all parameters of interest:
\begin{enumerate}
\item For all $(i, j) \in \Omega_{\text{obs}}$, draw $Z_{i, j} \sim  \mathcal{N}\bigl((\bW \bC)_{i, j} + \mu_i,1\bigr)$, truncating above 0 if $Y_{i,j} = 1$, and truncating below 0 if $Y_{i,j} = 0$. 
\item For all $i = 1, \ldots, Q$, draw $\mu_i \sim  \mathcal{N}(m_i,v)$ with
$v = (v_\bimud^{-1} + n^\prime)^{-1}$, $m_i = \mu_0 + v \sum_{\{ j : (i,j) \in \Omega_{\text{obs}}\}} \left( Z_{i,j} - \vecw_i^T \bmc_j \right)$, and $n^\prime$ the number of learners responding to question~$i$.

\item For all $j = 1, \ldots, N$, draw $\bmc_j \sim  \mathcal{N}(\bmm_j,\bM_j)$ with $\bM_j = (\bV^{-1} + \widetilde{\bW}^T \widetilde{\bW})^{-1}$, and $\bmm_j = \bM_j \widetilde{\bW}^T (\widetilde{\bmz_j} - \widetilde{\bimud})$. The notation $\widetilde{(\cdot)}$ denotes the restriction of the vector or matrix to the set of rows $i : (i,j) \in \Omega_{\text{obs}}$.
\item Draw $\bV \sim \textit{IW}(\bV_0 + \bC \bC^T, N+h)$. 
\item For all $i = 1, \ldots, Q$ and $k = 1, \ldots, K$, draw $W_{i, k} \sim \widehat{R}_{i, k} \,  \mathcal{N}^r(\widehat{M}_{i,k}, \widehat{S}_{i,k}) + (1-\widehat{R}_{i, k}) \delta_0$, where $\widehat{R}_{i,k}, \widehat{M}_{i,k}$, and $\widehat{S}_{i,k}$ are as stated in \fref{thm:wpost}.
\item For all $k = 1, \ldots, K$, let $b_k$ define the number of active (i.e., non-zero) entries of $\vecw_k$. Draw \mbox{$\lambda_k \sim \textit{Ga}(\alpha + b_k, \beta + \sum_{i=1}^Q W_{i, k})$}.
\item For all $k = 1, \ldots, K$, draw $r_k \sim \textit{Beta}(e+b_k, f+Q-b_k)$, with $b_k$ defined as in Step 6. 
\end{enumerate}

\subsection{Algorithmic Details and Improvements for SPARFA-B}
Here we discuss some several practical issues for efficiently implementing SPARFA-B, selecting the hyperparameters, and techniques for easy visualization of the SPARFA-B results. 

\subsubsection{Improving Computational Efficiency}

The Gibbs sampling scheme of SPARFA-B enables efficient implementation in several ways.  First, draws from the truncated normal in Step 1 of Section \ref{sec:smeth} are decoupled from one another, allowing them to be performed independently and, potentially, in parallel.  Second, sampling of the elements in each column of $\bW$ can be carried out in parallel by computing the relevant factors of Step 5 in matrix form. Since $K \ll Q,N$ by assumption~(A1), the relevant parameters are recomputed only a relatively small number of times.  
One taxing computation is the calculation of the covariance matrix $\bM_j$ for each $j= 1, \ldots, N$ in Step 3.  This computation is necessary, since we do not constrain each learner to answer the same set of questions which, in turn, changes the nature of the covariance calculation for each individual learner.  For data sets where all learners answer the same set of questions, this covariance matrix is the same for all learners and, hence, can be carried out once per MCMC iteration. 

\subsubsection{Parameter Selection}

The selection of the hyperparameters is performed at the discretion of the user. As is typical for Bayesian methods, non-informative (broad) hyperparameters can be used to avoid biasing results and to allow for adequate exploration of the posterior space.  Tighter hyperparameters can be used when additional side information is available. For example, prior information from subject matter experts might indicate which concepts are related to which questions or might indicate the intrinsic difficulty of the questions.  

Since SPARFA-M has a substantial speed advantage over SPARFA-B, it may be advantageous to first run SPARFA-M and then use its output to help in determining the hyperparameters or to initialize the SPARFA-B variables directly.

\subsubsection{Post-Processing for Data Visualization}
\label{sec:sparfabpostprocessing}

As discussed above, the generation of posterior statistics is one of the primary advantages of SPARFA-B.  However, for many tasks, such as visualization of the retrieved knowledge base, it is often convenient to post-process the output of SPARFA-B to obtain point estimates for each parameter.  
For many Bayesian methods, simply computing the posterior mean is often sufficient.  This is the case for most parameters computed by SPARFA-B, including~$\bC$ and~$\boldsymbol{\mu}$. The posterior mean of $\bW$, however, is generally non-sparse, since the MCMC will generally explore the possibility of including each entry of $\bW$.  Nevertheless, we can easily generate a sparse $\bW$ by examining the posterior mean of the inclusion statistics contained in~$\widehat{R}_{i,k}$,~$\forall i,k$. Concretely, if the posterior mean of $\widehat{R}_{i,k}$ is small, then we set the corresponding entry of $W_{i,k}$ to zero.  Otherwise, we set $W_{i,k}$ to its posterior mean.  We will make use of this method throughout the experiments presented in \fref{sec:experiments}.

\subsection{Related Work on  Bayesian Sparse Factor Analysis} \label{bayesian_prevwork}

Sparsity models for Bayesian factor analysis have been well-explored in the statistical literature \citep{west2003bayesian,tipping2001sparse,ishwaran2005spike}.  One popular avenue for promoting sparsity is to place a prior on the variance of each component in $\bW$ (see, e.g., \cite{tipping2001sparse}, \cite{fokoue2004stochastic}, and \cite{pournara2007factor}). In such a model, large variance values indicate active components, while small variance values indicate inactive components. Another approach is to model active and inactive components directly using a form of a spike-slab model due to \cite{west2003bayesian} and used in  \cite{bengioss},  \cite{l1vsbayes}, and \cite{hahn2010sparse}:
\begin{align*}
W_{i,k} \sim r_k \, \mathcal{N}(0,v_k)  + (1-r_k) \, \delta_0, \,\,\, v_k \sim \textit{IG}(\alpha,\beta), \,\,\, \text{and} \,\,\,
r_k &\sim \textit{Beta}(e,f).
\end{align*}
The approach employed in \eqref{eq:priors1}  utilizes a spike-slab prior with an exponential distribution, rather than a normal distribution, for the active components of $\bW$. We chose this prior for several reasons:  First, it enforces the non-negativity assumption (A3). Second, it induces a posterior distribution that can be both computed in closed form and sampled efficiently. Third, its tail is slightly heavier than that of a standard normal distribution, which improves the exploration of quantities further away from zero.

A sparse factor analysis model with non-negativity constraints that is related to the one proposed here was discussed in \cite{meng2010uncovering}, although their methodology is quite different from ours. Specifically, they impose non-negativity on the (dense) matrix $\bC$ rather than on the sparse factor loading matrix~$\bW$. Furthermore, they enforce non-negativity using a truncated normal\footnote{One could alternatively employ a truncated normal distribution on the support $[0,\infty)$ for the  active entries in  $\bW$.  In experiments with this model, we found a slight, though noticeable, improvement in prediction performance on real-data experiments using the exponential prior.} rather than an exponential prior. \\

\section{Tag Analysis: Post-Processing to Interpret the Estimated Concepts}
\label{sec:taganalysis}

So far we have developed SPARFA-M and SPARFA-B to estimate $\bW$, $\bC$, and $\boldsymbol{\mu}$ (or equivalently, $\bM$) in \fref{eq:qam} given the partial binary observations in $\bY$.  Both $\bW$ and $\bC$ encode a small number of latent concepts.  As we initially noted, the concepts are ``\emph{abstract}'' in that they are estimated from the data rather than dictated by a subject matter expert. 
In this section we develop a principled post-processing approach to interpret the meaning of the abstract concepts after they have been estimated from learner responses, which is important if our results are to be usable for learning analytics and content analytics in practice.  Our approach applies when the questions come with a set of user-generated ``tags'' or ``labels'' that describe in a free-form manner what ideas underlie each question.

We develop a post-processing algorithm for the estimated matrices $\bW$ and~$\bC$ that estimates the association between the latent concepts and the user-generated tags, enabling concepts to be interpreted as a ``bag of tags.'' Additionally, we show how to extract a personalized tag knowledge profile for each learner.  The efficacy of our tag-analysis framework will be demonstrated in the real-world experiments in \fref{sec:real}.

\subsection{Incorporating Question--Tag Information}
\label{sec:taginfo}

Suppose that a set of tags has been generated for each question that represent the topic(s) or theme(s) of each question.  The tags could be generated by the course instructors, subject matter experts, learners, or, more broadly, by crowd-sourcing.  In general, the tags provide a redundant representation of the true knowledge components, i.e., concepts are associated to a ``bag of tags.''

Assume that there is a total number of $M$ tags associated with the $Q$ questions.  
We form a $Q \times M$ matrix $\bT$, where each column of $\bT$ is associated to one of the $M$ pre-defined tags. 
We set $T_{i,m} = 1$ if tag $m\in\{1,\ldots,M\}$ is present in question $i$ and~$0$ otherwise.
Now, we postulate that the question association matrix $\bW$ extracted by SPARFA can be further factorized as $\bW = \bT \bA$, where $\bA$ is an $M \times K$ matrix representing the tags-to-concept mapping. This leads to the following additional assumptions: 
\begin{itemize}
\item[(A4)] {\em{Non-negativity:}} The matrix~$\bA$ is non-negative.  This increases the interpretability of the result, since concepts should not be negatively correlated with any tags, in general.

\item[(A5)] {\em{Sparsity:}} Each column of $\bA$ is sparse.  This ensures that the estimated concepts relate to only a few tags.
\end{itemize}

\subsection{Estimating the Concept--Tag Associations and Learner--Tag Knowledge}
\label{sec:cta}

The assumptions (A4) and (A5) enable us to extract $\bA$ using $\ellone$-norm regularized non-negative least-squares as described in \cite{tibsbook} and \cite{bpdn}. Specifically, to obtain each column $\veca_k$ of $\bA$, $k=1,\ldots,K$, we solve the following convex optimization problem, a non-negative variant of {\em basis pursuit denoising}:
\begin{align*}
(\text{BPDN}_+) \quad \underset{\veca_k\colon\!A_{m,k}\geq0\; \forall m}{\text{minimize}}\,\,\, \textstyle \frac{1}{2} \|\bmw_k - \bT \veca_k \|_2^2 + \eta \normone{\veca_k}.
\end{align*} 
Here, $\bmw_k$ represents the $k^\text{th}$ column of $\bW$, and the parameter $\eta$ controls the sparsity level of the solution $\veca_k$.

We propose a first-order method derived from the FISTA framework in \cite{fista} to solve $(\text{BPDN}_+)$. The algorithm consists of two steps:  A gradient step with respect to the $\elltwo$-norm penalty function, and a projection step with respect to the $\ellone$-norm regularizer subject to the non-negative constraints on $\veca_k$. 
By solving $(\text{BPDN}_+)$ for $k=1,\ldots,K$, and building $\bA=[\bma_1,\ldots,\bma_K]$, we can 
\begin{inparaenum}[(i)]
\item assign tags to each concept based on the non-zero entries in $\bA$ and 
\item estimate a tag-knowledge profile for each learner. 
\end{inparaenum}

\subsubsection{Associating Tags to Each Concept}

Using the concept--tag association matrix $\bA$ we can directly associate tags to each concept estimated by the SPARFA algorithms. 
We first normalize the entries in $\bma_k$ such that they sum to one. 
With this normalization, we can then calculate percentages that show the proportion of each tag that contributes to concept $k$ corresponding to the non-zero entries of $\bma_k$.
This concept tagging method typically will assign multiple tags to each concept, thus, enabling one to identify the coarse meaning of each concept (see \fref{sec:real} for examples using real-world data). 

\subsubsection{Learner Tag Knowledge Profiles}
\label{sec:tagprof}

Using the concept--tag association matrix $\bA$, we can assess each learner's knowledge of each tag. 
To this end, we form an $M \times N$ matrix $\bU = \bA \bC$, where the $U_{m,j}$ characterizes the knowledge of learner $j$ of tag $m$.
This information could be used, for example, by a PLS to automatically inform each learner which tags they have strong knowledge of and which tags they do not.  
Course instructors can use the information contained in $\bU$ to extract measures representing the knowledge of all learners on a given tag, e.g., to identify the tags for which the entire class lacks strong knowledge.  
This information would enable the course instructor to select future learning content that deals with those specific tags.
A real-world example demonstrating the efficacy of this framework is shown below in \fref{sec:real301}.

\section{Experiments}
\label{sec:experiments}

In this section, we validate SPARFA-M and SPARFA-B on both synthetic and real-world educational data sets. 
First, using synthetic data, we validate that both algorithms can accurately estimate the underlying factors from binary-valued observations and characterize their performance under different circumstances. 
Specifically, we benchmark the factor estimation performance of SPARFA-M and SPARFA-B against a variant of the well-established K-SVD algorithm (\cite{ksvd}) used in dictionary-learning applications.
Second, using  real-world graded learner-response data we demonstrate the efficacy SPARFA-M (both probit and logit variants) and of SPARFA-B for learning and content analytics.
Specifically, we showcase how the estimated learner concept knowledge, question--concept association, and intrinsic question difficulty can support machine learning-based personalized learning.  
Finally, we compare SPARFA-M against the recently proposed binary-valued collaborative
filtering algorithm CF-IRT (\cite{logitfa}) that predicts unobserved learner responses.

\subsection{Synthetic Data Experiments}
\label{sec:synth}

We first characterize the estimation performance of SPARFA-M and SPARFA-B using synthetic test data generated from a known ground truth model. 
We generate instances of~$\bW$,~$\bC$, and $\boldsymbol{\mu}$ under pre-defined distributions and then generate the binary-valued observations $\bY$ according to \fref{eq:qam}.

Our report on the synthetic experiments is organized as follows. 
In \fref{sec:baseline}, we outline K-SVD$_+$, a variant of the well-established K-SVD dictionary-learning (DL) algorithm originally proposed in \cite{ksvd}; we use it as a baseline method for comparison to both SPARFA algorithms. 
In \fref{sec:emetric} we detail the performance metrics. 
We compare SPARFA-M, SPARFA-B, and K-SVD$_+$ as we vary the problem size and number of concepts (\fref{sec:synthfull}), observation incompleteness (\fref{sec:synthpunc}),  and the sparsity of~$\bW$ (\fref{sec:synthspar}). 
In the above-referenced experiments, we simulate the observation matrix $\bY$ via the inverse probit link function and use only the probit variant of SPARFA-M in order to make a fair comparison with SPARFA-B. 
In a real-world situation, however, the link function is generally unknown. 
In \fref{sec:synthrob} we conduct model-mismatch experiments, where we generate data from one link function but analyze assuming the other.  

In all synthetic experiments, we average the results of all performance measures over 25 Monte-Carlo trials, limited primarily by the computational complexity of SPARFA-B, for each instance of the model parameters we control.

\subsubsection{Baseline Algorithm: K-SVD$_+$}
\label{sec:baseline}

Since we are not aware of any existing algorithms to solve \fref{eq:qam} subject to the assumptions (A1)--(A3), we deploy a novel baseline algorithm based on the well-known K-SVD algorithm of \cite{ksvd}, which is widely used in various dictionary learning settings but ignores the inverse probit or logit link functions.
Since the standard K-SVD algorithm also ignores the non-negativity constraint used in the SPARFA model, we develop a variant of the non-negative K-SVD algorithm proposed in \cite{nnksvd} that we refer to as K-SVD$_+$.
In the sparse coding stage of K-SVD$_+$, we use the non-negative variant of orthogonal matching pursuit (OMP) outlined in \cite{omp}; that is, 
we enforce the non-negativity constraint by iteratively picking the entry corresponding to the maximum inner product without taking its absolute value. 
We also solve a non-negative least-squares problem to determine the residual error for the next iteration. 
In the dictionary update stage of K-SVD$_+$, we use a variant of the rank-one approximation algorithm detailed in \cite[Figure~4]{nnksvd}, where we impose non-negativity on the elements in $\bW$ but not on the elements of~$\bC$.

K-SVD$_{+}$ has as input parameters the sparsity level of each row of $\bW$. 
In what follows, we provide K-SVD$_{+}$ with the known \emph{ground truth} for the number of non-zero components in order to obtain its best-possible performance. 
This will favor K-SVD$_{+}$ over both SPARFA algorithms, since, in practice, such oracle information is not available. 

\subsubsection{Performance Measures}
\label{sec:emetric}

In each simulation, we evaluate the performance of SPARFA-M, SPARFA-B, and K-SVD$_+$ by comparing the fidelity of the estimates $\widehat{\bW}$, $\widehat{\bC}$, and $\hat{\boldsymbol{\mu}}$ to the ground truth $\bW$, $\bC$, and $\boldsymbol{\mu}$. Performance evaluation is complicated by the facts that
\begin{inparaenum}[(i)]
\item SPARFA-B outputs posterior distributions rather than simple point estimates of the parameters and
\item factor-analysis methods are generally susceptible to permutation of the latent factors.
\end{inparaenum}
We address the first concern by post-processing the output of SPARFA-B to obtain point estimates for $\bW$, $\bC$, and $\boldsymbol{\mu}$ as detailed in \fref{sec:sparfabpostprocessing} using $\widehat{R}_{i,k} < 0.35$ for the threshold value. 
We address the second concern by normalizing the columns of $\bW$, $\widehat{\bW}$ and the rows of $\bC$, $\widehat{\bC}$ to unit $\elltwo$-norm, permuting the columns of~$\widehat{\bW}$ and $\widehat{\bC}$ to best match the ground truth, and then compare $\bW$ and $\bC$ with the estimates $\widehat{\bW}$ and $\widehat{\bC}$. We also compute the Hamming distance between the support set of $\bW$ and that of the (column-permuted) estimate $\widehat{\bW}$.  
To summarize, the performance measures used in the sequel are 
\begin{align*}
&  E_\bW = \|\bW - \widehat{\bW}\|_F^2/\|\bW\|_F^2, && E_\bC = \|\bC - \widehat{\bC}\|_F^2/\|\bC\|_F^2, \\ 
&  E_{\boldsymbol{\mu}} = \|\boldsymbol{\mu} - \hat{\boldsymbol{\mu}}\|_2^2/\|\boldsymbol{\mu}\|_2^2, && E_\bH = \| \bH - \widehat{\bH} \|_F^2/ \| \bH \|_F^2,
\end{align*} 
where $\bH \in \{0,1\}^{Q \times K}$ with $H_{i,k} = 1$ if $W_{i,k} > 0$ and $H_{i,k} = 0$ otherwise.
The $Q\times K$ matrix~$\widehat\bH$ is defined analogously using $\widehat\bW$. 

\subsubsection{Impact of Problem Size and Number of Concepts}
\label{sec:synthfull}

In this experiment, we study the performance of SPARFA vs.\ KSVD$_+$ as we vary the number of learners $N$, the number of questions $Q$, and the number of concepts $K$.  

\paragraph{Experimental setup}

We vary the number of learners $N$ and the number of questions $Q$ $\in$  $\{50, 100, 200\}$, and the number of concepts $K \in\{5,10\}$.  For each combination of $(N,Q,K)$, we generate $\bW$, $\bC$, $\boldsymbol{\mu}$, and~$\bY$ according to~\fref{eq:priors1} and~\fref{eq:priors2} with $v_\mu = 1$, $\lambda_k={2}/{3}\;\forall k$, and $\bV_{\!0} = \bI_K$. 
For each instance, we choose the number of non-zero entries in each row of $\bW$ as $\textit{DU}(1,3)$ where $\textit{DU}(a,b)$ denotes the discrete uniform distribution in the range $a$ to $b$.
For each trial, we run the probit version of SPARFA-M, SPARFA-B, and K-SVD$_+$  to obtain the estimates $\widehat{\bW}$, $\widehat{\bC}$, $\hat{\boldsymbol{\mu}}$, and calculate $\widehat{\bH}$.  
For all of the synthetic experiments with SPARFA-M, we set the regularization parameters $\gamma=0.1$ and select $\lambda$ using the BIC (\cite{tibsbook}).
For SPARFA-B, we set the hyperparameters to $h=K+1$, $v_\mu = 1$, $\alpha=1$, $\beta=1.5$, $e=1$, and $f=1.5$; moreover, we burn-in the MCMC for 30,000 iterations and take output samples over the next 30,000 iterations.

\paragraph{Results and discussion}

\fref{fig:synth_size_bayes} shows box-and-whisker plots for the three algorithms and the four performance measures.
We observe that the performance of all of the algorithms generally improves as the problem size increases. Moreover, SPARFA-B has superior performance for $E_\bW$, $E_\bC$, and $E_{\boldsymbol{\mu}}$.
We furthermore see that  both SPARFA-B and SPARFA-M outperform K-SVD$_+$ on  $E_\bW$, $E_\bC$, and especially $E_{\boldsymbol{\mu}}$. 
K-SVD$_+$ performs very well in terms of $E_\bH$ (slightly better than both SPARFA-M and SPARFA-B) due to the fact that we provide it with the oracle sparsity level, which is, of course, not available in practice. 
SPARFA-B's improved estimation accuracy over SPARFA-M comes at the price of significantly higher computational complexity. For example, for $N=Q=200$ and $K=5$, SPARFA-B requires roughly 10~minutes on a 3.2\,GHz quad-core desktop PC, while SPARFA-M and K-SVD$_+$ require only $6$\,s.

In summary, SPARFA-B is well-suited to small problems where solution accuracy or the need for confidence statistics are the key factors; SPARFA-M, in contrast, is destined for analyzing large-scale problems where low computational complexity (e.g., to generate immediate learner feedback) is important.

\begin{figure}[tp]
\vspace{-1.0cm}
\centering

\subfigure{\includegraphics[width=0.245\textwidth]{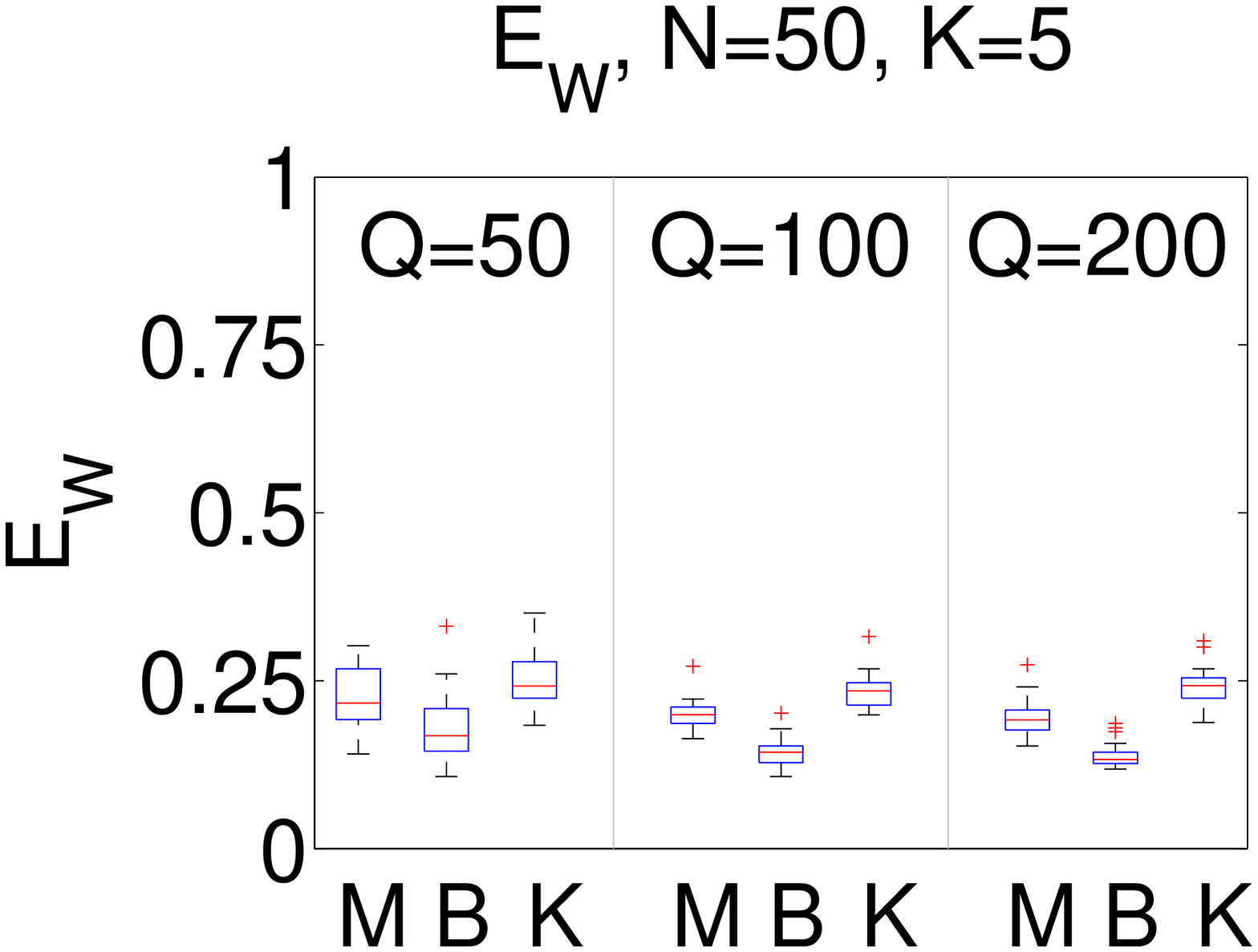}} \hspace{-0.1cm}
\subfigure{\includegraphics[width=0.245\textwidth]{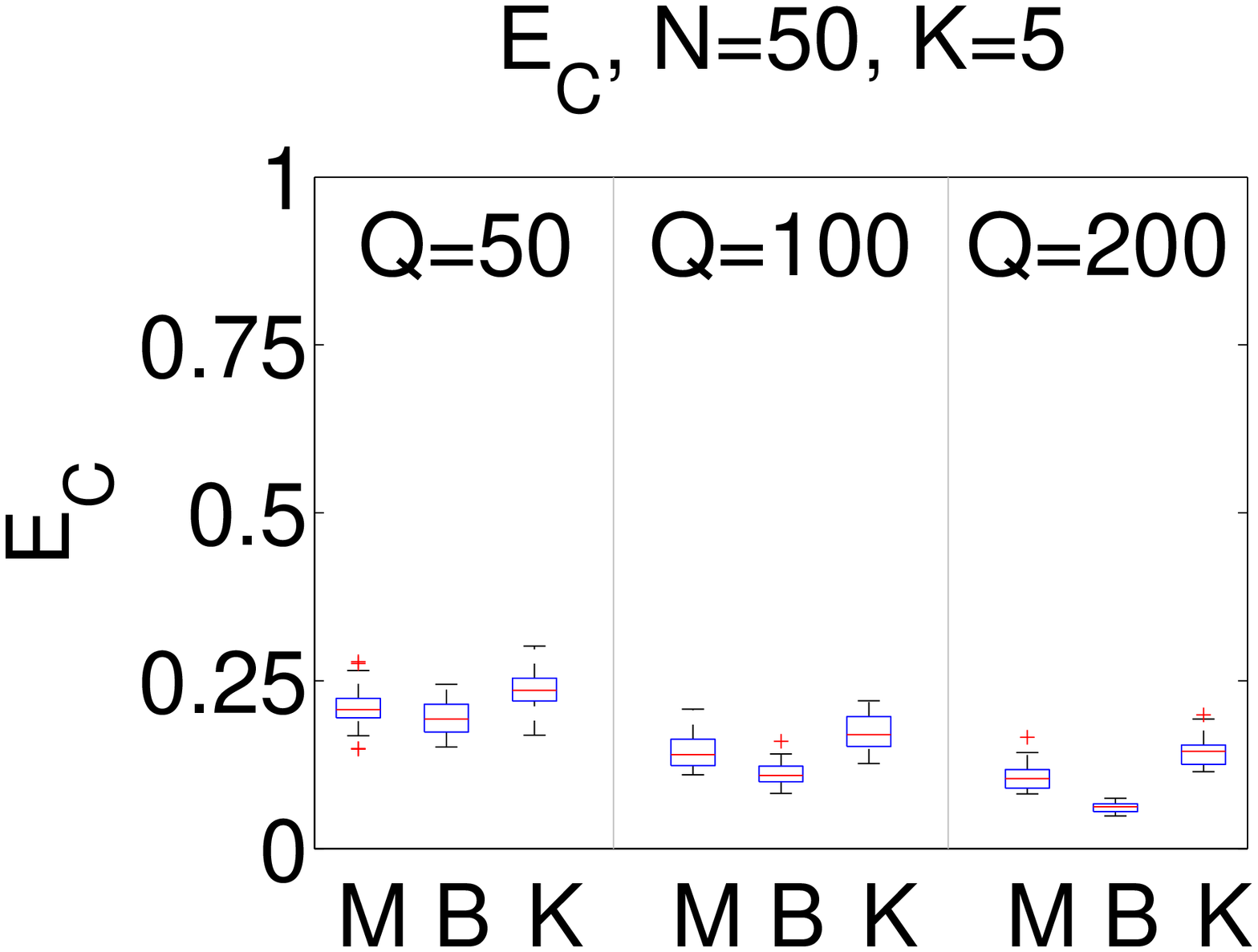}}\hspace{-0.1cm}
\subfigure{\includegraphics[width=0.245\textwidth]{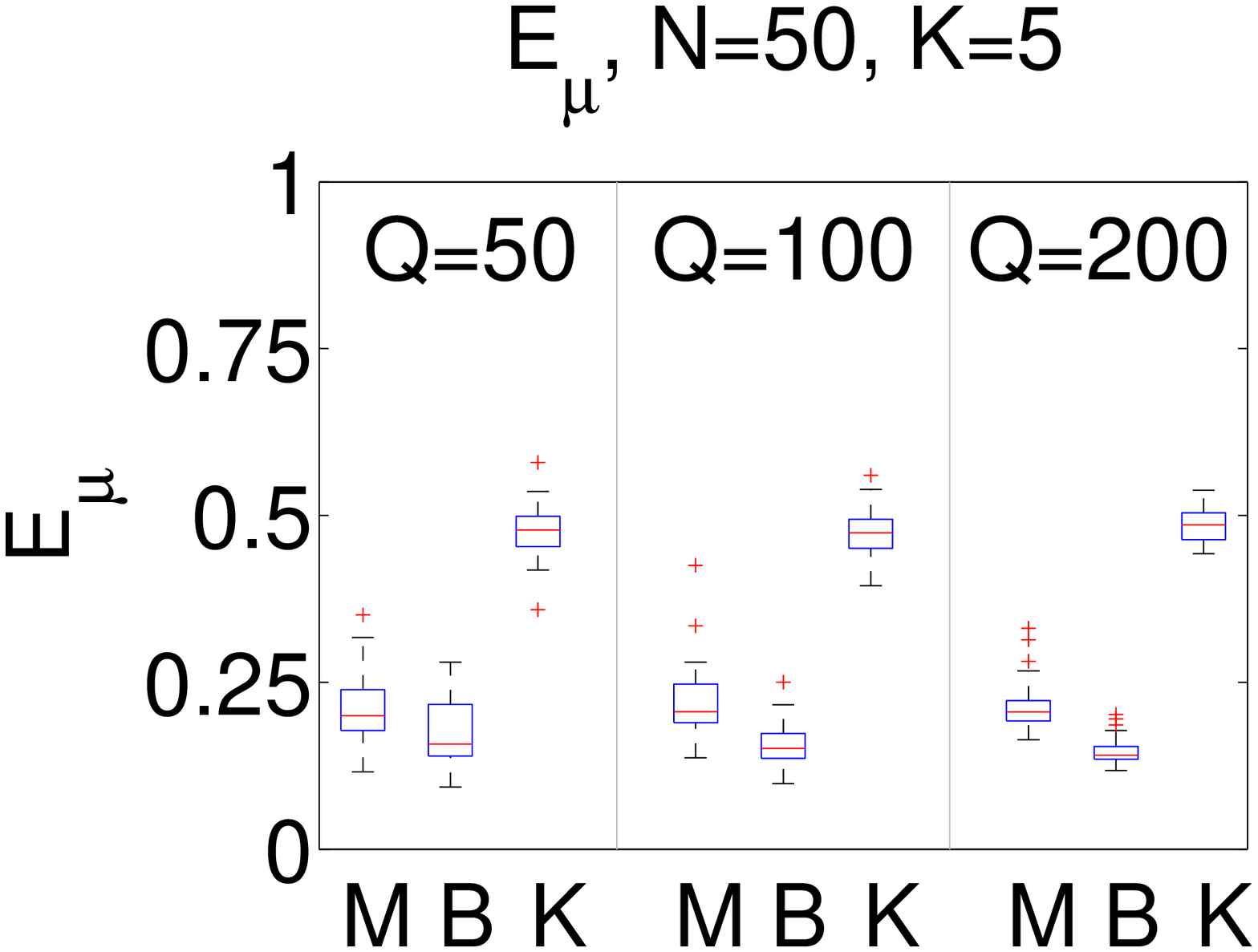}}
\hspace{-0.1cm}
\subfigure{\includegraphics[width=0.245\textwidth]{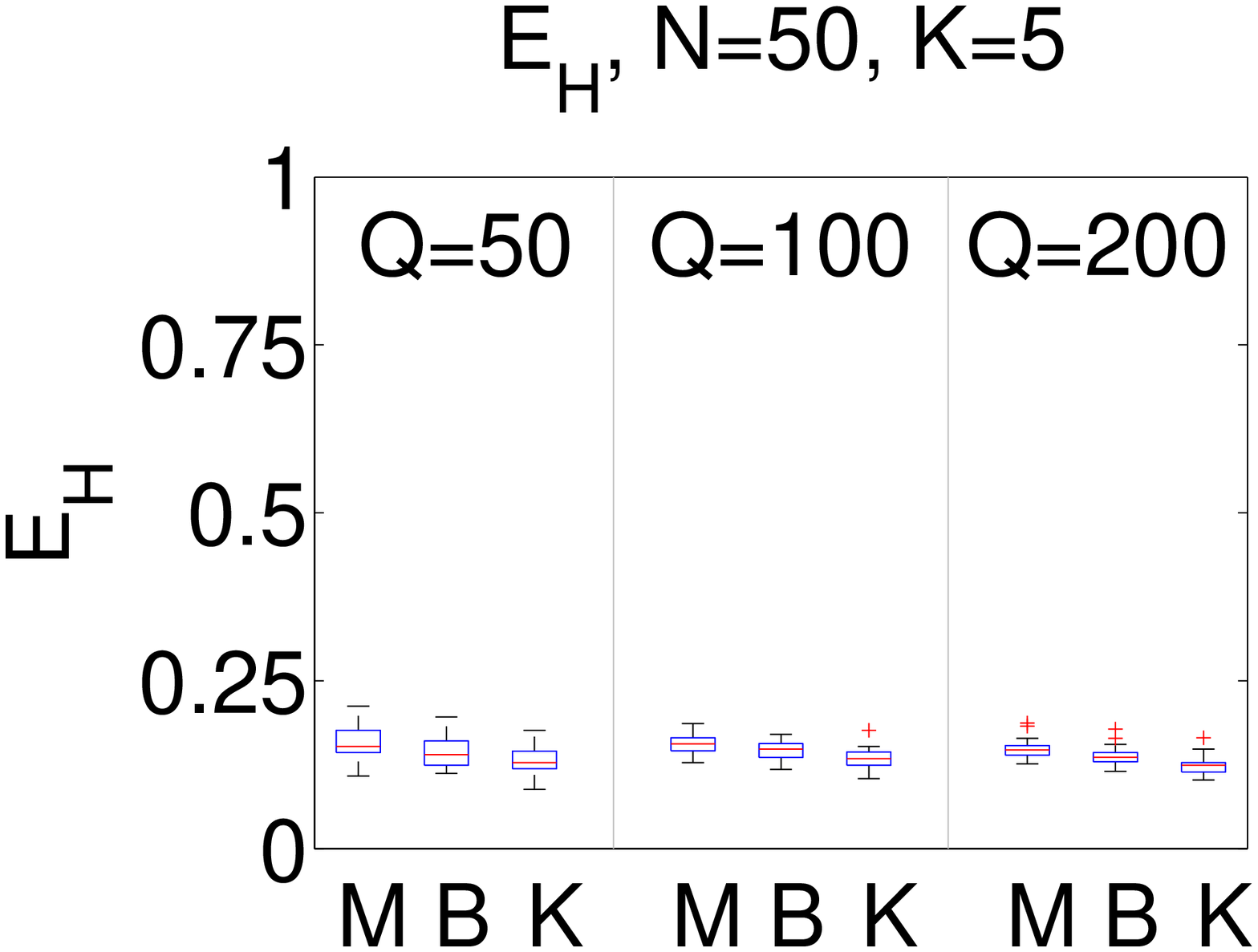}}
\addtocounter{subfigure}{-3}
\vspace{-0.0cm}
\subfigure{\includegraphics[width=0.245\textwidth]{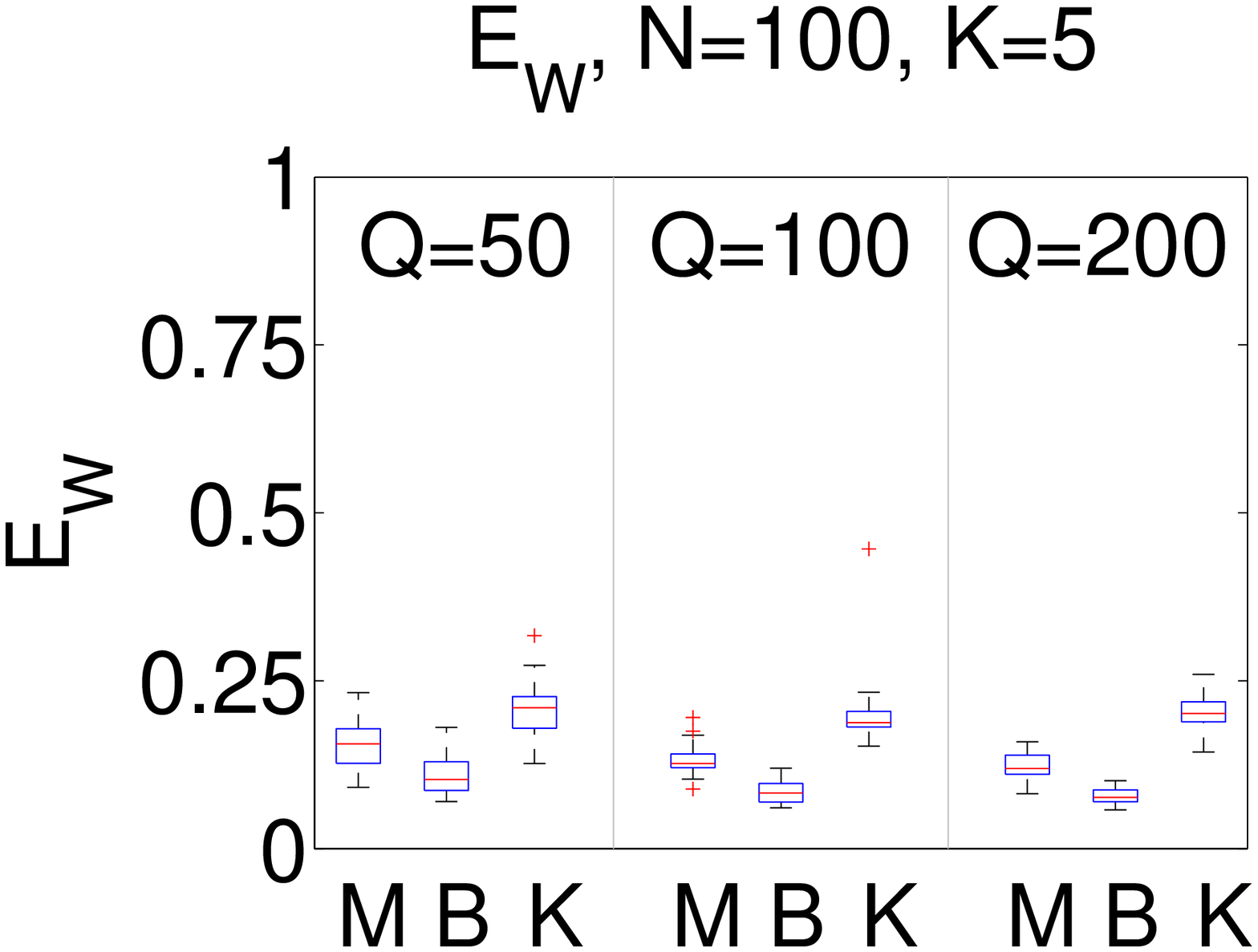}} \hspace{-0.1cm}
\subfigure{\includegraphics[width=0.245\textwidth]{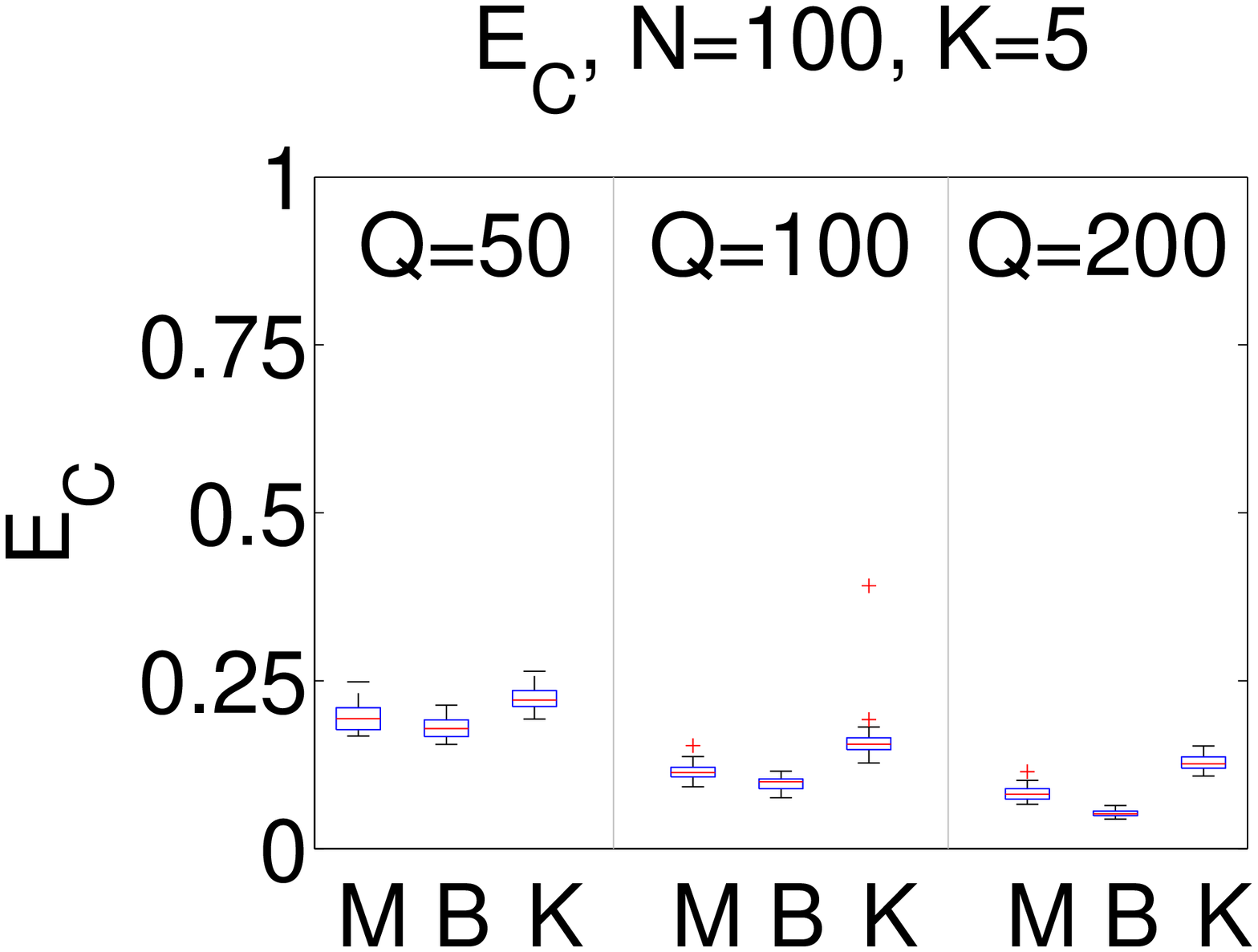}}\hspace{-0.1cm}
\subfigure{\includegraphics[width=0.245\textwidth]{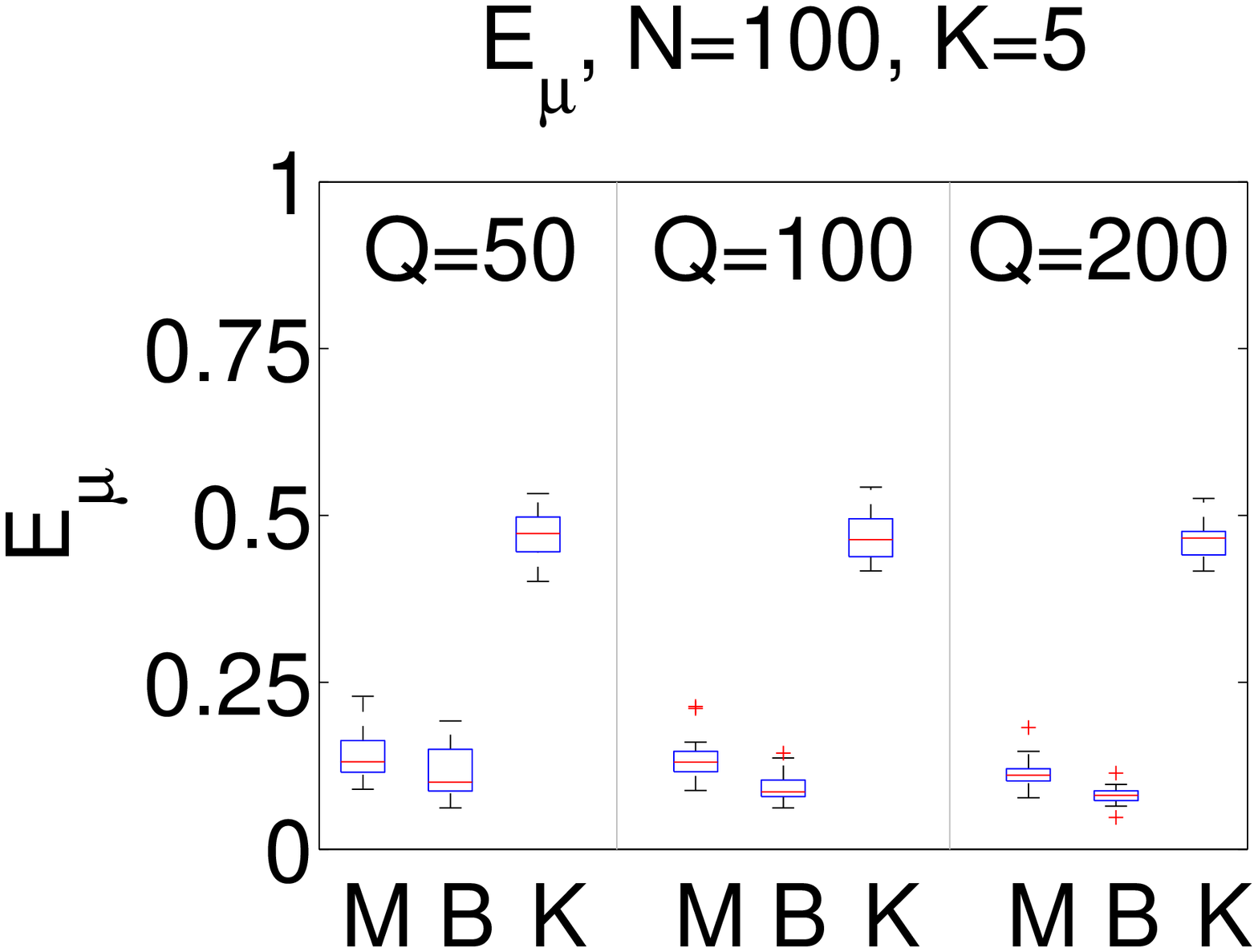}}
\hspace{-0.1cm}
\subfigure{\includegraphics[width=0.245\textwidth]{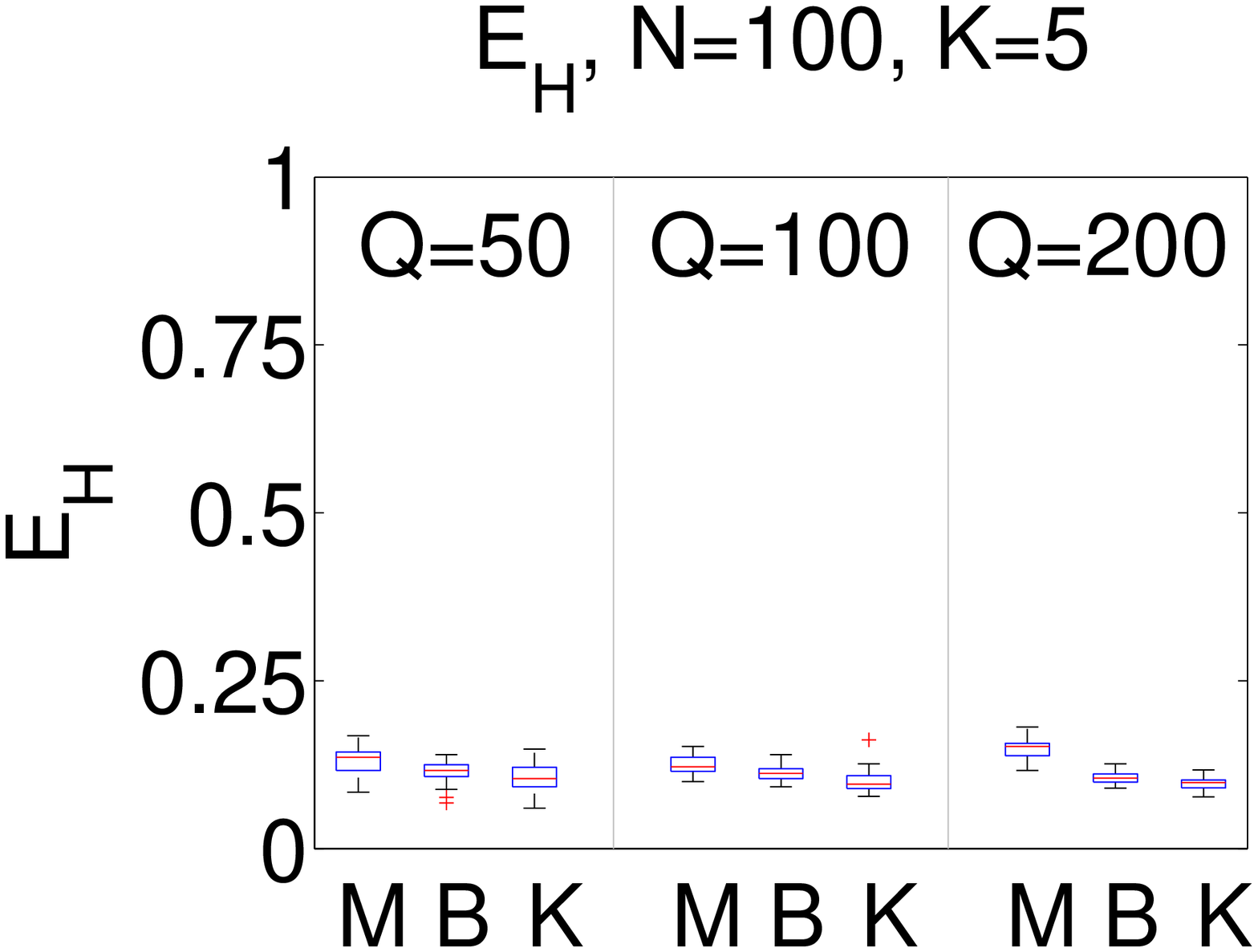}}
\addtocounter{subfigure}{-3}
\vspace{-0.0cm}
\subfigure{\includegraphics[width=0.245\textwidth]{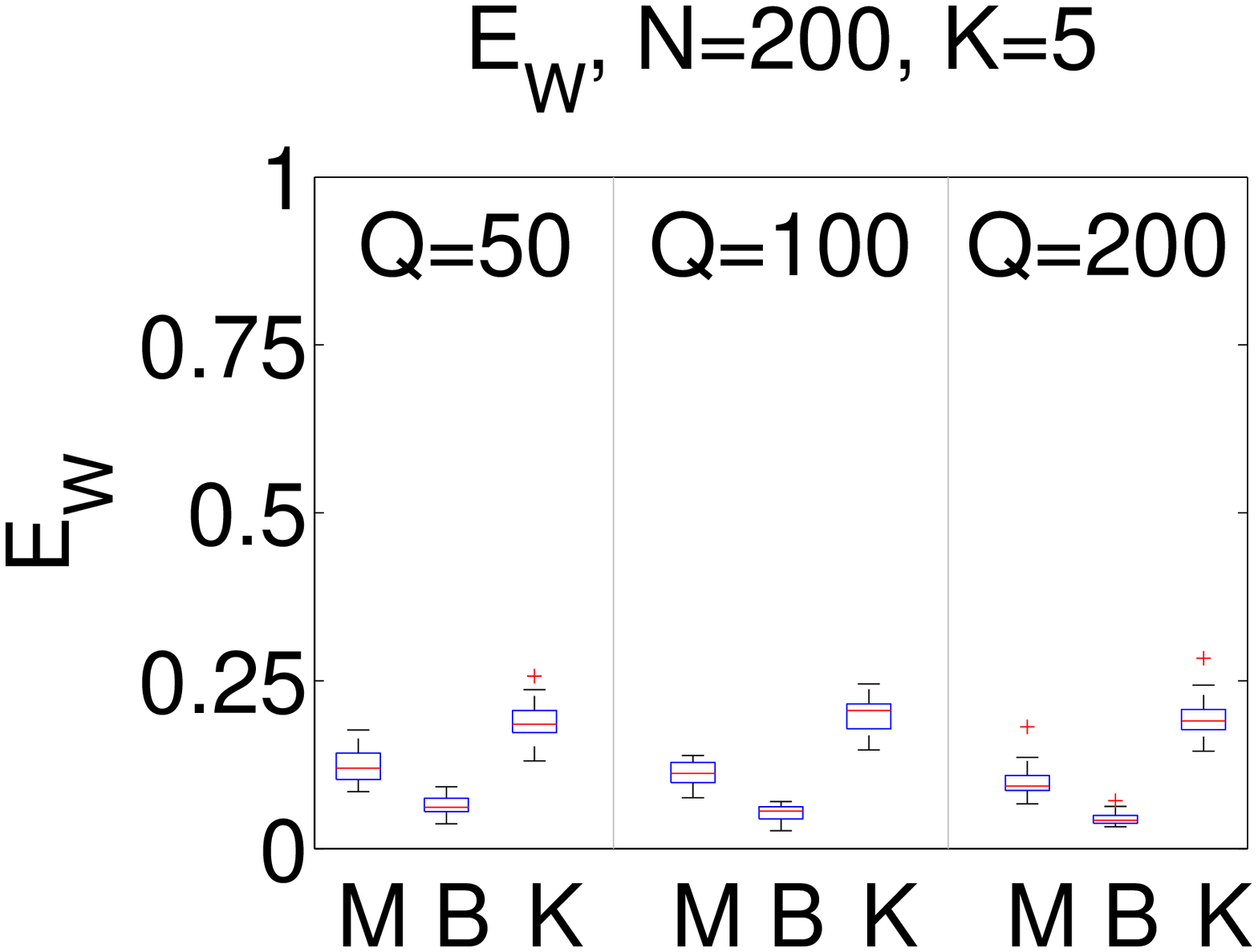}} \hspace{-0.1cm}
\subfigure{\includegraphics[width=0.245\textwidth]{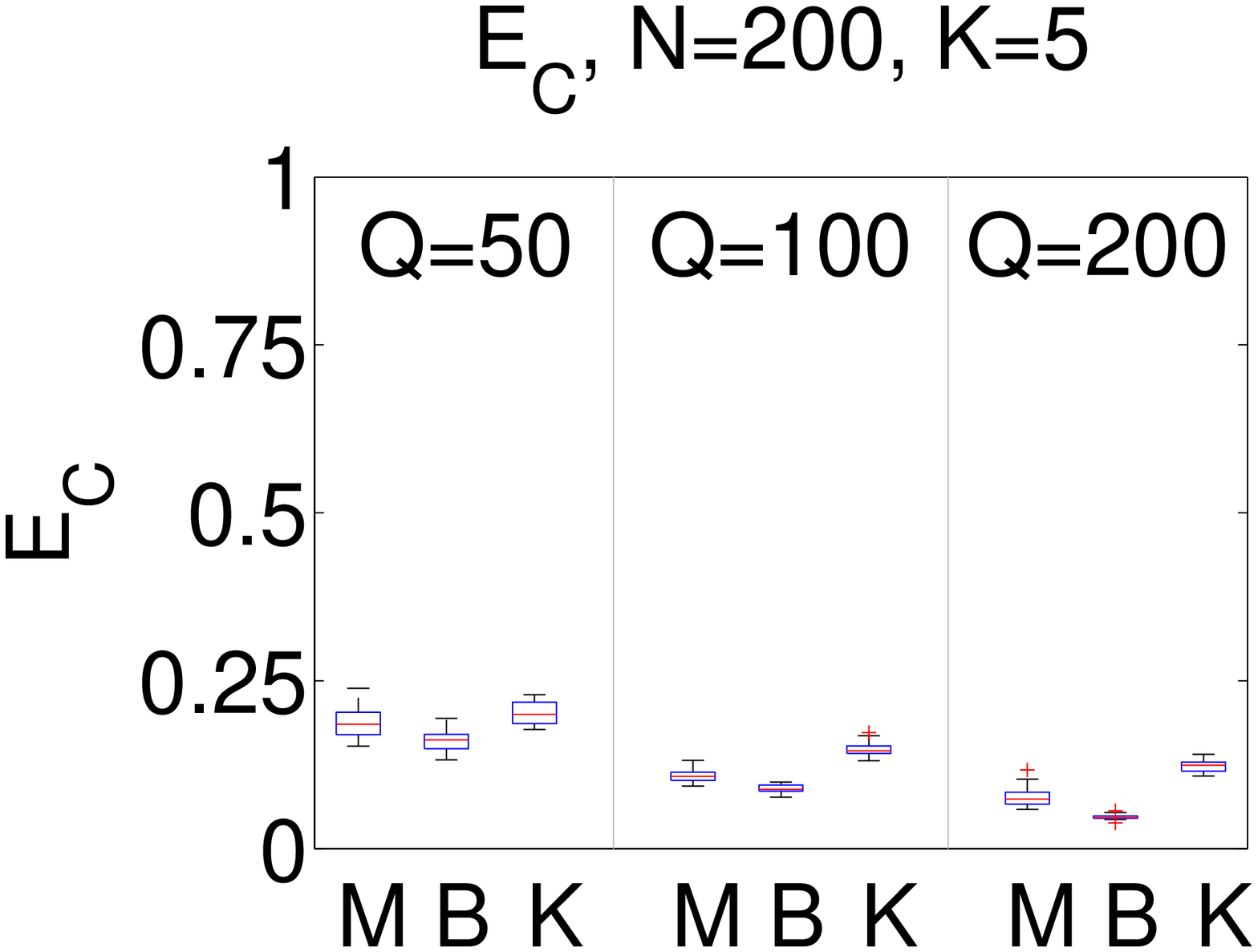}}\hspace{-0.1cm}
\subfigure{\includegraphics[width=0.245\textwidth]{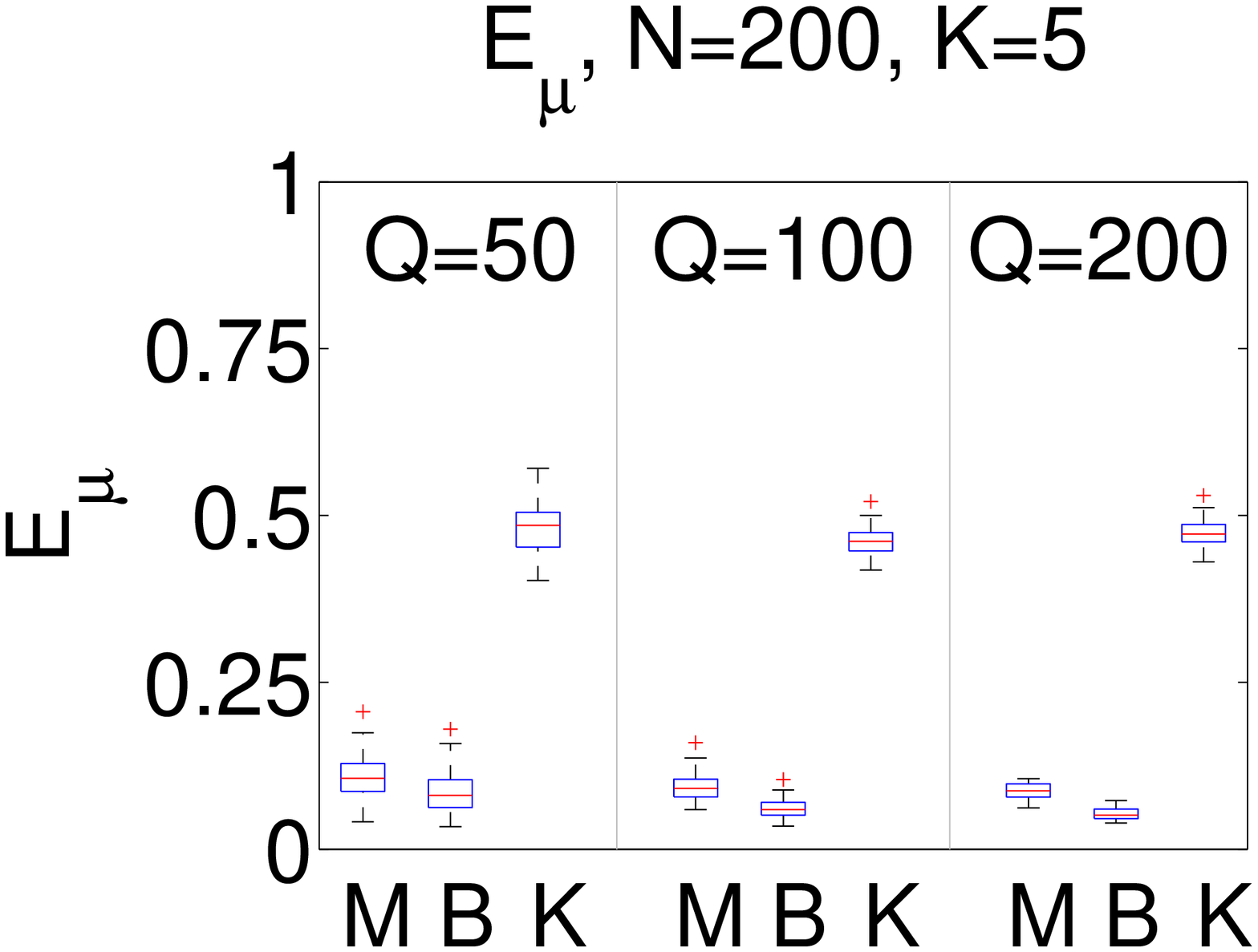}}
\hspace{-0.1cm}
\subfigure{\includegraphics[width=0.245\textwidth]{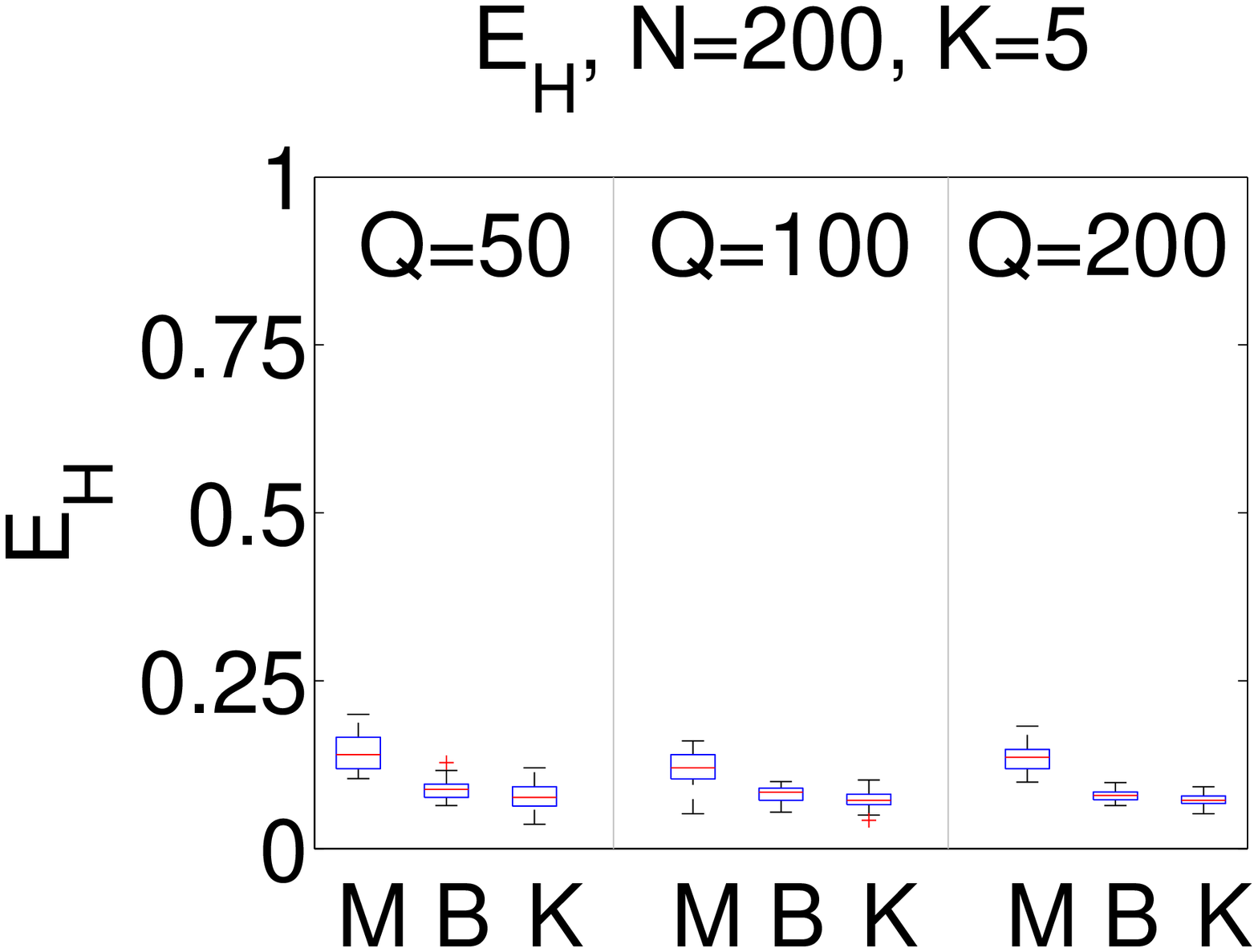}}
\addtocounter{subfigure}{-3}
\vspace{-0.0cm}
\subfigure{\includegraphics[width=0.245\textwidth]{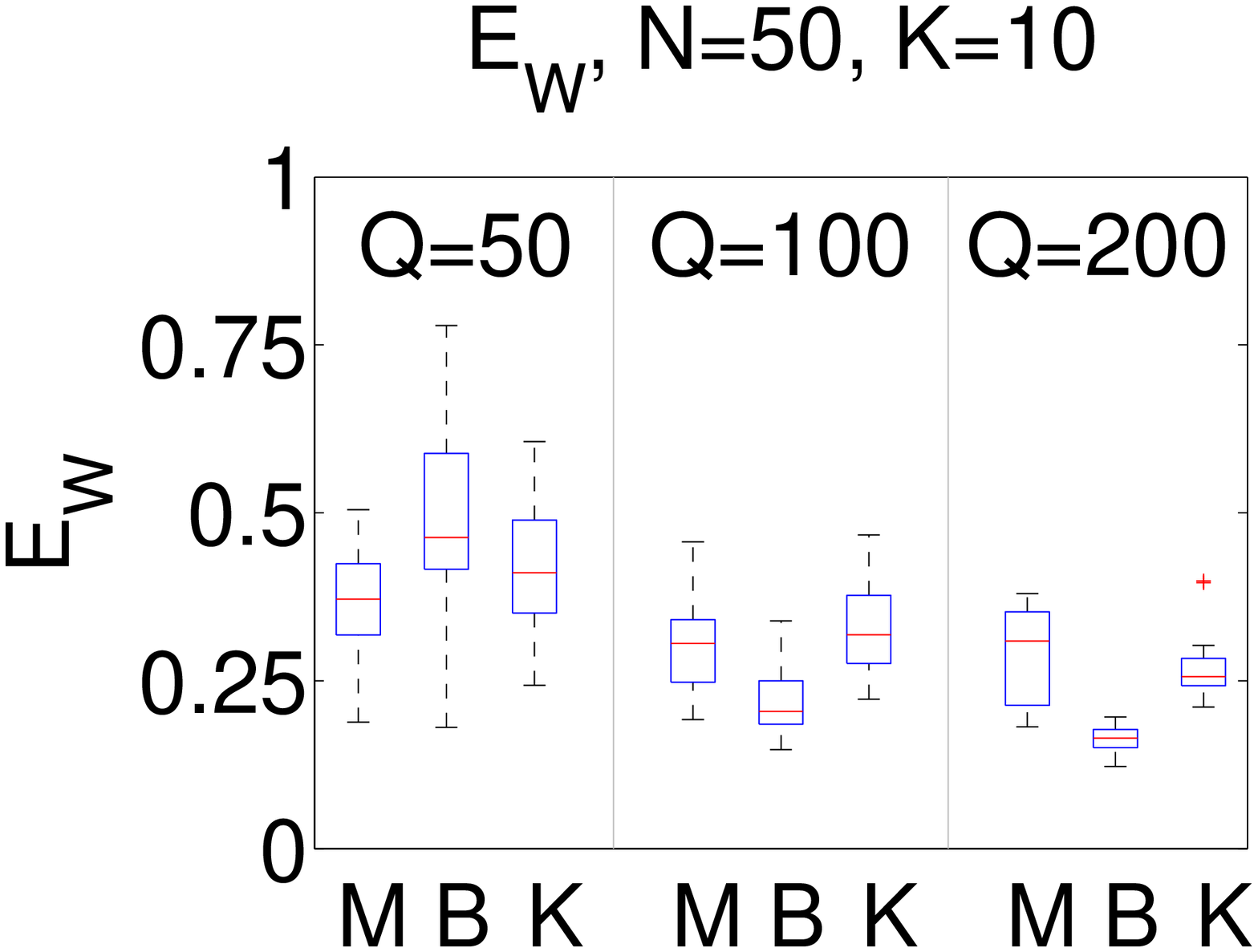}} \hspace{-0.1cm}
\subfigure{\includegraphics[width=0.245\textwidth]{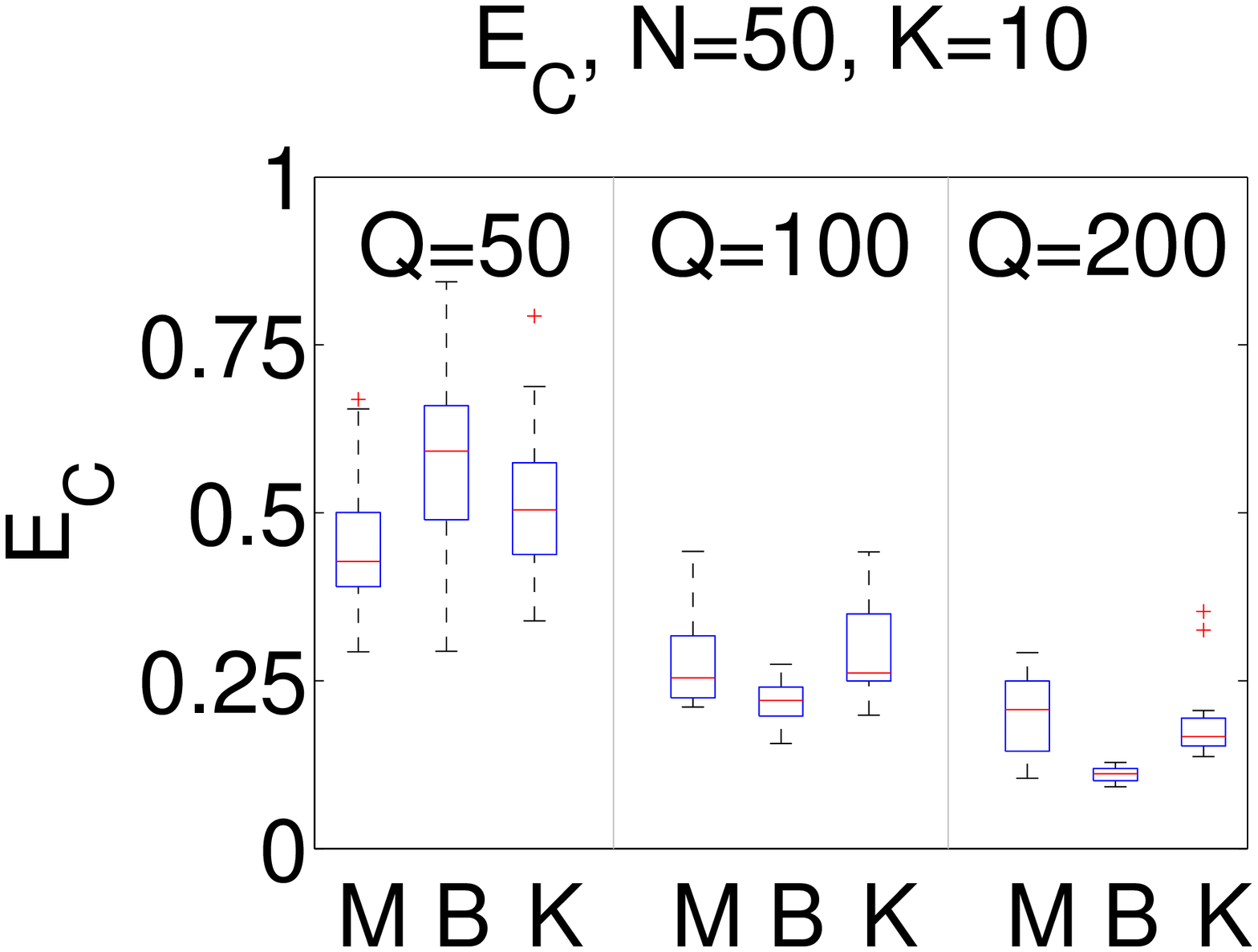}}\hspace{-0.1cm}
\subfigure{\includegraphics[width=0.245\textwidth]{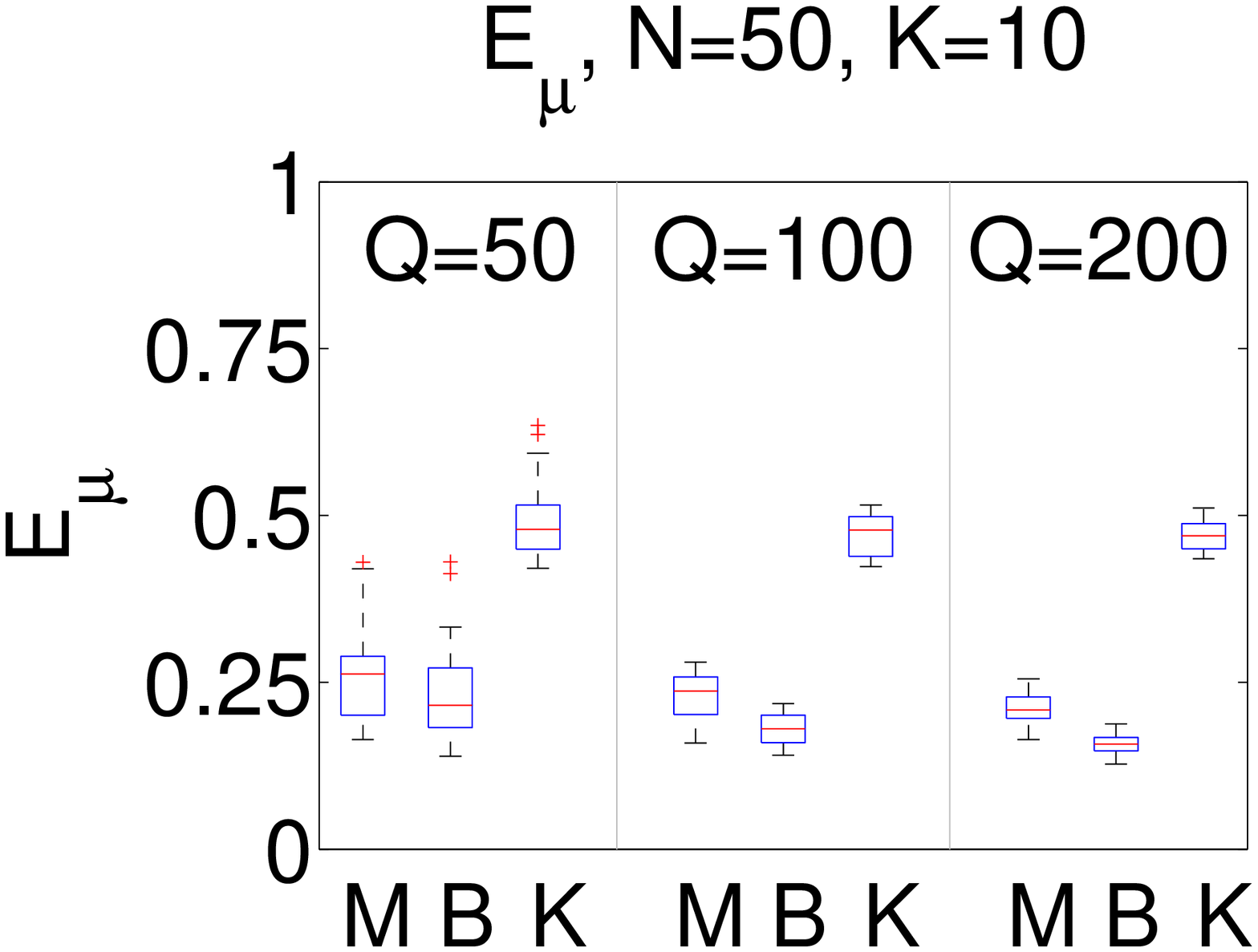}}
\hspace{-0.1cm}
\subfigure{\includegraphics[width=0.245\textwidth]{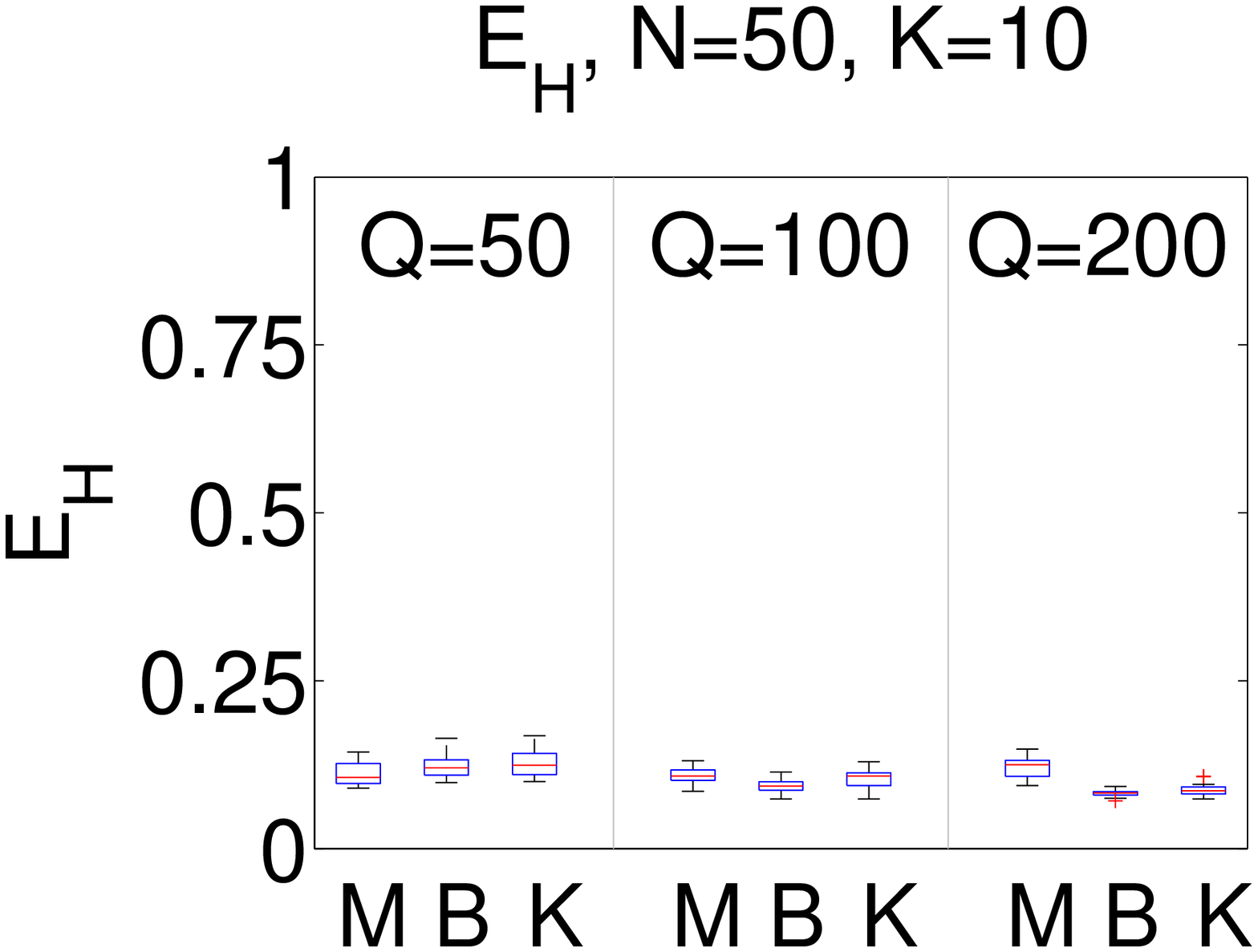}}
\addtocounter{subfigure}{-3}
\vspace{-0.0cm}
\subfigure{\includegraphics[width=0.245\textwidth]{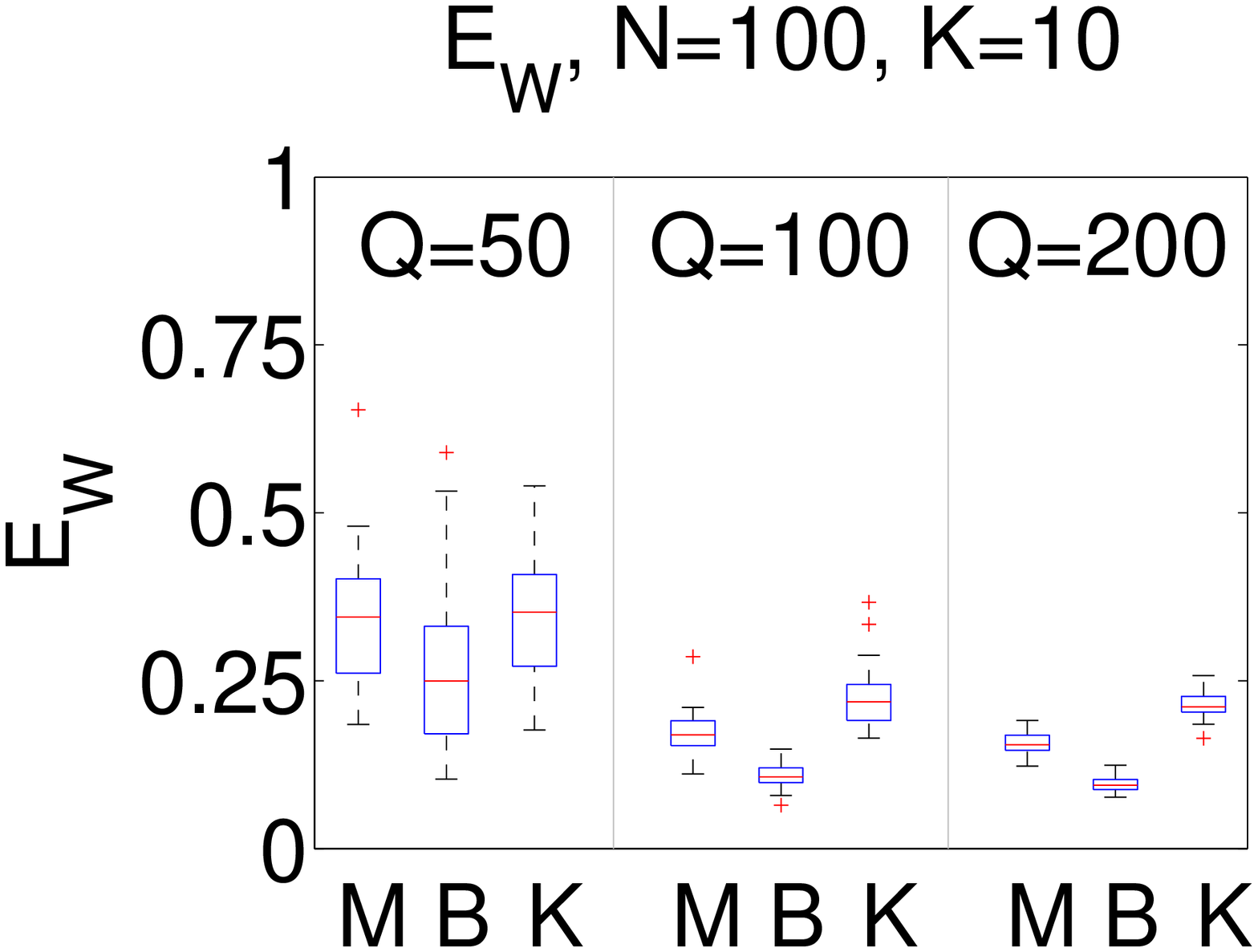}} \hspace{-0.1cm}
\subfigure{\includegraphics[width=0.245\textwidth]{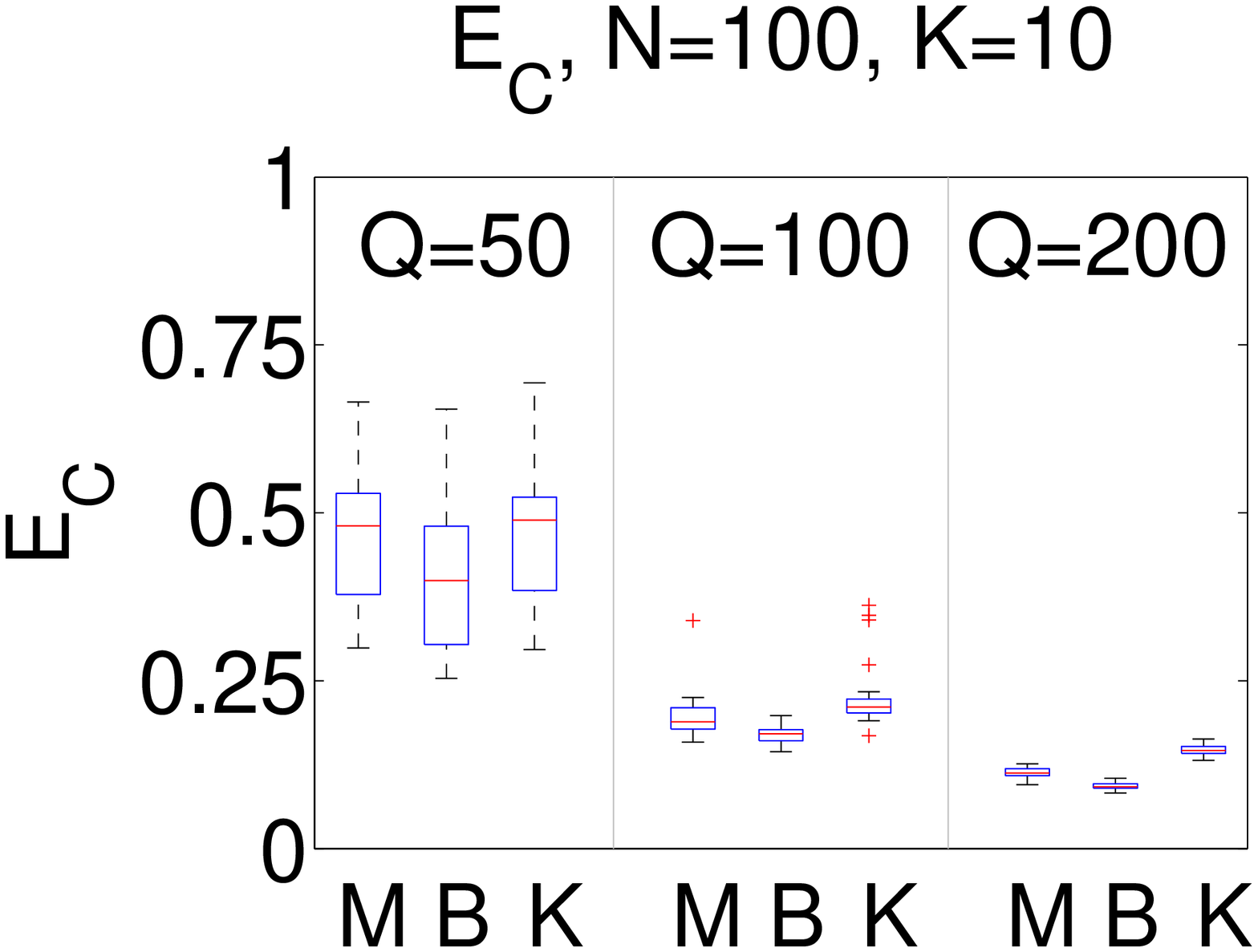}}\hspace{-0.1cm}
\subfigure{\includegraphics[width=0.245\textwidth]{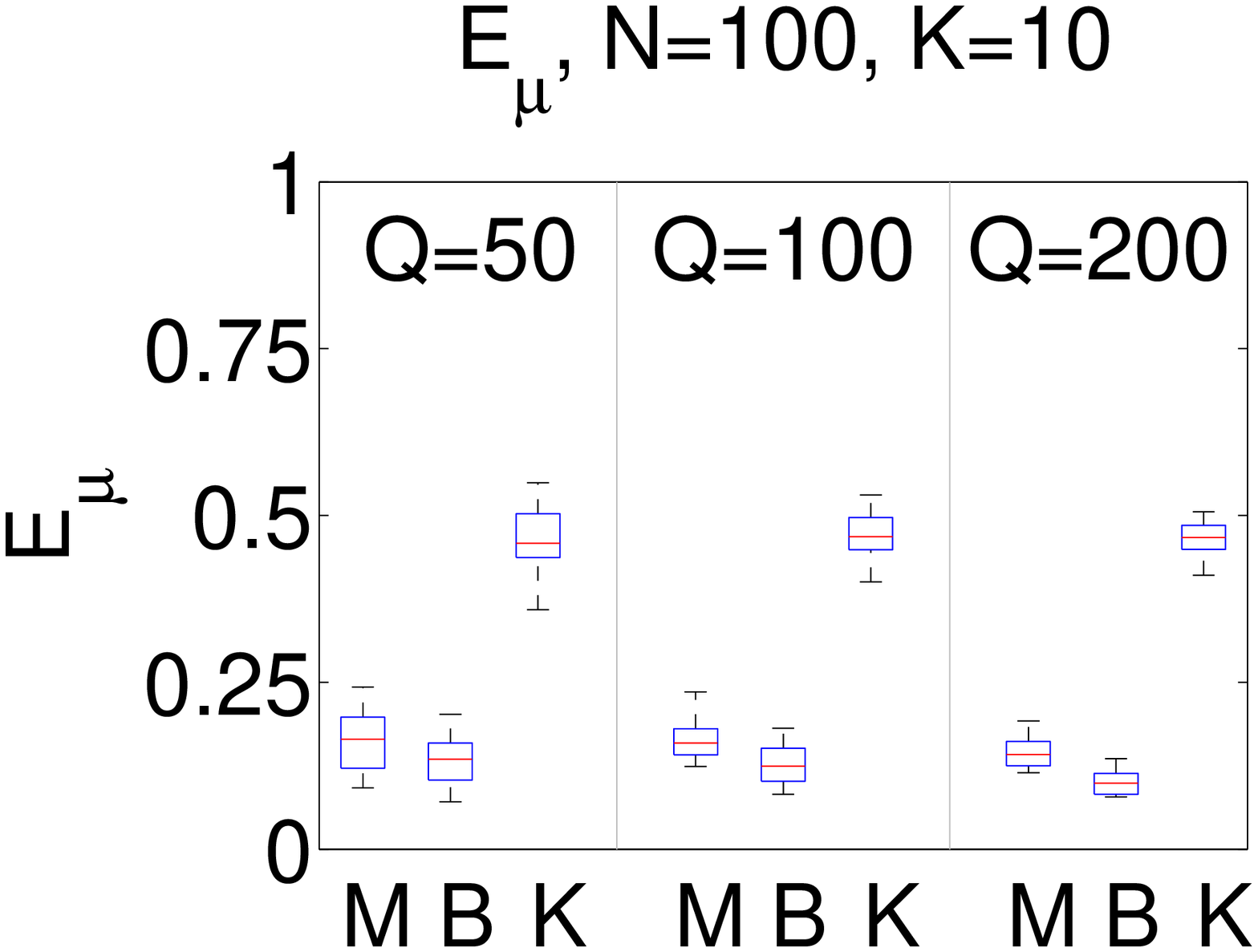}}
\hspace{-0.1cm}
\subfigure{\includegraphics[width=0.245\textwidth]{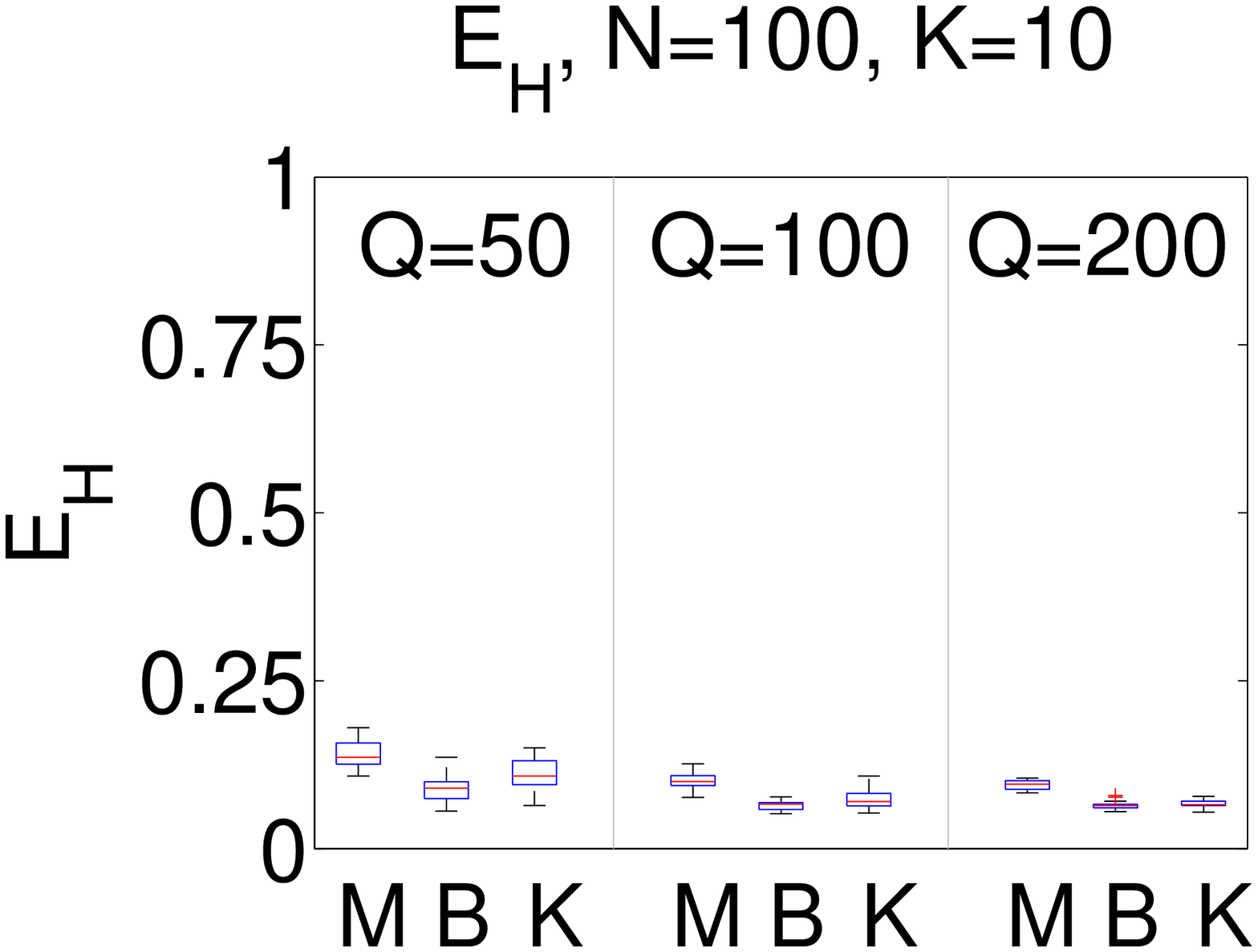}}
\addtocounter{subfigure}{-3}
\vspace{-0.0cm}
\subfigure{\includegraphics[width=0.245\textwidth]{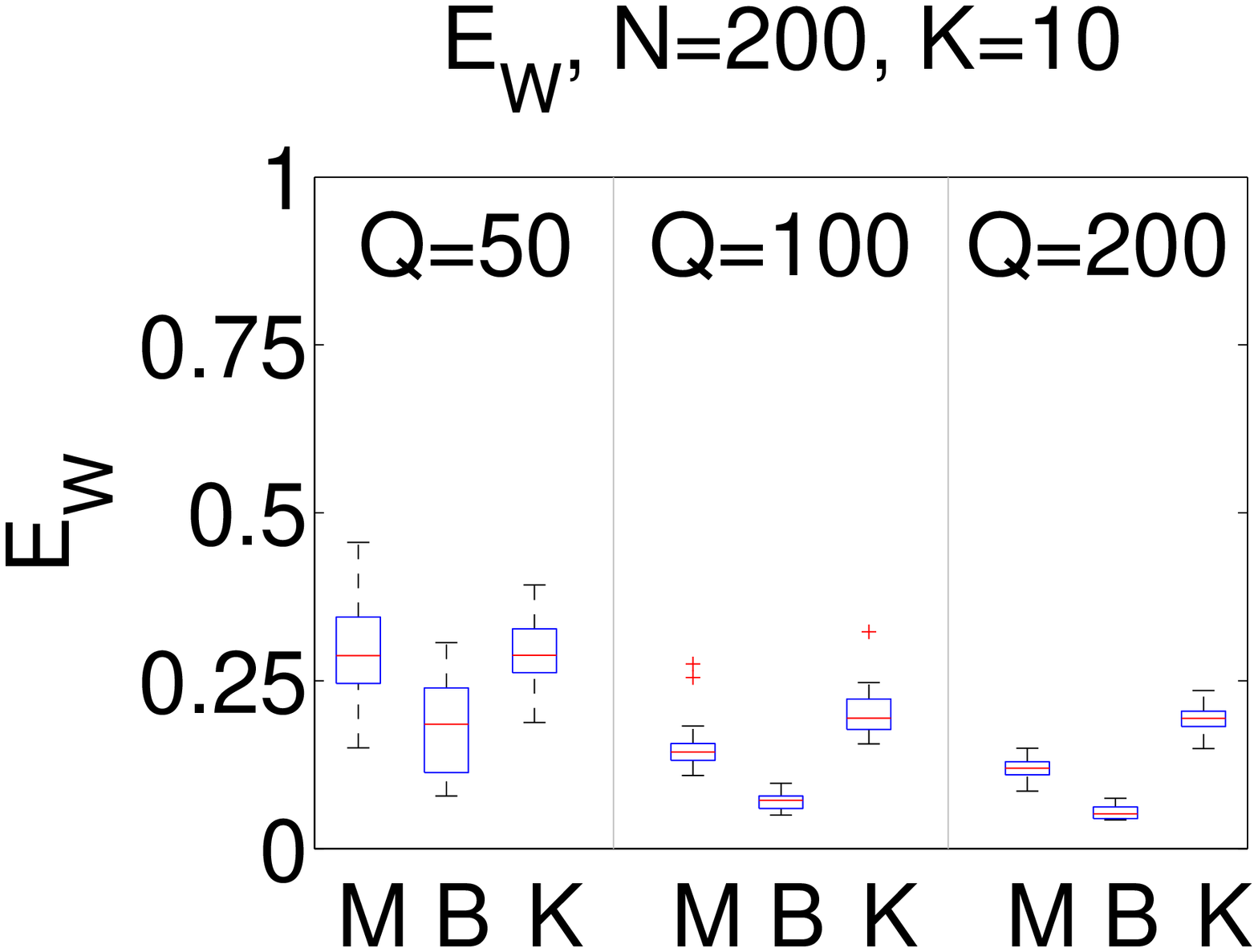}} \hspace{-0.1cm}
\subfigure{\includegraphics[width=0.245\textwidth]{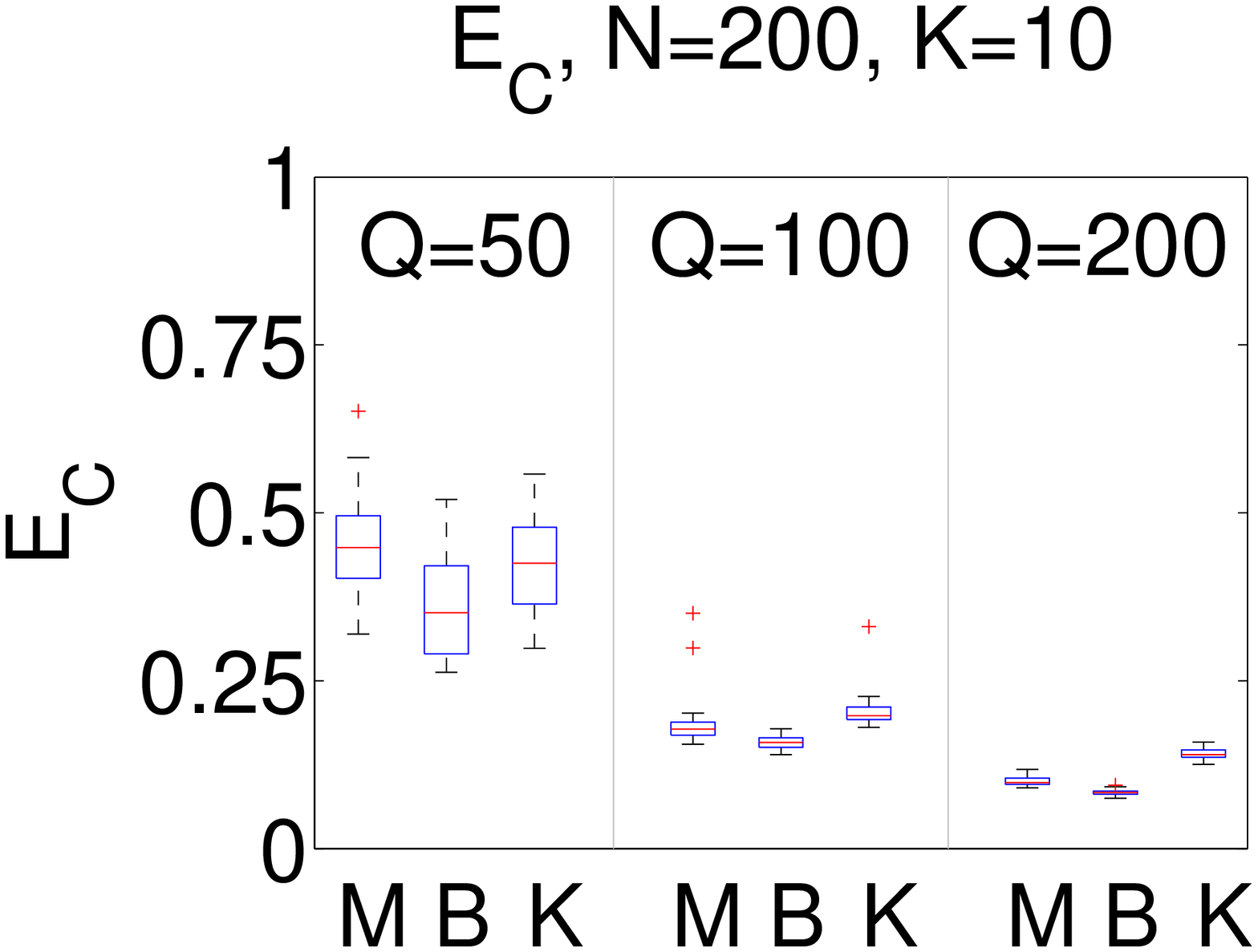}}\hspace{-0.1cm}
\subfigure{\includegraphics[width=0.245\textwidth]{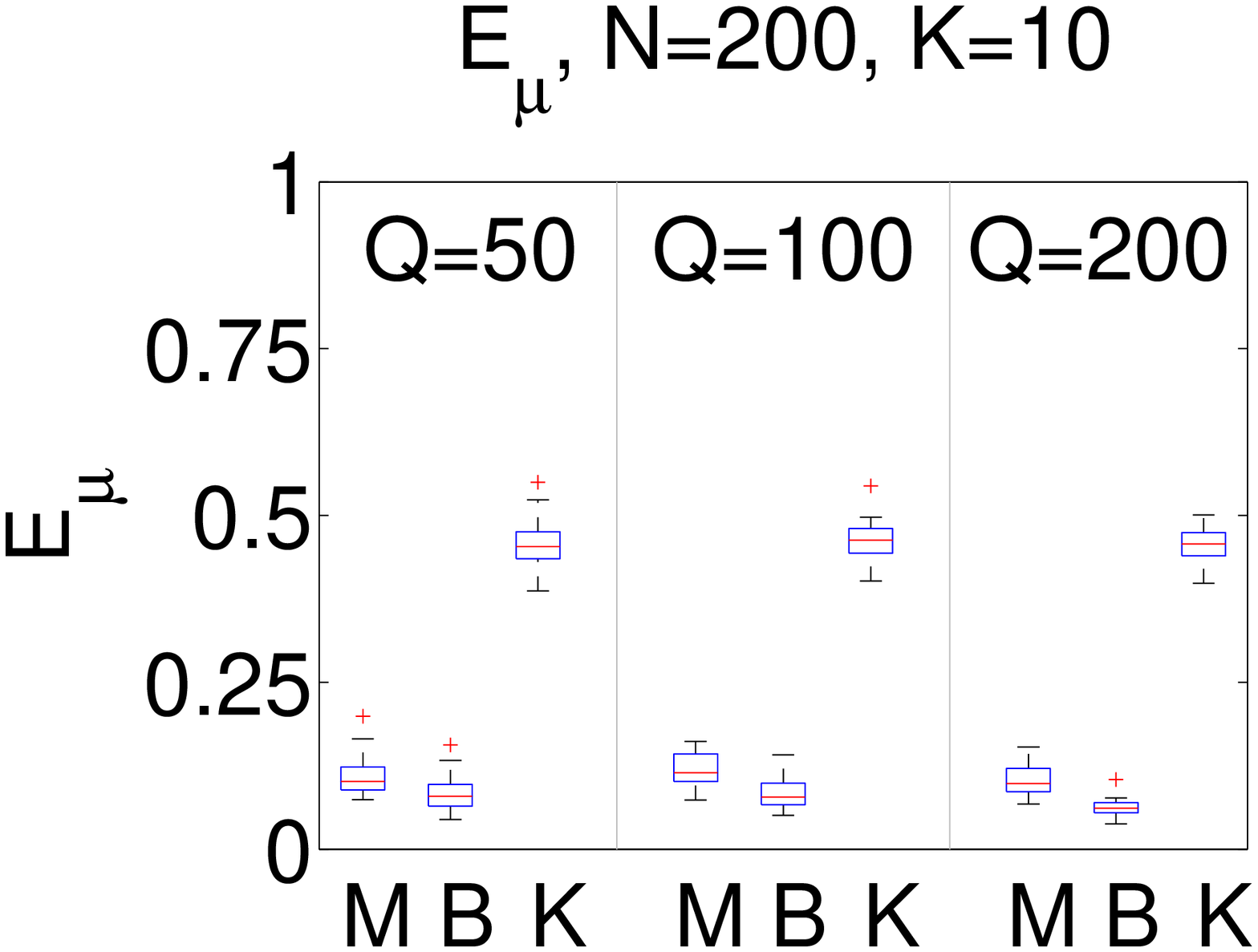}}
\hspace{-0.1cm}
\subfigure{\includegraphics[width=0.245\textwidth]{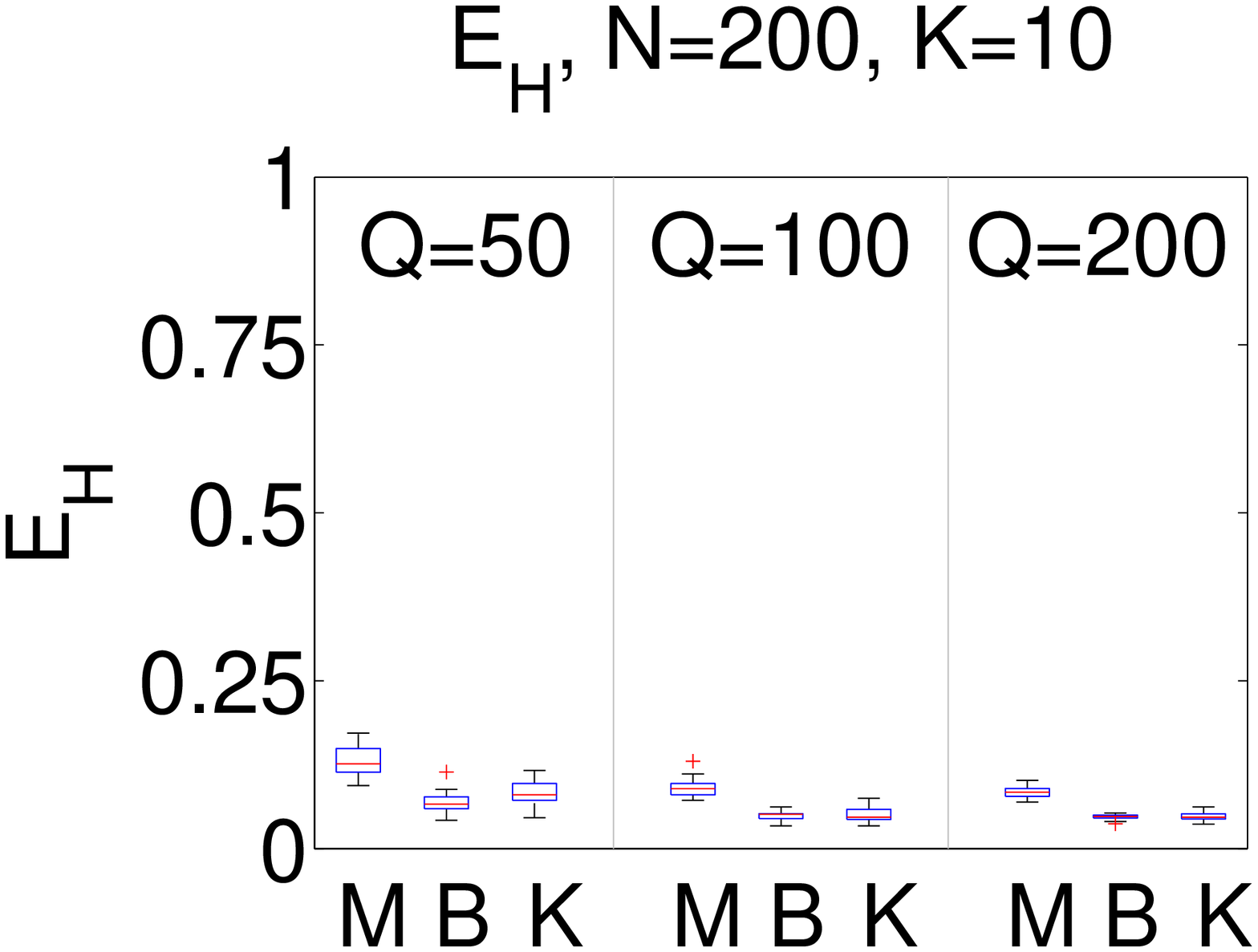}}
\vspace{-0.5cm}
  \caption{Performance comparison of SPARFA-M, SPARFA-B, and K-SVD$_+$ for different problem sizes $Q\times N$ and number of concepts $K$. The performance naturally improves as the problem size increases, while both SPARFA algorithms outperform K-SVD$_{+}$. \textsf{M} denotes SPARFA-M, \textsf{B} denotes SPARFA-B, and \textsf{K} denotes KSVD$_+$.}
\vspace{-1.0cm}
\label{fig:synth_size_bayes}
\end{figure}

\subsubsection{Impact of the Number of Incomplete Observations}
\label{sec:synthpunc}

In this experiment, we study the impact of the number of observations in $\bY$ on the performance of the probit version of SPARFA-M, SPARFA-B, and K-SVD$_+$. 

\paragraph{Experimental setup}

We set $N=Q=100$, $K=5$, and all other parameters as in \fref{sec:synthfull}.
We then vary the percentage $P_\text{obs}$ of entries in $\bY$ that are observed as $100\%$, $80\%$, $60\%$, $40\%$, and $20\%$. 
The locations of missing entries are generated i.i.d.\ and uniformly over the entire matrix.  
\paragraph{Results and discussion}

\fref{fig:synth_punc_bayes} shows that the estimation performance of all methods degrades gracefully as the percentage of missing observations increases. 
Again, SPARFA-B outperforms the other algorithms on $E_\bW$, $E_\bC$, and $E_{\boldsymbol{\mu}}$.  K-SVD$_+$ performs worse than both SPARFA algorithms except on $E_\bH$, where it achieves comparable performance. 
We conclude that SPARFA-M and SPARFA-B can both reliably estimate the underlying factors, even in cases of highly incomplete data. 

\begin{figure}[tp]
\centering

\subfigure{\includegraphics[width=0.245\textwidth]{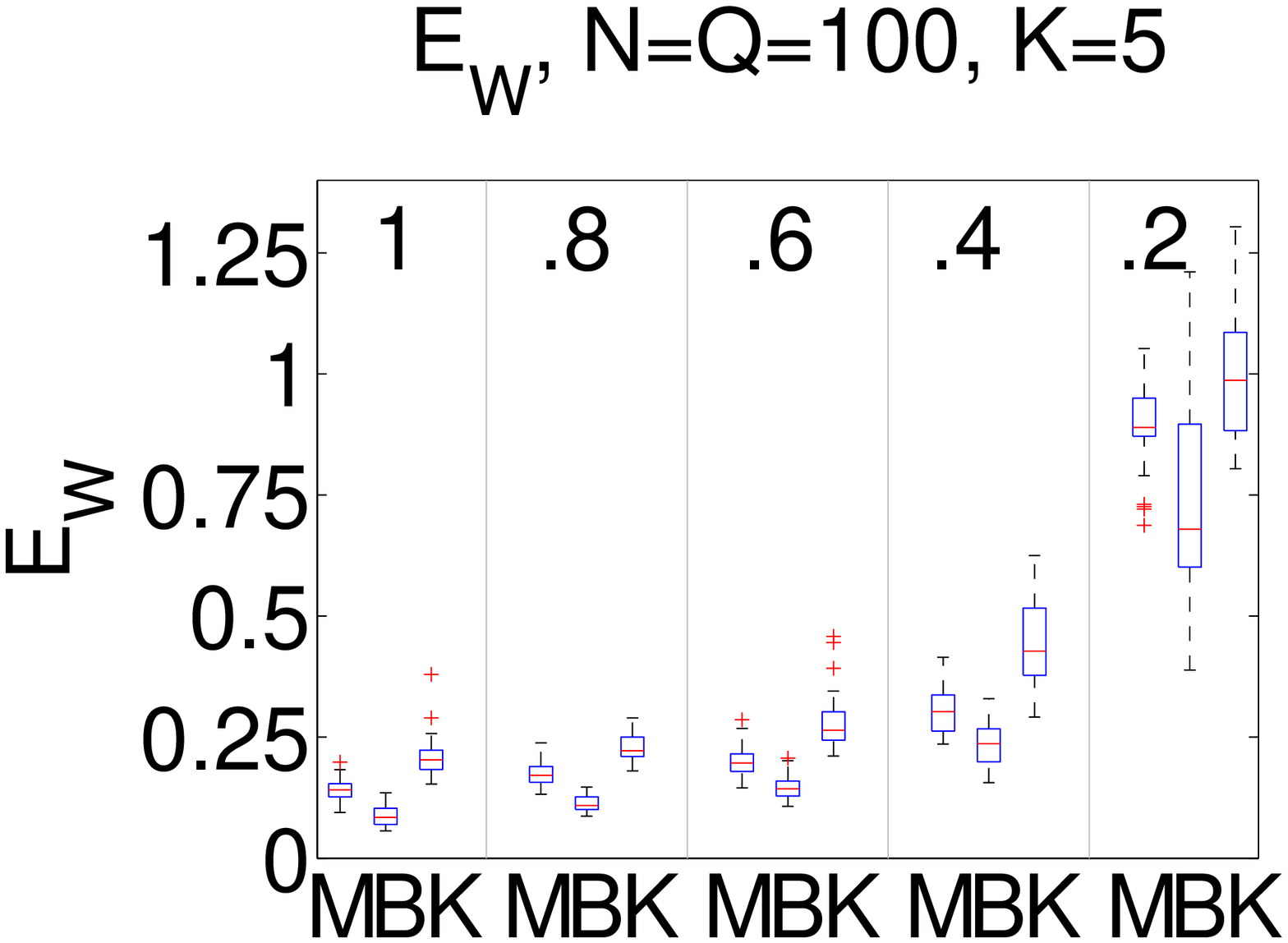}} \hspace{-0.1cm}
\subfigure{\includegraphics[width=0.245\textwidth]{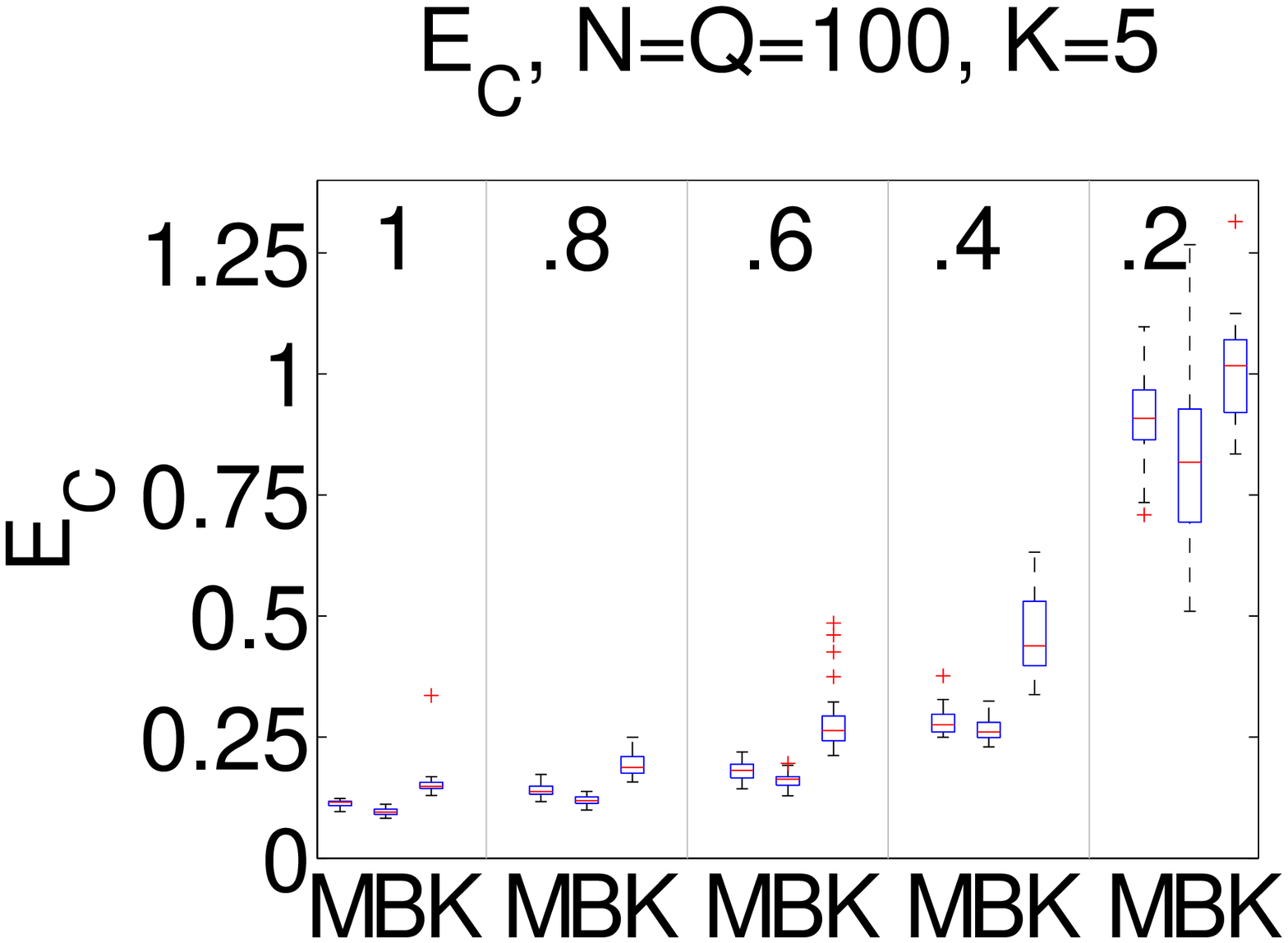}}\hspace{-0.1cm}
\subfigure{\includegraphics[width=0.245\textwidth]{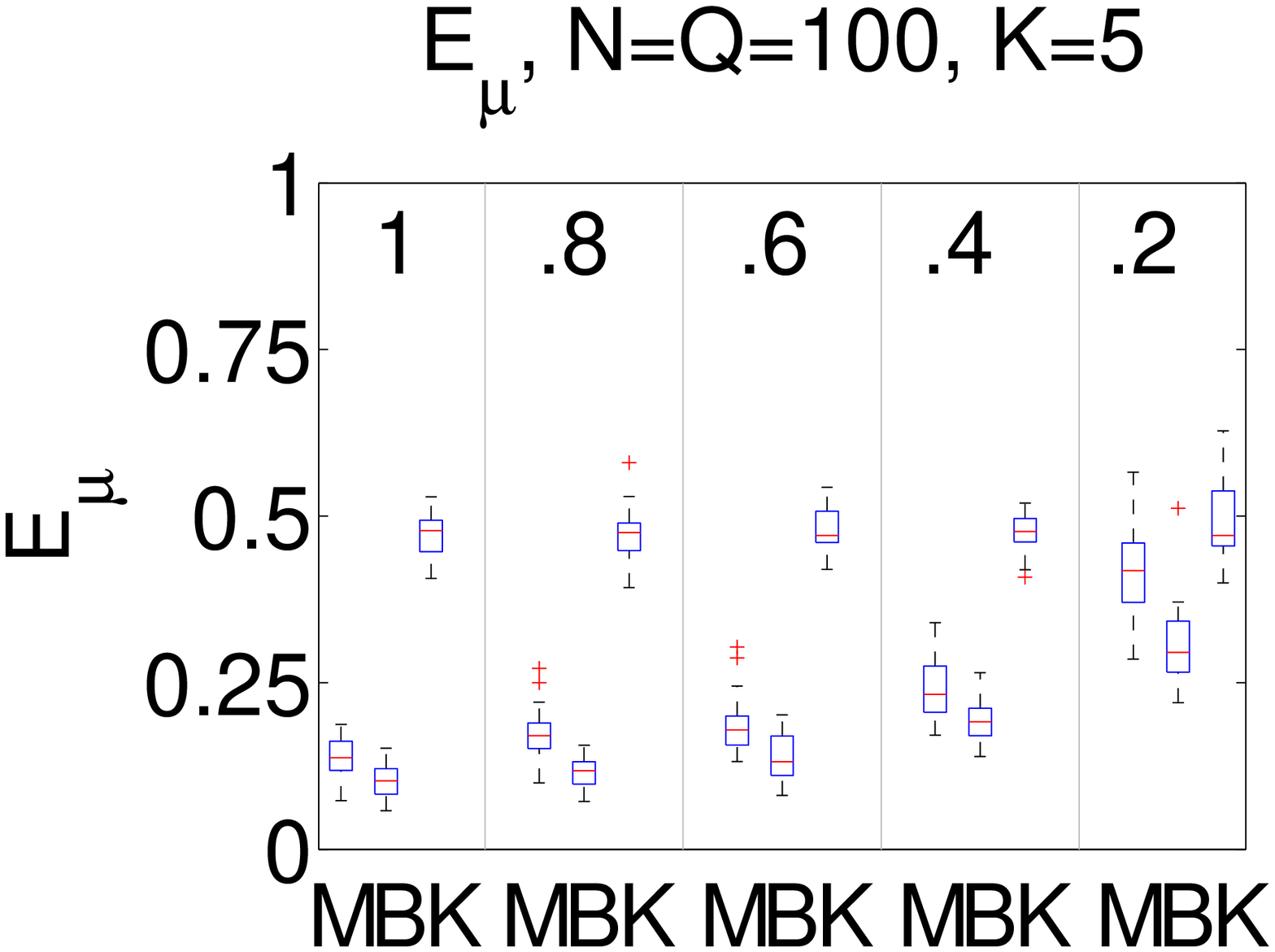}}
\hspace{-0.1cm}
\subfigure{\includegraphics[width=0.245\textwidth]{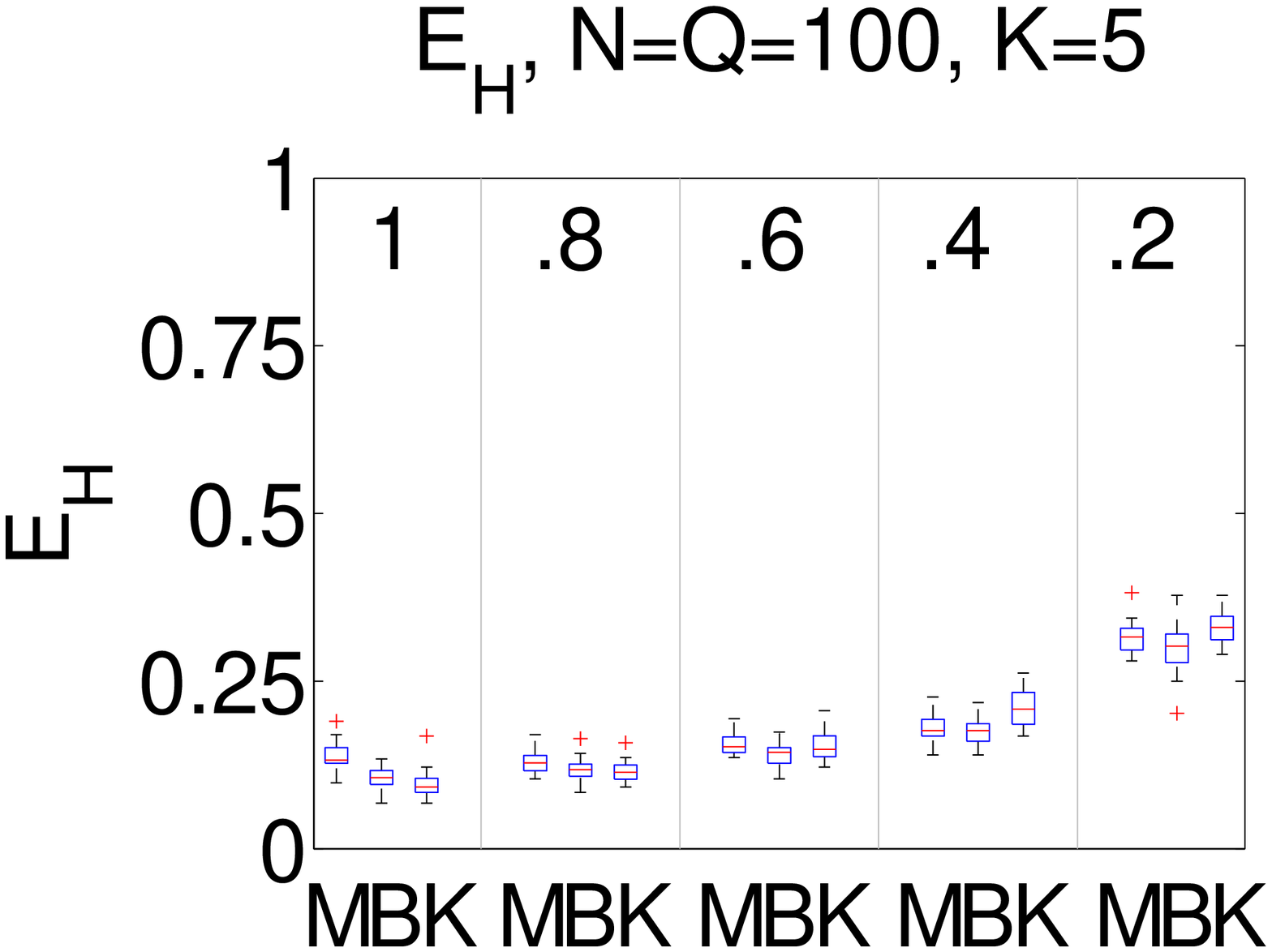}}
\addtocounter{subfigure}{-3}
\vspace{-0.6cm}
  \caption{Performance comparison of SPARFA-M, SPARFA-B, and K-SVD$_+$ for different percentages of observed entries in $\mathbf{Y}$. The performance degrades gracefully as the number of observations decreases, while the SPARFA algorithms outperform K-SVD$_{+}$.}
\vspace{-0.4cm}
\label{fig:synth_punc_bayes}
\end{figure}

\subsubsection{Impact of Sparsity Level}
\label{sec:synthspar}

In this experiment, we study the impact of the sparsity level in $\bW$ on the performance of the probit version of SPARFA-M, SPARFA-B, and K-SVD$_+$. 

\paragraph{Experimental setup}

We choose the active entries of $\bW$ i.i.d.\ $\textit{Ber}(q)$ and vary $q \in \{0.2, 0.4, 0.6, 0.8\}$ to control the number of non-zero entries in each row of $\bW$. All other parameters are set as in \fref{sec:synthfull}.
This data-generation method allows for scenarios in which some rows of $\bW$ contain no active entries as well as all active entries.
We set the hyperparameters for SPARFA-B to $h=K+1=6$, $v_\mu = 1$, and $e=1$, and $f=1.5$. For $q=0.2$ we set $\alpha = 2$ and $\beta = 5$. For $q=0.8$ we set $\alpha = 5$ and $\beta = 2$. For all other cases, we set $\alpha=\beta=2$.   

\paragraph{Results and discussion}

\fref{fig:synth_sparse_bayes} shows that sparser $\bW$ lead to lower estimation errors. This demonstrates that the SPARFA algorithms are well-suited to applications where the underlying factors have a high level of sparsity. 
SPARFA-B outperforms SPARFA-M across all metrics. 
The performance of K-SVD$_+$ is worse than both SPARFA algorithms except on the support estimation error $E_\bH$, which is due to the fact that K-SVD$_+$ is aware of the oracle sparsity level. 

\begin{figure*}[t]
\centering

\subfigure{\includegraphics[width=0.245\textwidth]{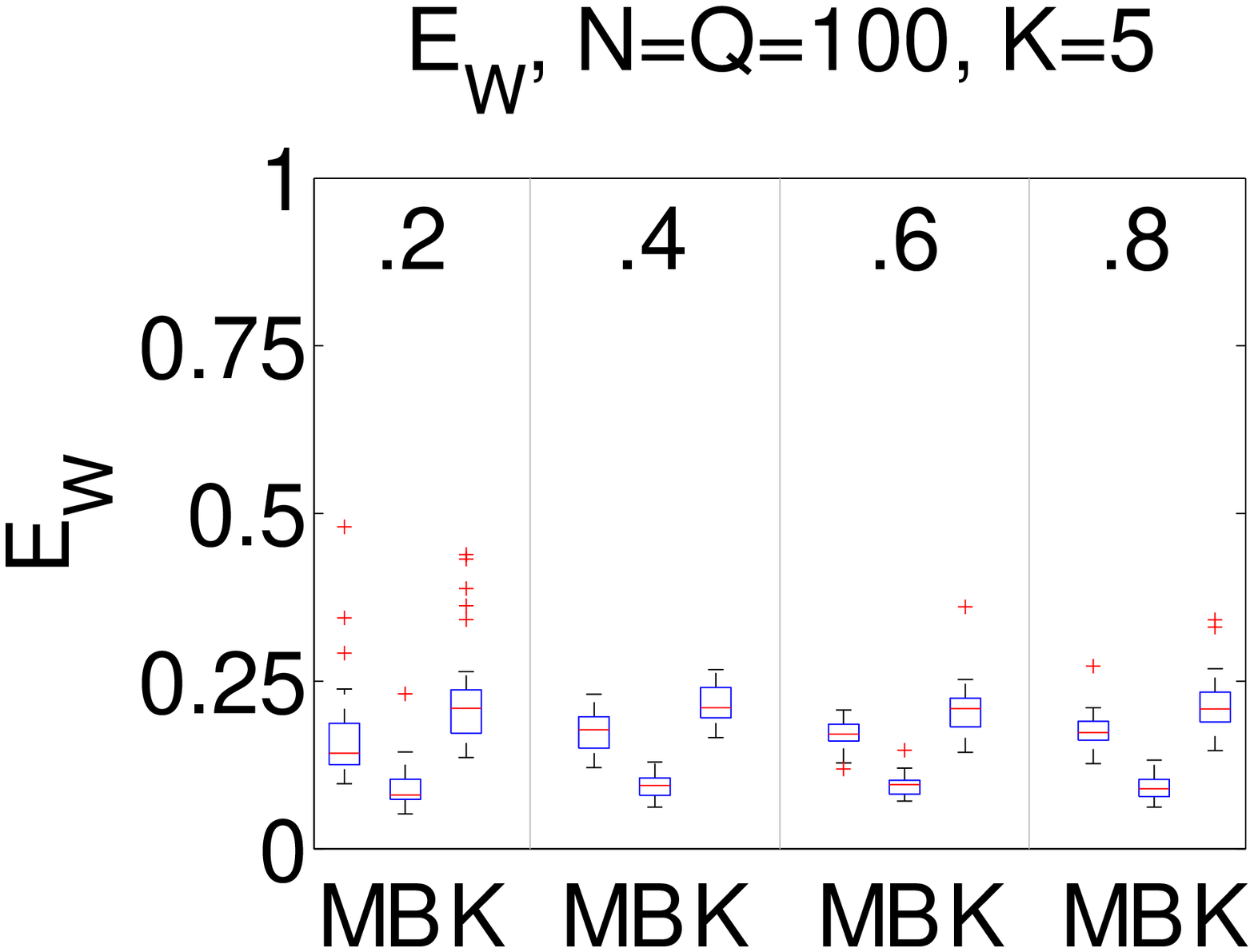}} \hspace{-0.1cm}
\subfigure{\includegraphics[width=0.245\textwidth]{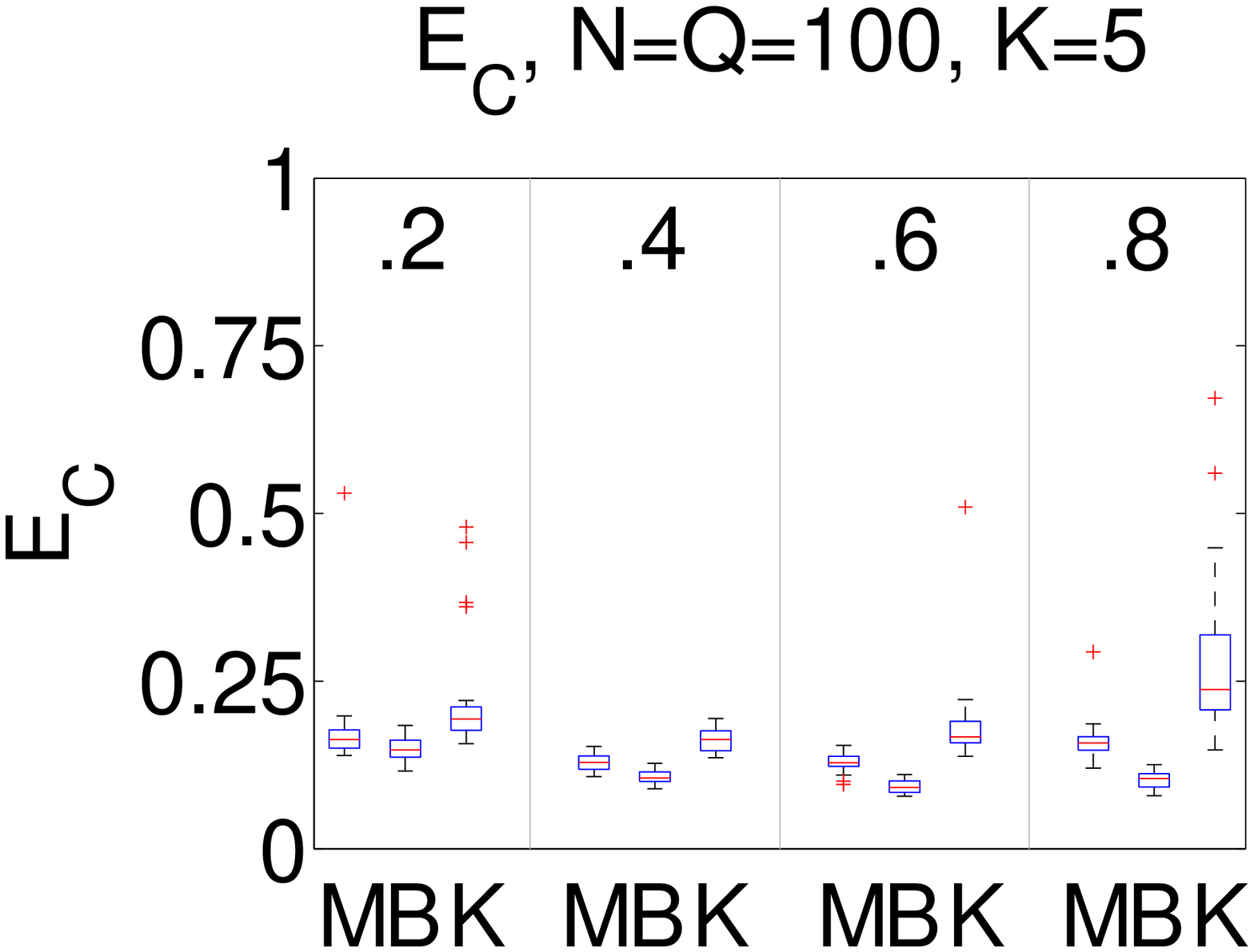}}\hspace{-0.1cm}
\subfigure{\includegraphics[width=0.245\textwidth]{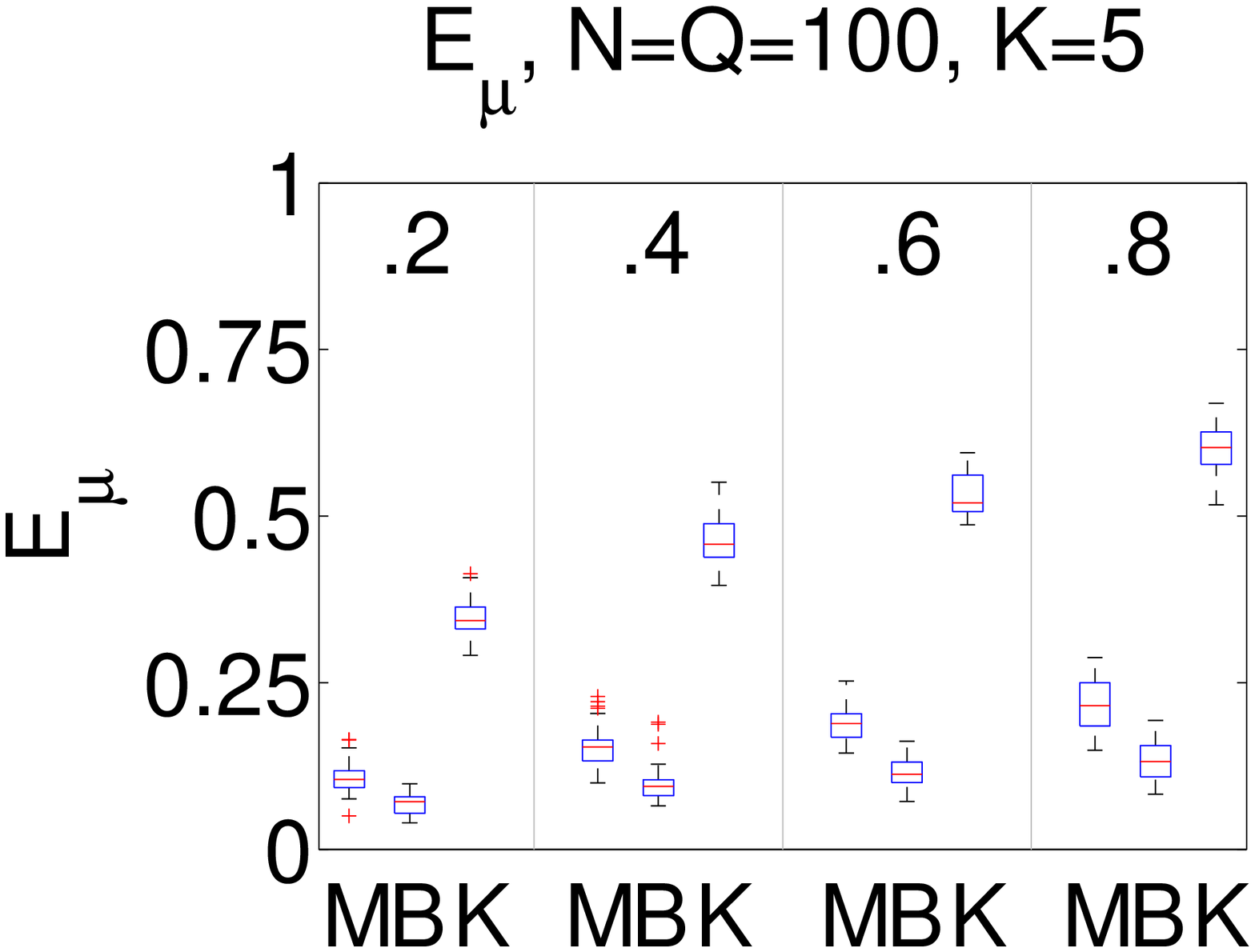}}
\hspace{-0.1cm}
\subfigure{\includegraphics[width=0.245\textwidth]{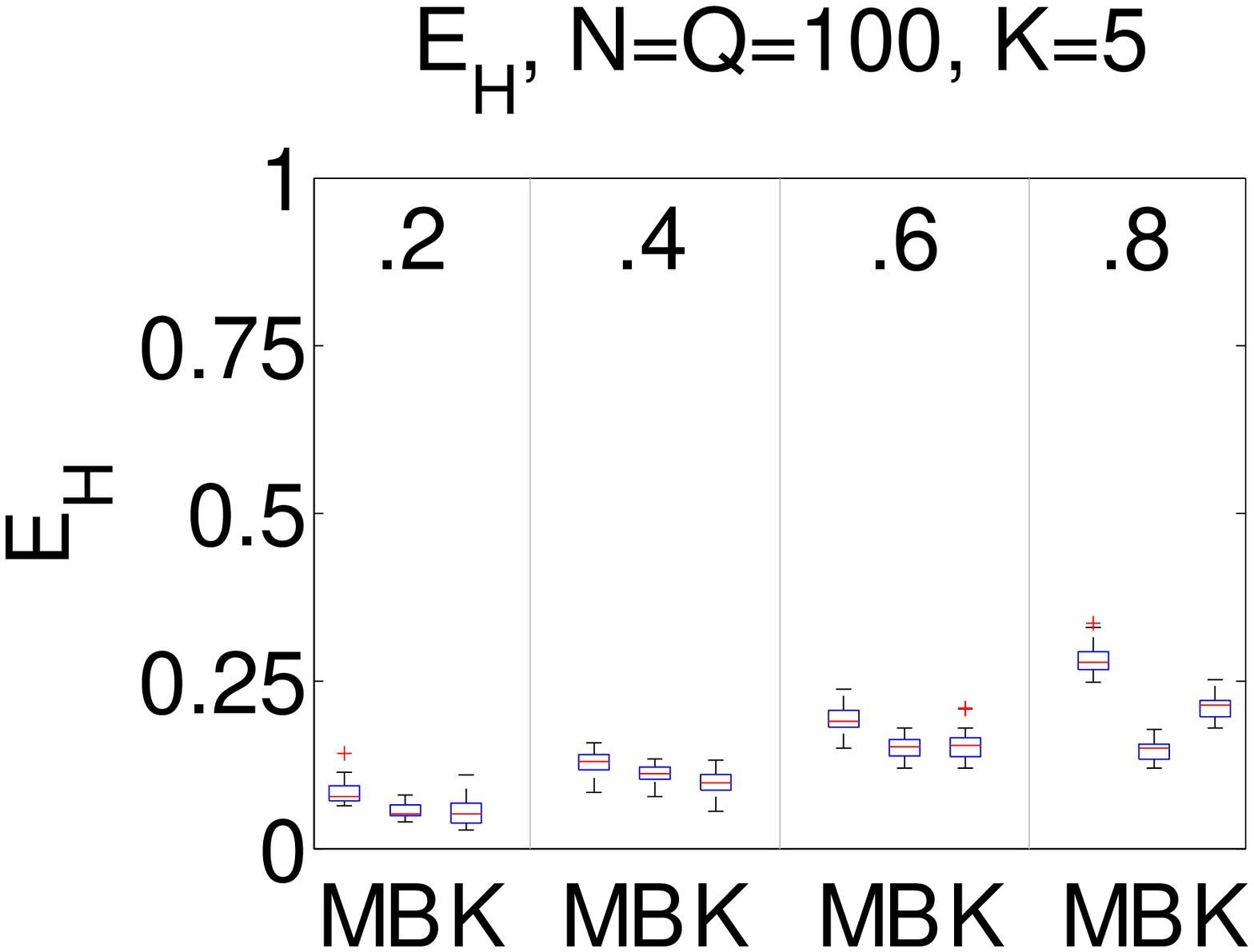}}
\addtocounter{subfigure}{-3}
\vspace{-0.6cm}
  \caption{Performance comparison of SPARFA-M, SPARFA-B, and K-SVD$_+$ for different sparsity levels in the rows in $\mathbf{W}$.  The performance degrades gracefully as the sparsity level increases, while the SPARFA algorithms outperform K-SVD$_{+}$.}
\vspace{-0.4cm}
\label{fig:synth_sparse_bayes}
\end{figure*}

\subsubsection{Impact of Model Mismatch}
\label{sec:synthrob}

In this experiment, we examine the impact of model mismatch by using a link function for estimation that does not match the true link function from which the data is generated.

\paragraph{Experimental setup}

We fix $N=Q=100$ and $K=5$, and set all other parameters as in \fref{sec:synthfull}. Then, for each generated instance of $\bW$, $\bC$ and $\boldsymbol{\mu}$, we generate $\bY_\text{pro}$ and $\bY_\text{log}$ according to both the inverse probit link and the inverse logit link, respectively.  We then run SPARFA-M (both the probit and logit variants), SPARFA-B (which uses only the probit link function), and K-SVD$_+$ on both $\bY_\text{pro}$ and $\bY_\text{log}$.

\paragraph{Results and discussion}

\fref{fig:synth_rob_bayes} shows that model mismatch does not severely affect $E_\bW$, $E_\bC$, and  $E_\bH$ for both SPARFA-M and SPARFA-B. 
However, due to the difference in the functional forms between the probit and logit link functions, model mismatch does lead to an increase in $E_{\boldsymbol{\mu}}$ for both SPARFA algorithms.  We also see that K-SVD$_+$ performs worse than both SPARFA methods, since it ignores the link function.

\begin{figure*}[t]
\centering

\subfigure{\includegraphics[width=0.245\textwidth]{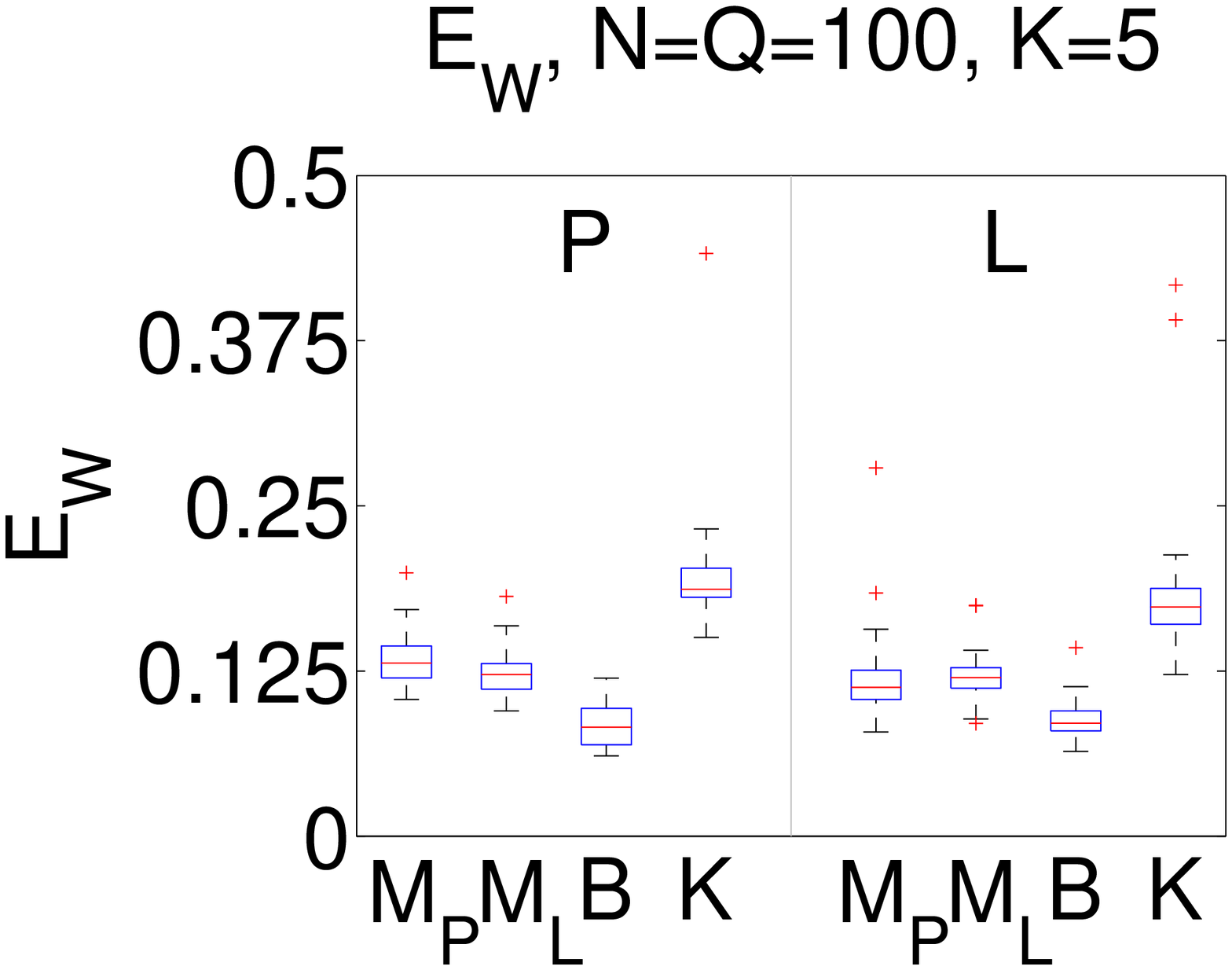}} \hspace{-0.1cm}
\subfigure{\includegraphics[width=0.245\textwidth]{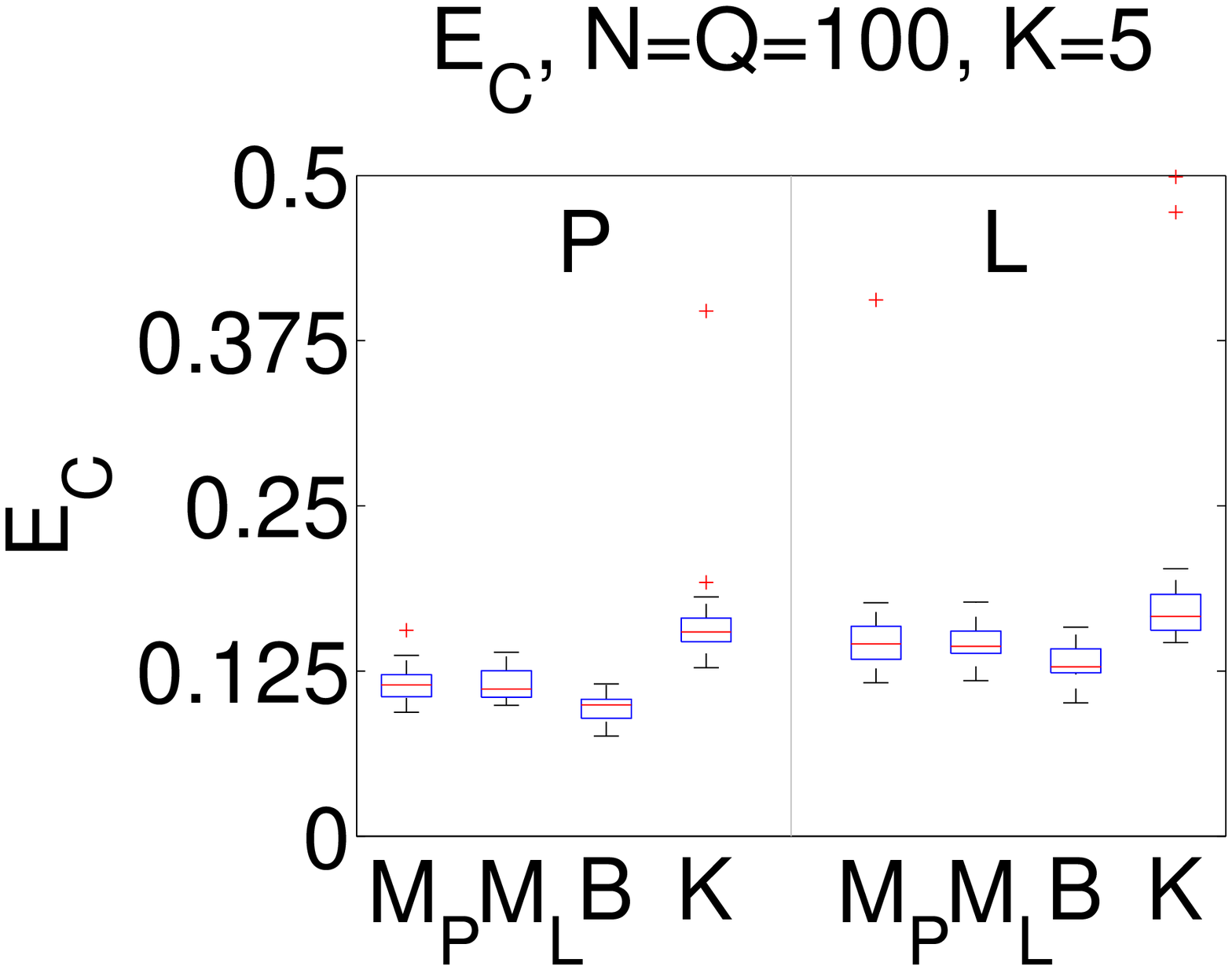}}\hspace{-0.1cm}
\subfigure{\includegraphics[width=0.245\textwidth]{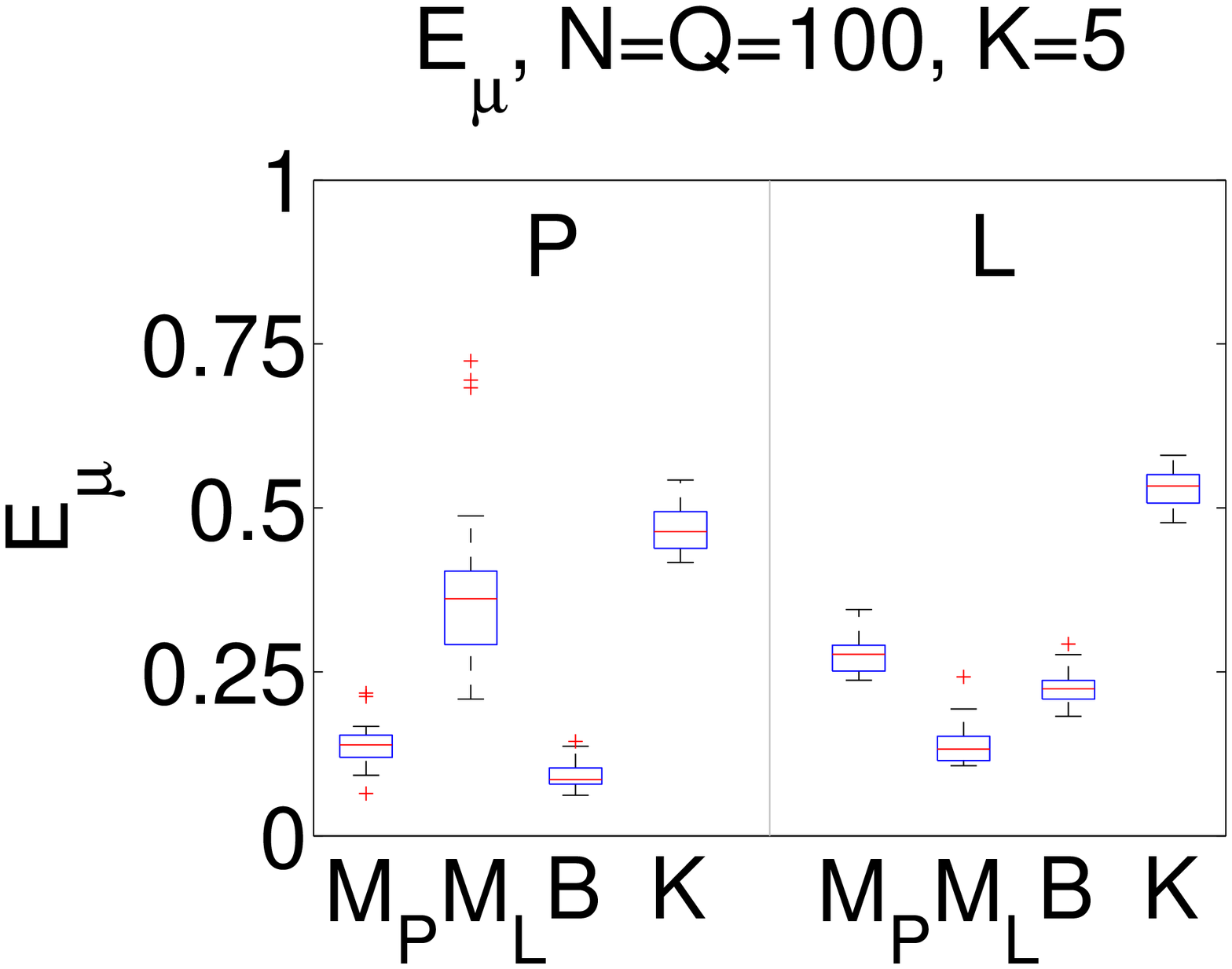}}
\hspace{-0.1cm}
\subfigure{\includegraphics[width=0.245\textwidth]{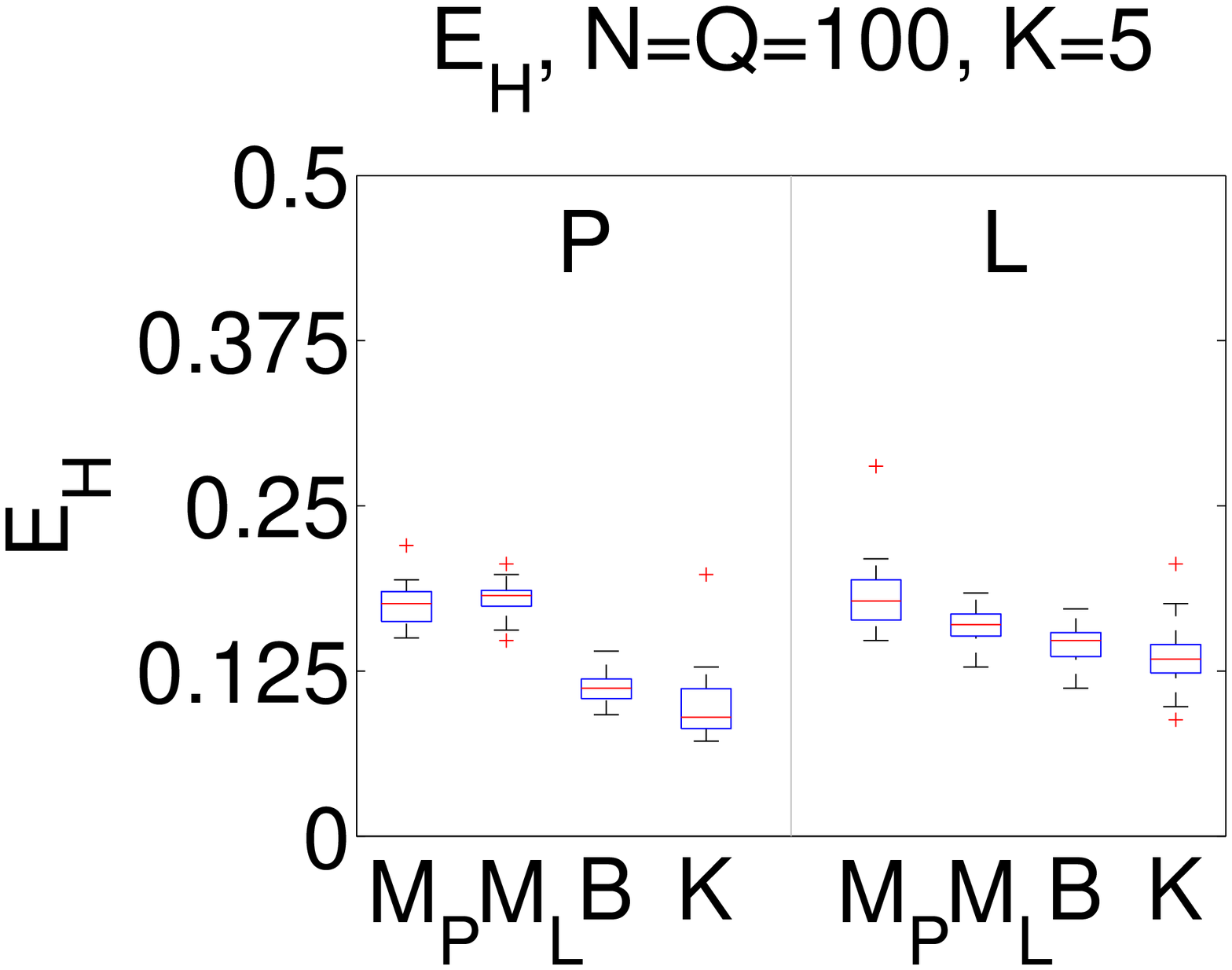}}
\addtocounter{subfigure}{-3}
\vspace{-0.6cm}
  \caption{Performance comparison of SPARFA-M, SPARFA-B, and K-SVD$_+$ with probit/logit model mismatch; $\textsf{M}_\textsf{P}$ and $\textsf{M}_\textsf{L}$ indicate probit and logit SPARFA-M, respectively. In the left/right halves of each box plot, we generate $\bY$ according to the inverse probit/logit link functions. The performance degrades only slightly with mismatch, while both SPARFA algorithms outperform K-SVD$_{+}$.}
\vspace{-0.4cm}
\label{fig:synth_rob_bayes}
\end{figure*}

\subsection{Real Data Experiments}
\label{sec:real}

We next test the SPARFA algorithms on three real-world educational datasets. 
Since all variants of SPARFA-M and SPARFA-B obtained similar results in the sythetic data experiments in \fref{sec:synth}, for the sake of brevity, we will often show the results for only one of the algorithms for each dataset.
In what follows, we select the sparsity penalty parameter~$\lambda$ in SPARFA-M using the BIC as described in \cite{tibsbook} and choose the hyperparameters for SPARFA-B to be largely non-informative. 

\subsubsection{Undergraduate DSP course}
\label{sec:real301}

\paragraph{Dataset}

We analyze a very small dataset consisting of $N=15$ learners answering $Q=44$ questions taken from the final exam of an introductory course on digital signal processing (DSP) taught at Rice University in Fall 2011 (\cite{301website}).  There is no missing data in the matrix $\bY$.

\paragraph{Analysis}

We estimate $\bW$, $\bC$, and $\boldsymbol\mu$ from $\bY$ using the logit version of SPARFA-M assuming $K = 5$ concepts to achieve a concept granularity that matches the complexity of the analyzed dataset.
Since the questions had been manually tagged by the course instructor, we deploy the tag-analysis approach proposed in  \fref{sec:taganalysis}.  Specifically, we form a $44 \times 12$ matrix $\bT$ using the $M=12$ available tags and estimate the $12 \times 5$ concept--tag association matrix $\bA$ in order to interpret the meaning of each retrieved concept. For each concept, we only show the top 3 tags and their relative contributions. We also compute the $12 \times 15$ learner tag knowledge profile matrix $\bU$.

\paragraph{Results and discussion}

\begin{figure}[tp]
\vspace{-1.3cm}
\centering
\subfigure[Question--concept association graph.  Circles correspond to concepts and rectangles to questions; the values in each rectangle corresponds to that question's intrinsic difficulty.]{
\includegraphics[width=0.85\columnwidth]{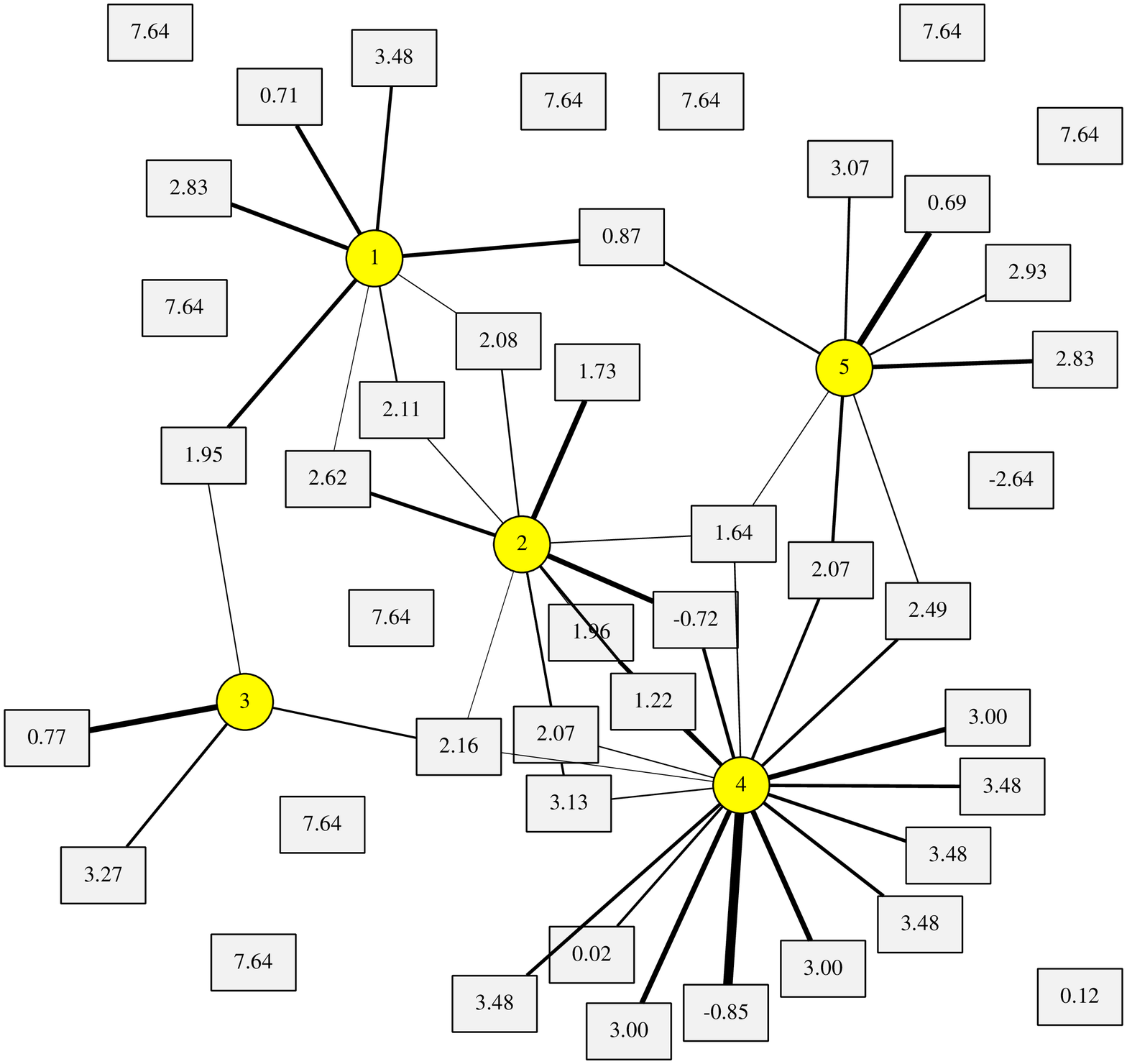}
\label{fig:301W_a}
}\\[0.3cm]
\subfigure[Most important tags and relative weights for the estimated concepts.]{
\scalebox{.9}{%
\begin{tabular}{llllll}
\toprule[0.2em]
Concept 1 && Concept 2 && Concept 3\\
\midrule[0.1em]
Frequency response & (46\%) & Fourier transform & (40\%) & $z$-transform & (66\%)\\
Sampling rate & (23\%) & Laplace transform & (36\%) & Pole/zero plot & (22\%)\\ 
Aliasing & (21\%) & $z$-transform & (24\%) & Laplace transform & (12\%)\\
\midrule[0.2em]
 Concept 4 && Concept 5\\ 
\midrule[0.1em]
Fourier transform & (43\%) & Impulse response & (74\%) \\
Systems/circuits & (31\%) & Transfer function & (15\%) \\ 
Transfer function & (26\%) &  Fourier transform & (11\%)\\ 
\bottomrule[0.2em]\\
\end{tabular}}
\vspace{-1.0cm}
\label{fig:301W_b}
}
\caption{\subref{fig:stems_a} Question--concept association graph and \subref{fig:stems_b} most important tags associated with each concept for an undergraduate DSP course with $N=15$ learners answering $Q=44$ questions.}
\label{fig:301W}
\end{figure}

\fref{fig:301W}(a) visualizes the estimated question--concept association matrix $\widehat{\bW}$ as a bipartite graph consisting of question and concept nodes.\footnote{To avoid the scaling identifiability problem that is typical in factor analysis, we normalize each row of~$\bC$ to unit $\elltwo$-norm and scale each column of $\bW$ accordingly prior to visualizing the bipartite graph. This enables us to compare the strength of question--concept associations across different concepts.}
In the graph, circles represent the estimated concepts and squares represent questions, with thicker edges indicating stronger question--concept associations (i.e., larger entries~$\widehat{W}_{i,k}$).
Questions are also labeled with their estimated intrinsic difficulty $\mu_i$, with larger positive values of $\mu_i$ indicating easier questions.
Note that ten questions are not linked to any concept.  
All $Q=15$ learners answered these questions correctly; as a result nothing can be estimated about their underlying concept structure.
\fref{fig:301W}(b) provides the concept--tag association (top 3 tags) for each of the 5 estimated concepts.

\begin{table}[tb]
\centering
\caption{Selected tag knowledge of Learner~1.} 
\vspace{0.2cm}
\label{tbl:301tag1}
    \begin{tabular}{ccccc}
\toprule[0.1em]
$z$-transform & Impulse response & Transfer function &Fourier transform & Laplace transform \\ 
\midrule[0.05em]
1.09 & $-1.80$ & $-0.50$ & 0.99 & $-0.77$  \\
\bottomrule[0.1em]
\end{tabular}
\end{table}

\begin{table}[tb]
\centering
\caption{Average tag knowledge of all learners.}
\label{tbl:301tag2}
\vspace{0.2cm}
    \begin{tabular}{ccccc}
\toprule[0.1em]
$z$-transform & Impulse response & Transfer function &Fourier transform & Laplace transform \\ 
\midrule[0.05em]
0.04 & $-0.03$ & $-0.10$ & 0.11 & 0.03  \\
\bottomrule[0.1em]
 \end{tabular}
\end{table}

\fref{tbl:301tag1} provides Learner 1's knowledge of the various tags relative to other learners.  
Large positive values mean that Learner~1 has strong knowledge of the tag, while large negative values indicate a deficiency in knowledge of the tag. 
\fref{tbl:301tag2} shows the average tag knowledge of the entire class, computed by averaging the entries of each row in the learner tag knowledge matrix $\bU$ as described in \fref{sec:tagprof}.
\fref{tbl:301tag1} indicates that Learner~1 has particularly weak knowledges of the tag ``Impulse response.'' 
Armed with this information, a PLS could automatically suggest remediation about this concept to Learner~1. 
\fref{tbl:301tag2} indicates that the entire class has (on average) weak knowledge of the tag ``Transfer function.'' With this information, a PLS could suggest to the class instructor that they provide remediation about this concept to the entire class.

\subsubsection{Grade 8 science course}
\label{sec:realstems}

\paragraph{Dataset}

The STEMscopes dataset was introduced in \fref{sec:SPARFASectionref}.  There is substantial missing data in the matrix $\bY$, with only 13.5\% of its entries observed.

\paragraph{Analysis}

We compare the results of SPARFA-M and SPARFA-B on this data set to highlight the pros and cons of each approach.
For both algorithms, we select $K=5$ concepts. 
For SPARFA-B, we fix reasonably broad (non-informative) values for all hyperparameters. 
For $\mu_0$ we calculate the average rate of correct answers $p_s$ on observed graded responses of all learners to all questions and use $\mu_0 = \Phi^{-1}_\text{pro}(p_s)$. 
The variance $v_{\boldsymbol{\mu}}$ is left sufficiently broad to enable adequate exploration of the intrinsic difficulty for each questions. 
Point estimates of $\bW$, $\bC$ and $\boldsymbol{\mu}$ are generated from the SPARFA-B posterior distributions using the methods described in \fref{sec:sparfabpostprocessing}. 
Specifically, an entry $\widehat{W}_{i,k}$ that has a corresponding active probability $\widehat{R}_{i,k} < 0.55$ is thresholded to 0. Otherwise, we set  $\widehat{W}_{i,k}$ to its posterior mean.
On a 3.2\,GHz quad-core desktop PC, SPARFA-M converged to its final estimates in 4\,s, while SPARFA-B required 10 minutes.

\paragraph{Results and discussion}

Both SPARFA-M and SPARFA-B deliver comparable factorizations. The estimated question--concept association graph for SPARFA-B is shown in \fref{fig:stems_a}, with the accompanying concept--tag association in \fref{fig:stems_b}.
Again we see a sparse relationship between questions and concepts.
The few outlier questions that are not associated with any concept are generally those questions with very low intrinsic difficulty or those questions with very few responses.

One advantage of SPARFA-B over SPARFA-M is its ability to provide not only point estimates of the parameters of interest but also reliability information for those estimates. This reliability information can be useful for decision making, since it enables one to tailor actions according to the associated uncertainty. If there is considerable uncertainty regarding learner mastery of a particular concept, for example, it may be a more appropriate use of time of the learner to ask additional questions  that reduce the uncertainty, rather than assigning new material for which the learner may not be adequately prepared.

\begin{figure*}[t]
\centering
\includegraphics[width=0.55\textwidth]{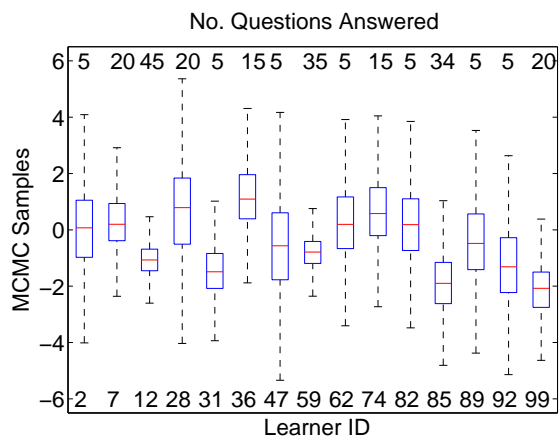}
\vspace{-0.6cm}
\caption{Concept~5 knowledge estimates generated by SPARFA-B for the STEMscopes data for a randomly selected subset of learners.
The box-whisker plot shows the posterior variance of the MCMC samples, with each box-whisker plot corresponding to a different learner in the dataset. Anonymized learner IDs are shown on the bottom, while the number of relevant questions answered by each learner answered is indicated on the top of the plot.}
\vspace{-0.5cm}
\label{fig:randomvariances}
\end{figure*}

We demonstrate the utility of SPARFA-B's posterior distribution information on the learner concept knowledge matrix $\bC$. 
\fref{fig:randomvariances} shows box-whisker plots of the MCMC output samples over 30,000 iterations (after a burn-in period of 30,000 iterations) for a set of learners for Concept~5.  Each box-whisker plot corresponds to the posterior distribution for a different student.
These plots enable us to visualize both the posterior mean and variance associated with the concept knowledge estimates $\hat{\vecc}_j$.  As one would expect, the estimation variance tends to decrease as the number of answered questions increases (shown in the top portion of \fref{fig:randomvariances}).

The exact set of questions answered by a learner also affects the posterior variance of our estimate, as different questions convey different levels of information regarding a learner's concept mastery. An example of this phenomenon is observed by comparing Learners~7 and 28. Each of these two learners answered 20 questions and had a nearly equal number of correct answers (16 and 17, respectively).  A conventional analysis that looked only at the percentage of correct answers would conclude that both learners have similar concept mastery. However, the actual set of questions answered by each learner is not the same, due to their respective instructors assigning different questions. 
While SPARFA-B finds a similar posterior mean for Learner~7 and Learner~28, it finds very different posterior variances, with considerably more variance for Learner~28.
The SPARFA-B posterior samples shed additional light on the situation at hand.  Most of the questions answered by Learner~28 are deemed easy (defined as having intrinsic difficulties~$\hat\mu_i$ larger than one).  Moreover, the remaining, more difficult questions answered by Learner~28 show stronger affinity to concepts other than Concept~5. 
In contrast, roughly half of the questions answered by Learner~7 are deemed hard and all of these questions have stronger affinity to Concept~5.   Thus, the questions answered by Learner~28 convey only weak information about the knowledge of Concept~5, while those answered by Learner~7 convey strong information.  Thus, we cannot determine from Learner 28's responses whether they have mastered Concept~5 well or not. Such SPARFA-B posterior data would enable a PLS to quickly assess this scenario and tailor the presentation of future questions to Learner~28---in this case, presenting more difficult questions related to Concept~5 would reduce the estimation variance on their concept knowledge and allow a PLS to better plan future educational tasks for this particular learner. 

Second, we demonstrate the utility of SPARFA-B's posterior distribution information on the question--concept association matrix $\bW$. Accurate estimation of $\bW$ enables course instructors and content authors to validate the extent to which problems measure knowledge across various concepts.
 In general, there is a strong degree of commonality between the results of SPARFA-M and SPARFA-B, especially as the number of learners answering a question grow.  
We present some illustrative examples of support estimation on~$\bW$ for both SPARFA algorithms in \fref{tbl:qresults}. We use the labels ``Yes''/``No'' to indicate inclusion of a concept by SPARFA-M and show the posterior inclusion probabilities for each concept by SPARFA-B.
Here, both SPARFA-M and SPARFA-B agree strongly on both Question~3 and Question~56. Question~72 is answered by only 6 learners, and SPARFA-M discovers a link between this question and Concept~5. SPARFA-B proposes Concept~5 in $58\%$ of all MCMC iterations, but also Concept~1 in 60\% of all MCMC iterations. Furthermore, the proposals of Concept~1 and Concept~5 are nearly mutually exclusive; in most iterations only one of the two concepts is proposed, but both are rarely proposed jointly. This behavior implies that SPARFA-B has found two competing models that explain the data associated with Question~72. To resolve this ambiguity, a PLS would need to gather more learner responses.

\begin{table}
\centering
\caption{Comparison of SPARFA-M and SPARFA-B for three selected questions and the $K=5$ estimated concepts in the STEMscopes dataset.  
For SPARFA-M, the labels ``Yes'' and ``No'' indicate whether a particular concept was detected in the question. For SPARFA-B, we show the posterior inclusion probability (in percent), which indicates the percentage of iterations in which a particular concept was sampled.}
\label{tbl:qresults}
\begin{tabular}{lllllll}
\toprule[0.2em]
~&~&C1&C2&C3&C4&C5 \\
\midrule[0.15em]
\multirow{2}{*}{Q3 (27 responses)} & M & Yes & No & No & No & Yes \\
~ & B & 94\% & 36\% & 48\% & 18\% & 80\% \\ 
\midrule[0.1em]
\multirow{2}{*}{Q56 (5 responses)} & M & No & No & No & No & No \\
~ & B & 30\% & 30\% & 26\% & 31\% & 31\% \\ 
\midrule[0.1em]
\multirow{2}{*}{Q72 (6 responses)} & M & No & No & No & No & Yes \\
~ & B & 61\% & 34\% & 29\% & 36\% & 58\% \\
\bottomrule[0.2em]
\end{tabular}
\end{table}

\subsubsection{Algebra Test Administered on Amazon Mechanical Turk}
\label{sec:realmturk}

For a final demonstration of the capabilities the SPARFA algorithms, we analyze a dataset from a high school algebra test carried out by Daniel Calder\'{o}n of Rice University on Amazon Mechanical Turk, a crowd-sourcing marketplace (\cite{mechturkwebsite}). 

\paragraph{Dataset}

The dataset consists of $N=99$ learners answering $Q=34$ questions covering topics such as geometry, equation solving, and visualizing function graphs.  Calder\'{o}n manually labeled the questions from a set of $M = 10$ tags.  The dataset is fully populated, with no missing entries. 

\paragraph{Analysis}

We estimate $\bW$, $\bC$, and $\boldsymbol\mu$ from the fully populated $34 \times 99$ binary-valued matrix $\bY$ using  the logit version of SPARFA-M assuming $K = 5$ concepts. 
We deploy the tag-analysis approach proposed in \fref{sec:taganalysis} to interpret each concept. 
Additionally, we calculate the likelihoods of the responses using \fref{eq:qa} and the estimates $\widehat{\bW}$, $\widehat{\bC}$ and $\hat{\boldsymbol{\mu}}$.
The results from SPARFA-M are summarized in \fref{fig:mkturk}.
We detail the results of our analysis for Questions~19--26 in \fref{tbl:ans2} and for Learner~1 in \fref{tbl:stu2}.

\begin{figure}[tp]
\vspace{-1.0cm}
\centering
\subfigure[Question--concept association graph.]{
\includegraphics[width=0.84\columnwidth]{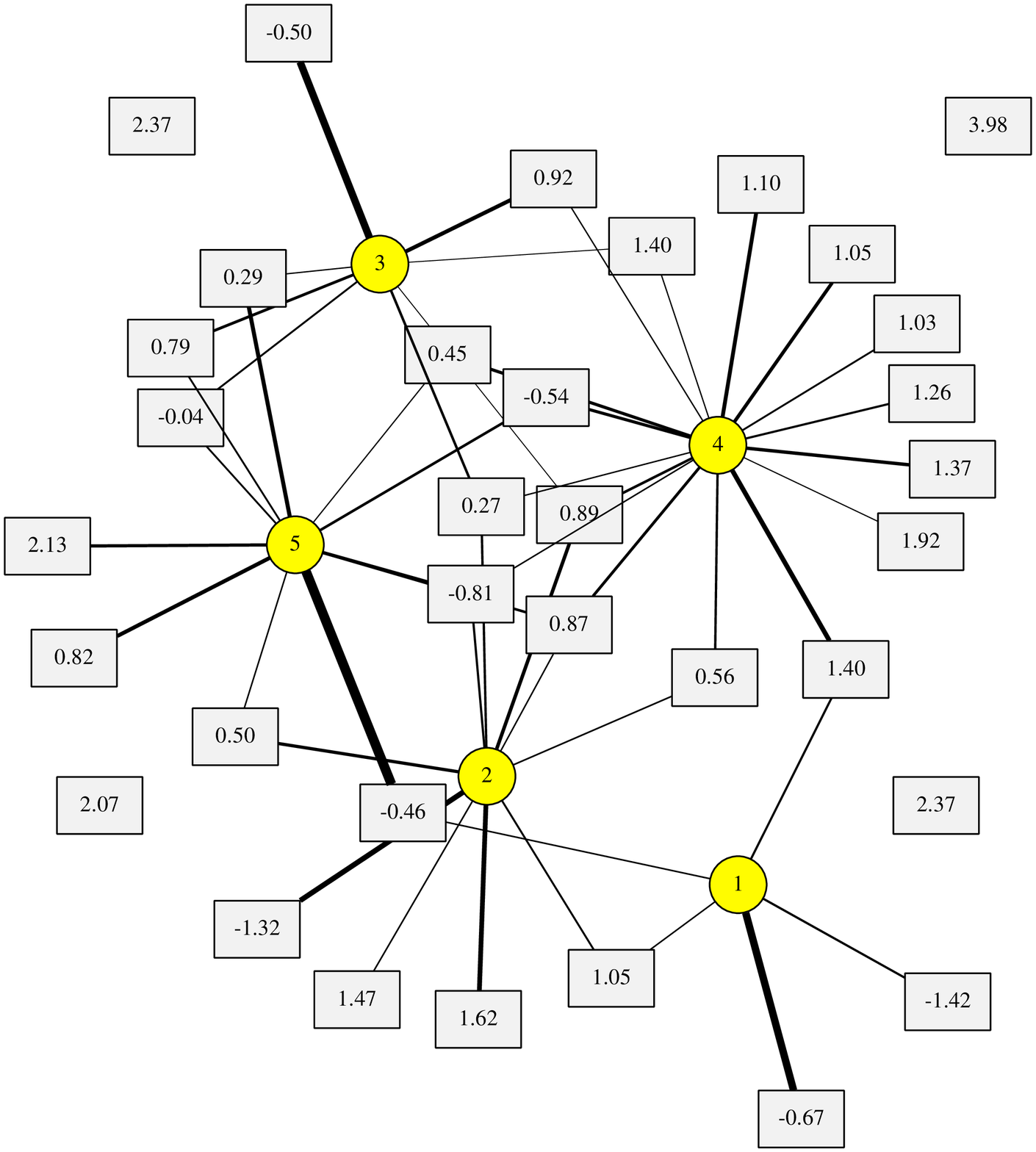}
\label{fig:mturk_a}
}\\[0.3cm]
\subfigure[Most important tags and relative weights for the estimated concepts.]{
\scalebox{.8}{%
\begin{tabular}{llllll}
\toprule[0.2em]
Concept 1 && Concept 2 && Concept 3\\
\midrule[0.1em]
Fractions & (57\%)& Plotting functions & (64\%) & Geometry & (63\%)\\
Solving equations & (42\%) & System of equations & (27\%) & Simplifying expressions & (27\%)\\ 
Arithmetic & (1\%) & Simplifying expressions & (9\%)& Trigonometry & (10\%)\\
\midrule[0.1em]
 Concept 4 && Concept 5\\ 
\midrule[0.1em]
Simplifying expressions & (64\%) & Trigonometry & (53\%) \\
Trigonometry & (21\%) & Slope & (40\%) \\ 
Plotting Functions & (15\%) & Solving equations & (7\%) \\ 
\bottomrule[0.2em]\\
\end{tabular}}
\label{fig:mturk_b}
}
\vspace{-0.2cm}
\caption{\subref{fig:mturk_a} Question--concept association graph and \subref{fig:mturk_b} most important tags associated with each concept  for a high-school algebra test carried out on Amazon Mechanical Turk with $N=99$ users answering $Q=34$ questions.}
\label{fig:mkturk}
\vspace{-0.5cm}
\end{figure}

\paragraph{Results and discussion}

With the aid of SPARFA, we can analyze the strengths and weaknesses of each learner's concept knowledge both individually and relative to other users.
We can also detect outlier responses that are due to guessing, cheating, or carelessness.
The values in the estimated concept knowledge matrix $\widehat{\bC}$ measure each learner's concept knowledge relative to all other learners. 
The estimated intrinsic difficulties of the questions $\hat{\boldsymbol{\mu}}$ provide a relative measure that summarizes how all users perform on each question. 

Let us now consider an example in detail; see \fref{tbl:ans2} and \fref{tbl:stu2}.
Learner~1 incorrectly answered Questions~21 and 26 (see \fref{tbl:ans2}), which involve Concepts~1 and 2.
Their knowledge of these concepts is not heavily penalized, however (see \fref{tbl:stu2}), due to the high intrinsic difficulty of these two questions, which means that most other users also incorrectly answered them.    
User~1 also incorrectly answered Questions~24 and 25, which involve Concepts~2 and 4.
Their knowledge of these concepts is penalized, due to the low intrinsic difficulty of these two questions, which means that most other users correctly answered them.   
Finally, Learner~1 correctly answered Questions~19 and 20, which involve Concepts~1 and 5.  
Their knowledge of these concepts is boosted, due to the high intrinsic difficulty of these two questions.

SPARFA can also be used to identify each user's individual strengths and weaknesses.
Continuing the example, Learner~1 needs to improve their knowledge of Concept~4 (associated with the tags ``Simplifying expressions'', ``Trigonometry,'' and ``Plotting functions'') significantly, while their deficiencies on Concepts~2 and 3 are relatively minor. 

Finally, by investigating the likelihoods of the graded responses, we can detect outlier responses, which would enables a PLS to detect guessing and cheating. 
By inspecting the concept knowledge of Learner~1 in \fref{tbl:stu2}, we can identify insufficient knowledge of  Concept~4. Hence, Learner 1's correct answer to Question~22 is likely due to a random guess, since the predicted likelihood of providing the correct answer is estimated at only~$0.21$.

\begin{table}[tb] 
\caption{Graded responses and their underlying concepts for Learner~1 (1 designates a correct response and 0 an incorrect response).}\vspace{0.2cm}
\centering
\scalebox{.97}{%
\begin{tabular}{lcccccccc}
\toprule[0.15em]
Question number & $19$ & $20$ & $21$ & $22$ & $23$ & $24$ & $25$ & $26$\\
\midrule
Learner's graded response $Y_{i,j}$ & $1$ & $1$ & $0$ & $ 1$ & $1$  & $0$ & $0$ & $0$\\
Correct answer likelihood & $\multirow{2}{*}{0.79}$ & \multirow{2}{*}{$0.71$} & \multirow{2}{*}{$0.11$} & \multirow{2}{*}{$\bf{0.21}$} & \multirow{2}{*}{$0.93$}  & \multirow{2}{*}{$0.23$} & \multirow{2}{*}{$0.43$} & \multirow{2}{*}{$0.00$} \\
\,\,\, $p(Y_{i,j}=1 | \vecw_i, \vecc_j, \mu_i)$ \\
Underlying concepts & $1$ & $1,5$ & $1$ & $ 2,3,4$ & $3,5$ & $2,4$ & $1,4$ & $2,4$\\ 
Intrinsic difficulty $\mu_i$ & $-1.42$ & $-0.46$ & $-0.67$ & $ 0.27$ & $0.79$ & $0.56$ & $1.40$ & $-0.81$\\
\bottomrule[0.15em]
    \end{tabular}
}
\vspace{-0.0cm}
\label{tbl:ans2}
\end{table}

\begin{table}[tb] 
\caption{Estimated concept knowledge for Learner~1.} \vspace{0.2cm}
\centering
\begin{tabular}{lccccc}
\toprule[0.15em]
Concept number & $1$ & $2$ & $3$ & $4$ & $5$ \\ 
\midrule
Concept knowledge  & $0.46$ & $-0.35$ & $0.72$ & $-1.67$ & $0.61$ \\ 
\bottomrule[0.15em]
\end{tabular}
\label{tbl:stu2}
\vspace{-.0cm}
\end{table}

\newpage

\subsection{Predicting Unobserved Learner Responses} \label{sec:realpred}

We now compare SPARFA-M against the recently proposed binary-valued collaborative
filtering algorithm CF-IRT (\cite{logitfa}) in an experiment to predict unobserved learner responses.

\paragraph{Dataset and experimental setup} In this section, we study both the Mechanical Turk algebra test dataset and a portion of the ASSISTment dataset (\cite{assistment}). The ASSISTment dataset consists of $N=403$ learners answering $Q=219$ questions, with $25\%$ of the responses observed (see \cite{tesr} for additional details on the dataset). In each of the 25 trials we run for both datasets, we hold out 20\% of the observed learner responses as a test set, and train both the logistic variant of \mbox{SPARFA-M}\footnote{In order to arrive at a fair comparison, we choose to use the logistic variant of \mbox{SPARFA-M}, since CF-IRT also relies on a logistic model.} and CF-IRT on the rest. The regularization parameters of both algorithms are selected using $4$-fold cross-validation on the training set. We use two performance metrics to evaluate the performance of these algorithms, namely (i) the prediction accuracy, which corresponds to the percentage of correctly predicted unobserved responses, and (ii) the average prediction likelihood $\frac{1}{\mid \bar{\Omega}_\text{obs} \mid} \textstyle \sum_{i,j:(i,j)\in \bar{\Omega}_\text{obs}} p(Y_{i,j} | \vecw_i, \vecc_j)$ of the unobserved responses, as proposed in \cite{dynamickt}, for example.

\begin{figure*}[t]
\vspace{0.51cm}
\centering
\subfigure[Prediction accuracy for the Mechanical Turk algebra test dataset.]{
\includegraphics[width=0.45\textwidth]{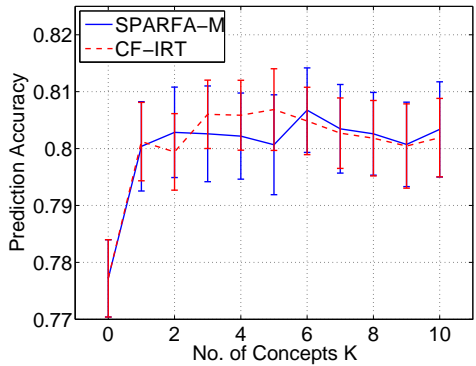}
\label{fig:mturk_a}
}
\hspace{0.6cm}
\subfigure[Average prediction likelihood for the Mechanical Turk algebra test dataset.]{
\includegraphics[width=0.45\textwidth]{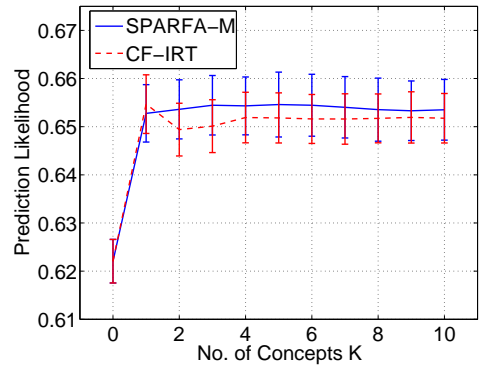}
\label{fig:mturk_l}
}\\[0.3cm]
\vspace{-0.6cm}
\subfigure[Prediction accuracy for the ASSISTment dataset.]{
\includegraphics[width=0.45\textwidth]{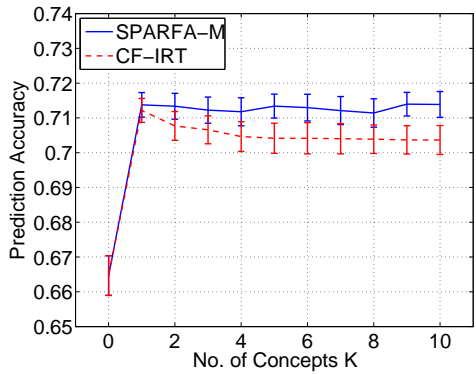}
\label{fig:ass_a}
}\hspace{0.8cm}
\subfigure[Average prediction likelihood for the ASSISTment dataset.]{
\includegraphics[width=0.45\textwidth]{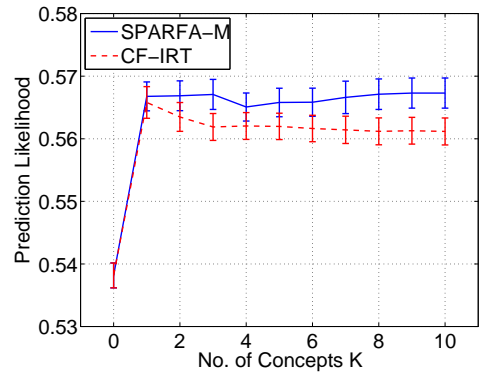}
\label{fig:ass_l}
}
\vspace{-0.3cm}
\caption{Performance comparison of SPARFA-M and CF-IRT on \subref{fig:mturk_a} prediction accuracy and \subref{fig:mturk_l} average prediction likelihood for the Mechanical Turk algebra test dataset, \subref{fig:ass_a} prediction accuracy and \subref{fig:ass_l} average prediction likelihood for the ASSISTment dataset. SPARFA-M achieves comparable or better performance than CF-IRT while enabling interpretability of the estimated latent concepts}
\label{fig:pred}
\vspace{-0.9cm}
\end{figure*}

\paragraph{Results and discussion} Figure~\ref{fig:pred} shows the prediction accuracy and prediction likelihood for both the Mechanical Turk algebra test dataset and the ASSISTment dataset. We see that \mbox{SPARFA-M} delivers comparable (sometimes slightly superior) prediction performance to CF-IRT in predicting unobserved learner responses.

Furthermore, we see from \fref{fig:pred} that the prediction performance varies little over different values of $K$, meaning that the specific choice of $K$ has little influence on the prediction performance within a certain range. This phenomenon agrees with other collaborative filtering results (see, e.g., \cite{svdplus,ordpred}). Consequently, the choice of $K$ essentially dictates the granularity of the abstract concepts we wish to estimate. 
This justifies our choice of $K=5$ in the real data experiments of \fref{sec:real} when we visualize the question--concept associations as bi-partite graphs, as it provides a desirable granularity of the estimated concepts in the datasets.
We emphasize that SPARFA-M is able to provide interpretable estimated factors while achieving comparable (or slightly superior) prediction performance than that achieved by CF-IRT, which does not provide interpretability. This feature of SPARFA is key for the development of PLSs, as it enables an automated way of generating interpretable feedback to learners in a purely data-driven fashion.

\section{Related Work on Machine Learning-based Personalized Learning}
\label{sec:rw}


A range of different machine learning algorithms have been applied in educational contexts.  
Bayesian belief networks have been successfully used to probabilistically model and analyze learner response data (e.g., \cite{gregk2,woolf08,gregk1}).  Such models, however, rely on predefined question--concept dependencies (that are not necessarily the true dependencies governing learner responses) and primarily only work for a single concept.  
In contrast, SPARFA discovers question--concept dependencies from solely the graded learner responses to questions and naturally estimates multi-concept question dependencies. 

\sloppy
Modeling question--concept associations has been studied in \cite{qmatrix}, \cite{viet}, \cite{viettwo}, and \cite{qprobit}.
The approach in \cite{qmatrix} characterizes the underlying question--concept associations using binary values, which ignore the relative strengths of the question--concept associations. 
In contrast, SPARFA differentiates between strong and weak relationships through the real-valued weights $W_{i,k}$.
The matrix and tensor factorization methods proposed in \cite{qmatrix}, \cite{viet}, and \cite{viettwo} treat graded learner responses as real but deterministic values. 
In contrast, the probabilistic framework underlying SPARFA provides a statistically principled model for graded responses; the likelihood of the observed graded responses provides even more explanatory power.  
\fussy

Existing intelligent tutoring systems capable of modeling question--concept relations probabilistically include Khan Academy (\cite{Khan,huweb}) and the system of \cite{minkak1}. 
Both approaches, however, are limited to dealing with a single concept.
In contrast, SPARFA is built from the ground up to deal with multiple latent concepts.  

A probit model for graded learner responses is used in \cite{qprobit} without exploiting the idea of low-dimensional latent concepts.  
In contrast, SPARFA leverages multiple latent concepts and therefore can create learner concept knowledge profiles for personalized feedback.  Moreover, SPARFA-M is compatible with the popular logit model.  

The recent results developed in \cite{predfact} and \cite{logitfa} address the problem of \emph{predicting} the missing entries in a binary-valued graded learner response matrix.  
Both papers use low-dimensional latent factor techniques specifically developed for collaborative filtering, as, e.g., discussed in \cite{amaz} and \cite{cfeval}. 
While predicting missing correctness values is an important task, these methods do not take into account the sparsity and non-negativity of the matrix $\bW$; this inhibits the interpretation of the relationships among questions and concepts. 
In contrast, SPARFA accounts for both the sparsity and non-negativity of $\bW$, which enables the interpretation of the value $C_{k,j}$ as learner $j$'s knowledge of concept $k$.

There is a large body of work on item response theory (IRT), which uses statistical models to analyze and score graded question response data (see, e.g., \cite{lordirt}, \cite{irtest}, and \cite{mirt} for overview articles).  The main body of the IRT literature builds on the model developed by \cite{rasch} and has been applied mainly in the context of adaptive testing (e.g., in the graduate record examination (GRE) and graduate management (GMAT) tests \cite{seqtest}, \cite{abilityest}, and \cite{raschest}). 
While the SPARFA model shares some similarity to the model in \cite{rasch} by modeling question--concept association strengths and intrinsic difficulties of questions, it also models each learner in terms of a multi-dimensional concept knowledge vector. This capability of SPARFA is in stark contrast to the Rasch model, where each learner is characterized by a single, scalar ability parameter. 
Consequently, the SPARFA framework is able to provide stronger explanatory power in the estimated factors compared to that of the conventional Rasch model. 
We finally note that multi-dimensional variants of IRT have been proposed in \cite{mcdonald}, \cite{yao}, and \cite{mirt}. We emphasize, however, that the design of these algorithms leads to poor interpretability of the resulting parameter estimates.

\section{Conclusions}
\label{sec:conclusions}

In this paper, we have formulated a new approach to learning and content analytics, which is based on a new statistical model that encodes the probability that a learner will answer a given question correctly in terms of three factors: 
\begin{inparaenum}[(i)]
\item the learner's knowledge of a set of latent concepts,
\item how the question related to each concept, and 
\item the intrinsic difficulty of the question.  
\end{inparaenum}
We have proposed two algorithms, SPARFA-M and SPARFA-B, to estimate the above three factors given incomplete observations of graded learner question responses.  
SPARFA-M uses an efficient Maximum Likelihood-based bi-convex optimization approach to produce point estimates of the factors, while  SPARFA-B uses Bayesian factor analysis to produce posterior distributions of the factors.  
In practice, SPARFA-M is beneficial in applications where timely results are required; SPARFA-B is favored in situations where posterior statistics are required.
We have also introduced a novel method for incorporating user-defined tags on questions to facilitate the interpretability of the estimated factors.
Experiments with both synthetic and real world education datasets have demonstrated both the efficacy and robustness of the SPARFA algorithms.   

The quantities estimated by SPARFA can be used directly in a range of PLS functions.
For instance, we can identify the knowledge level of learners on particular concepts and diagnose why a given learner has incorrectly answered a particular question or type of question. 
Moreover, we can discover the hidden relationships among questions and latent concepts, which is useful for identifying questions that do and do not aid in measuring a learner's conceptual knowledge. 
Outlier responses that are either due to guessing or cheating can also be detected.
In concert, these functions can enable a PLS to generate personalized feedback and recommendation of study materials, thereby enhancing overall learning efficiency.

\sloppy

Various extensions to the SPARFA framework developed here have been proposed recently. In particular, a variant of \mbox{SPARFA-M} that analyzes ordinal rather than binary-valued responses and  directly utilizes tag information in the probabilistic model has been detailed in \cite{sparfatag}. Another variant of SPARFA-M that improves the interpretability of the underlying concepts via the joint analysis of graded learner responses and question/response text has been proposed in \cite{sparfatop}. A nonparametric Bayesian variant of SPARFA-B that estimates both the number of concepts $K$ as well as the reliability of each student from data has been developed in \cite{nonparamsparfab}.

\fussy

Before closing, we would like to point out a connection between SPARFA and \emph{dictionary learning} that is of independent interest.  
This connection can be seen by noting that \fref{eq:qam} for both the probit and inverse logit functions is statistically equivalent to (see \cite{gpml}): 
\begin{align*} 
Y_{i,j}=[\sign(\bW\bC+\bM+\bN)]_{i,j}, \,\, (i,j)\in\Omega_\text{obs},
\end{align*}
where $\sign(\cdot)$ denotes the entry-wise sign function and the entries of $\bN$ are i.i.d.\ and drawn from either a standard Gaussian or standard logistic distribution. 
Hence, estimating~$\bW$,~$\bC$, and $\bM$ (or equivalently, $\boldsymbol{\mu}$) is equivalent to learning a (possibly overcomplete) dictionary from the data $\bY$.  
The key departures from the dictionary-learning literature (\cite{ksvd,bach}) and algorithm variants capable of handling missing observations (\cite{studi}) are the binary-valued observations and the non-negativity constraint on~$\bW$.  
Note that the algorithms developed in \fref{sec:matrix} to solve the sub-problems by holding one of the factors~$\bW$ or~$\bC$ fixed and solving for the other variable can be used to solve noisy binary-valued (or 1-bit) compressive sensing or sparse signal recovery problems, e.g., as studied in \cite{petros1bit}, \cite{jason1bit}, and \cite{plan1bit}.
Thus, the proposed SPARFA algorithms can be applied to a wide range of applications beyond education, including the analysis of survey data, voting patterns, gene expression, and signal recovery from noisy 1-bit compressive measurements. 

\appendix
\section{Proof of \fref{thm:rrconv}}
\label{app:convRR}

We now establish the convergence of the FISTA algorithms that solve the SPARFA-M subproblems $(\text{RR})_1^+$ and $(\text{RR})_2$.
We start by deriving the relevant Lipschitz constants.

\begin{lem}[Scalar Lipschitz constants] \label{lem:scalarlipschitz}
Let $g_\text{pro}(x) = \frac{\Phi_\text{pro}'(x)}{\Phi_\text{pro}(x)}$ and $g_\text{log}(x) = \frac{\Phi_\text{log}'(x)}{\Phi_\text{log}(x)}$, \mbox{$x \in \mathbb{R}$}, where $\Phi_\text{pro}(x)$ and $\Phi_\text{log}(x)$ are the inverse probit and logit link functions defined in~\fref{eq:probitlink} and~\fref{eq:logitlink}, respectively. Then, for $y,z \in \mathbb{R}$ we have 
\begin{align} 
\abs{g_\text{pro}(y)-g_\text{pro}(z)} & \leq L_\text{pro} \abs{y-z} \label{eq:prolips},\\ 
\abs{g_\text{log}(y)-g_\text{log}(z)} & \leq  L_\text{log} \abs{y-z} \label{eq:loglips},
\end{align}
with the constants $L_\text{pro}=1$ for the probit case  and $L_\text{log}=1/4$ for the logit case.
\end{lem}

\begin{proof}
For simplicity of exposition, we omit the subscripts designating the probit and logit cases in what follows.
We first derive $L_\text{pro}$ for the probit case by computing the derivative of $g(x)$ and bounding its derivative from below and above.
The derivative of $g(x)$ is given by 
\begin{align} \label{eq:probitderivative}
g'(x) = -\frac{\mathcal{N}(x)}{\Phi(x)} \left( x+\frac{\mathcal{N}(x)}{\Phi(x)} \right).
\end{align}
where $\mathcal{N}(t) = \frac{1}{\sqrt{2 \pi}} e^{-t^2/2}$ is the PDF of the standard normal distribution.

We first bound this derivative for $x \leq 0$. To this end, we individually bound the first and second factor in \fref{eq:probitderivative} using the following bounds listed in \cite{nist}:
\begin{align*}
-\frac{x}{2} + \sqrt{\frac{x^2}{4}+\frac{2}{\pi}} \leq \frac{\mathcal{N}(x)}{\Phi(x)} \leq -\frac{x}{2} + \sqrt{\frac{x^2}{4}+1}, \quad x\leq 0
\end{align*}
and
\begin{align*}
\frac{x}{2} + \sqrt{\frac{x^2}{4}+\frac{2}{\pi}} \leq x + \frac{\mathcal{N}(x)}{\Phi(x)} \leq \frac{x}{2} + \sqrt{\frac{x^2}{4}+1}, \quad x\leq 0. 
\end{align*}
Multiplying the above inequalities leads to the bounds
\begin{align} \label{eq:onesideok}
-1 \leq g'(x) \leq -\frac{2}{\pi}, \quad \quad x\leq 0.
\end{align}

We next bound the derivative of \fref{eq:probitderivative} for $x>0$.
For $x>0$, $\mathcal{N}(x)$ is a positive decreasing function and $\Phi(x)$ is a positive increasing function; hence $\frac{\mathcal{N}(x)}{\Phi(x)}$ is a decreasing function and $\frac{\mathcal{N}(x)}{\Phi(x)} \leq \frac{\mathcal{N}(0)}{\Phi(0)} = \sqrt{2/\pi}$. 
Thus, we arrive at 
\begin{align*}
g'(x) = -\frac{\mathcal{N}(x)}{\Phi(x)} \left(x + \frac{\mathcal{N}(x)}{\Phi(x)}\right) \geq -\frac{\mathcal{N}(x)}{\Phi(x)} \left(x + \sqrt{2/\pi}\right),
\end{align*}
where we have used the facts that $\Phi(x) \geq 1/2$ and $\mathcal{N}(x) \leq \frac{1}{\sqrt{2\pi}}$ for $x>0$.
According to \fref{eq:probitlink} and the bound of \cite{phibound}, we have
\begin{align}
\Phi(x) = \frac{1}{2} + \int_{0}^x \! \mathcal{N}(t\mid 0,1)  \mathrm{d}t \geq \frac{1}{2} + \frac{1}{2} \sqrt{1-e^{-x^2/2}} \geq 1-\frac{1}{2}e^{-x^2/2}, \label{eq:mrlanstrick}
\end{align}
where the second inequality follows from the fact that $(1-e^{-x^2/2}) \in [0,1]$. Using \fref{eq:mrlanstrick} we can further  bound $g'(x)$ from below as
\begin{align*}
g'(x) \geq - \frac{\mathcal{N}(x)}{1-\frac{1}{2}e^{-x^2/2}} \left(x + \sqrt{2/\pi}\right).
\end{align*}
Let us now assume that  
\begin{align*}
- \frac{\mathcal{N}(x)}{1-\frac{1}{2}e^{-x^2/2}} \left(x + \sqrt{2/\pi}\right) \geq -1.
\end{align*}
In order to prove that this assumption is true, we rearrange terms to obtain
\begin{align} \label{eq:assumptionreformulated}
\left( \frac{x}{\sqrt{2\pi}}+\left({1}/{\pi}+{1}/{2} \right) \right) e^{-x^2/2} \leq 1.
\end{align}
Now we find the maximum of the LHS of \fref{eq:assumptionreformulated} for $x>0$. To this end, we observe that $\frac{x}{\sqrt{2\pi}}+\left({1}/{\pi}+{1}/{2} \right)$ is monotonically increasing and that $e^{-x^2/2}$ monotonically decreasing for $x>0$; hence, this function has a unique maximum in this region. By taking its derivative and setting it to zero, we obtain
\begin{align*}
x^2 + \sqrt{2/\pi}+\sqrt{\pi/2} -1 = 0
\end{align*}
Substituting the result of this equation, i.e., $\hat x \approx 0.4068$, into \fref{eq:assumptionreformulated} leads to  
\begin{align*}
\left(\frac{\hat x}{\sqrt{2\pi}}+\left({1}/{\pi}+{1}/{2}\right)\right) e^{-\hat{x}^2/2} \approx 0.9027 \leq 1,
\end{align*}
which certifies our assumption. Hence, we have
\begin{align*}
-1 \leq g'(x) \leq 0, \quad x > 0.
\end{align*}

Combining this result with the one for $x \leq 0$ in \fref{eq:onesideok} yields
\begin{align*}
-1 \leq g'(x) \leq 0, \quad x \in \mathbb{R}.
\end{align*}
We finally obtain the following bound on the scalar Lipschitz constant \fref{eq:prolips}:
\begin{align*}
\abs{g_\text{pro}(y) - g_\text{pro}(z)} \leq \abs{ \int_y^z \abs{g_\text{pro}'(x)} dx}  \leq \abs{ \int_y^z \,1\, dx} = \abs{y - z},
\end{align*}
which concludes the proof for the probit case. 

We now develop the bound $L_\text{log}$ for the logit case. To this end, we bound the derivative of $g_\text{log}(x) = \frac{1}{1+e^{x}}$ as follows:
\begin{align*}
0 \geq g_\text{log}'(x) = -\frac{e^x}{(1+e^x)^2} = -\frac{1}{e^x + e^{-x} +2} \geq -\frac{1}{4}.
\end{align*}
where we used the inequality of arithmetic and geometric means.
Consequently, we have the following bound on the scalar Lipschitz constant \fref{eq:loglips}:
\begin{align*}
\abs{g_\text{log}(y) - g_\text{log}(z)} \leq \abs{ \int_y^z \abs{g_\text{log}'(x)} dx} \leq | \int_y^z \,\frac{1}{4}\, dx| = \frac{1}{4} |y - z|,
\end{align*}
which concludes the proof for the logit case. 
\end{proof}

The following lemma establishes a bound on the (vector) Lipschitz constants for the individual regularized regression problems $(\text{RR}_1^+)$ and  $(\text{RR}_2)$ for both the probit and the logit case, using the results in \fref{lem:scalarlipschitz}. 
We work out in detail the analysis of $(\text{RR}_1^+)$ for $\vecw_i$, i.e., the transpose of the $i^\text{th}$ row of $\bW$. The proofs for the remaining subproblems for other rows of $\bW$ and all columns of $\bC$ follow analogously. 

\begin{lem}[Lipschitz constants] \label{lem:vectorlips}
For a given $i$ and $j$, let
\begin{align*}
 f_w(\vecw_i) \!& =\! - \sum_j \!\log \, p(Y_{i,j}|\vecw_i, \vecc_j)\! + \frac{\mu}{2} \normtwo{\vecw_i}^2 = - \!\sum_j \!\log \, \Phi((2\, Y_{i,j}-1) \vecw_i^T \vecc_j) + \frac{\mu}{2} \normtwo{\vecw_i}^2, \\
 f_c(\vecc_j) \!& =\! - \sum_i \!\log \, p(Y_{i,j}|\vecw_i, \vecc_j)\! = - \!\sum_i \!\log \, \Phi((2\, Y_{i,j}-1) \vecw_i^T \vecc_j), 
\end{align*} 
where $Y_{i,j}$, $\vecw_i$, and $\vecc_j$ are defined as in \fref{sec:models}. Here, $\Phi(x)$ designates the inverse link function, which can either be \fref{eq:probitlink} or \fref{eq:logitlink}.
Then, for any $\vecx, \vecy \in \mathbb{R}^K$, we have
\begin{align*}
\| \nabla f_w(\vecx)- \nabla f_w(\vecy)\|_2 &\leq (L \sigma_\text{max}^2(\bC) +\mu) \|\vecx - \vecy\|_2, \\
\| \nabla f_c(\vecx)- \nabla f_c(\vecy)\|_2 &\leq L \sigma_\text{max}^2(\bW) \|\vecx - \vecy\|_2,
\end{align*}
where $L = L_\text{pro}=1$ and $L = L_\text{log} = {1}/{4}$ are the scalar Lipschitz constants for the probit and logit cases from \fref{lem:scalarlipschitz}, respectively. 
\end{lem}

\begin{proof}
For the sake of brevity, we only show the proof for $f_w(\vecx)$ in the probit case. The logit cases and the cases for $f_c(\vecx)$ follow analogously.  In what follows, the PDF $\mathcal{N}(x)$ and CDF $\Phi(x)$ of the standard normal density (the inverse probit link function) defined in \fref{eq:probitlink} are assumed to operate element-wise on the vector $\vecx \in \mathbb{R}^K$.

In order to simplify the derivation of the proof, we define the following effective matrix associated to $\vecw_i$ as
 \begin{align*}
\bC_{\text{eff},i} = [\,(2\, Y_{i,1}-1) \vecc_1, \ldots, (2\, Y_{i,N}-1)\vecc_N\,],
\end{align*}
which is equivalent to a right-multiplication $\bC=[\bmc_1, \ldots, \bmc_N]$ with a diagonal matrix containing the binary-valued response variables $(2\, Y_{i,j}-1)\in\{-1,+1\}\;\forall j$.
We can now establish an upper bound of the $\elltwo$-norm of the difference between the gradients at two arbitrary points~$\vecx$ and $\vecy$ as follows:
\begin{align}
\| \nabla f_w(\vecx)- \nabla f_w(\vecy)\|_2 &= \left\| \bC_{\text{eff},i} \frac{\mathcal{N}(\bC_{\text{eff},i}^T \vecx)}{\Phi(\bC_{\text{eff},i}^T \vecx)} - \bC_{\text{eff},i} \frac{\mathcal{N}(\bC_{\text{eff},i}^T \vecy)}{\Phi(\bC_{\text{eff},i}^T \vecy)} + \mu \vecx - \mu \vecy \right\|_2 \notag \\
& \leq \sigma_\text{max}(\bC_{\text{eff},i}) \left\|  \frac{\mathcal{N}(\bC_{\text{eff},i}^T \vecx)}{\Phi(\bC_{\text{eff},i}^T \vecx)} -  \frac{\mathcal{N}(\bC_{\text{eff},i}^T \vecy)}{\Phi(\bC_{\text{eff},i}^T \vecy)} \right \|_2 \label{eq:bound1} + \mu\|     \vecx -  \vecy  \|_2 \\
& \leq L \sigma_\text{max}(\bC_{\text{eff},i}) \|  \bC_{\text{eff},i}^T \vecx -  \bC_{\text{eff},i}^T \vecy\|_2 + \mu \|   \vecx -  \vecy\|_2 \label{eq:bound2} \\
& \leq L  \sigma_\text{max}^2(\bC_{\text{eff},i}) \|   \vecx -  \vecy\|_2 + \mu \|   \vecx -  \vecy\|_2 \label{eq:bound3}  \\
&= (L \sigma_\text{max}^2(\bC) + \mu) \|   \vecx -  \vecy\|_2. \label{eq:bound4} 
\end{align}
Here, \fref{eq:bound1} uses the triangle inequality and the Rayleigh-Ritz theorem of \cite{hornjohnson}, where $\sigma_\text{max}(\bC_{\text{eff},i})$ denotes the principal singular value of $\bC_{\text{eff},i}$. The bound  \fref{eq:bound2} follows from \fref{lem:scalarlipschitz}, and \fref{eq:bound3} is, once more, a consequence of the Rayleigh-Ritz theorem. 
The final equality \fref{eq:bound4} follows from the fact that flipping the signs of the columns of a matrix (as we did to arrive at $\bC_{\text{eff},i}$) does not affect its singular values, which concludes the proof. Note that the proof for $f_c(\cdot)$ follows by omitting $\mu$ and substitute $\bC$ by $\bW$ in \fref{eq:bound4}.
\end{proof}

Note that in all of the above proofs we only considered the case where the observation matrix $\bY$ is fully populated. Our proofs easily adapt to the case of missing entries in $\bY$, by replacing the matrix $\bC$ to $\bC_\mathcal{I}$, where $\bC_\mathcal{I}$ corresponds to the matrix containing the columns of $\bC$ corresponding to the observed entries indexed by the set $\mathcal{I} = \{j: (i,j) \in \Omega_\text{obs} \}$.  We omit the details for the sake of brevity.

\section{Proof of \fref{thm:sparfamconv}}
\label{app:globproof}

\sloppy
Minimizing $F(\vecx)$ as defined in \fref{thm:sparfamconv} using SPARFA-M corresponds to a multi-block coordinate descent problem, where the subproblems $(\text{RR})_1^+$ and $(\text{RR})_2$ correspond to \cite[Problem.~1.2b and 1.2a]{wotao}, respectively. Hence, we can use the results of  \cite[Lemma~2.6, Corrollary~2.7, and Theorem~2.8]{wotao} to establish global convergence of SPARFA-M.
To this end, we must verify that the problem (P) satisfies all of the assumptions in \cite[Assumption~1, Assumption~2, and Lemma~2.6]{wotao}. 
\fussy

\subsection{Prerequisites}

We first show that the smooth part of the cost function in $(\text{P})$, i.e., the negative log-likelihood plus both $\elltwo$-norm regularization terms, is Lipschitz continuous on any bounded set in $\mathbb{R}^{(N+Q)K}$. Then, we show that the probit log-likelihood function is real analytic. Note that the logit log-likelihood function is real analytic as shown in \cite[Section~2.3]{wotao}. Finally, we combine both results to prove \fref{thm:sparfamconv}, which establishes the global convergence of SPARFA-M.
\begin{lem}[Lipschitz continuity] \label{lem:biglips}
Define $\vecx = [\vecw_1^T, \ldots, \vecw_Q^T,\,\,  \vecc_1^T, \ldots, \vecc_N^T\,]^T$, and let
\begin{align*}
f(\vecx) =  -\! \sum_{(i,j)\in\Omega_\text{obs}} \! \log p(Y_{i,j}|\vecw_i , \vecc_j) + \frac{\mu}{2} \sum_{i}  \| \vecw_i \|_2^2 + \frac{\gamma}{2}  \sum_{j} \| \vecc_j \|_2^2.
\end{align*} 
Then, $f(\vecx)$ is Lipschitz continuous on any bounded set $\mathcal{D}= \{ \vecx: \normtwo{\vecx} \leq D \}$.
\end{lem}

\begin{proof}
Let $\vecy, \vecz \in \mathcal{D}$, recall the notation of \fref{lem:vectorlips}, and let $\vecw_i^y$, $\vecw_i^z$, $\vecc_j^y$, and~$\vecc_j^z$ denote the blocks of variables $\vecw_i$ and $\vecc_j$ in $\vecy$ and $\vecz$, respectively. We now have
\begin{align}
\normtwo{\nabla f(\vecy) \! - \! \nabla f(\vecz)} & = \!  \Big(  \!\sum_{i,j} \! \big( \left( \nabla f_w(\vecw_i^y) \! - \! \nabla f_w(\vecw_i^z) \right)^2 \! + \! ( \nabla f_c(\vecc_j^y) \! - \! \nabla f_c(\vecc_j^z) )^2 + \gamma^2 \|\vecc_j^y \! - \! \vecc_j^z\|^2_2 \big) \! \Big)^{\frac{1}{2}}  \notag \\
& \leq \Big( \sum_{i,j} \big( \! \left( \! L \sigma_\text{max}^2(\bC) \! + \! \mu \right)^2 \normtwo{\vecw_i^y \!  - \! \vecw_i^z}^2 + \left( \! L^2 \sigma_\text{max}^4(\bW) \! + \! \gamma^2 \right) \|\vecc_j^y \! - \! \vecc_j^z\|^2_2  \big)  \Big)^{\frac{1}{2}} \label{eq:fromlem4} \\
& \leq \left( L \left( \|\bW\|_F^2 + \|\bC\|_F^2 \right) + \text{max} \{\mu, \gamma\} \right) \normtwo{\vecy - \vecz} \label{eq:nucfro} \\ 
& \leq \left( L D^2 + \text{max}\{\mu, \gamma\} \right) \|\vecy - \vecz\|_2, \notag 
\end{align} 
where \fref{eq:fromlem4} follows from \fref{lem:vectorlips}, and \fref{eq:nucfro} follows from the fact that the maximum singular value of a matrix is no greater than its Frobenius norm (\cite{hornjohnson}), which is bounded by $D$ for $\vecx, \vecy \in \mathcal{D}$.  We furthermore have $L = 1$ for the probit case and $L = {1}/{4}$ for the logit case, shown in \fref{lem:scalarlipschitz}.
Thus, $f(\vecx)$ is Lipschitz continuous on any bounded set.
\end{proof}
\begin{lem}[Real analyticity] \label{lem:analytic}
Define 
$\vecx = [\vecw_1^T, \ldots, \vecw_Q^T, \, \vecc_1^T, \ldots, \vecc_N^T\,]^T$, and let 
\begin{align*}
g(\vecx) &=  - \sum_{(i,j)\in\Omega_\text{obs}} \log p(Y_{i,j}|\vecw_i , \vecc_j) =  - \sum_{(i,j)\in\Omega_\text{obs}} \log \Phi((2 Y_{i,j}-1) \vecw_i^T \vecc_j),
\end{align*} 
where $\Phi_\text{pro}(\cdot)$ is the inverse probit link function defined in \fref{eq:probitlink}. Then, $g(\vecx)$ is real analytic.
\end{lem}

\begin{proof}
We first show that the scalar negative log-likelihood function for the inverse probit link function is real analytic. 
To this end, recall the important property established by \cite{ra} that compositions of real analytic functions are real analytic. Therefore, the standard normal density $\mathcal{N}(x)$ is real analytic, since the exponential function and $x^2$ are both real analytical functions. 
Consequently, let $\mathcal{N}^{(k)}(x)$ denote the $k^\text{th}$ derivative of $\mathcal{N}(x)$, then $\left( \frac{\mathcal{N}^{(k)}(x)}{k!} \right)^\frac{1}{k}$ is bounded for all $k$, according to the definition of real analytic functions.

Now we show that $\Phi(x) = \int_{- \infty}^x \mathcal{N}(t) \text{d}t$ is also real analytic. Its $k^\text{th}$ derivative is given by $\Phi^{(k)}(x) = \mathcal{N}^{(k-1)}(x)$, and therefore $\left( \frac{\Phi^{(k)}(x)}{k!} \right)^\frac{1}{k}$ is obviously bounded for all $k$, since $\left( \frac{\mathcal{N}^{(k)}(x)}{k!} \right)^\frac{1}{k}$ is bounded for all $k$ as we have just shown. Thus, $\Phi(x)$ is real analytic. 

Given that $\Phi(x)$ is real-analytic, it follows that the negative log-probit-likelihood $- \text{log} \Phi(x)$ is real analytic, since both the logarithm function and the inverse probit link function are real analytic. 
Finally, extending the proof from scalar functions to vector functions preserves analyticity according to \cite[Section~2.2]{wotao}.
\end{proof}

\subsection{Proof of \fref{thm:sparfamconv}}
\label{sec:proof2}

We are finally armed to  prove \fref{thm:sparfamconv}.
We begin by showing that our problem (P) meets \cite[Assumptions~1 and 2]{wotao}. Then, we show that (P) meets all the additional assumptions needed for the convergence results in \cite[Lemma~2.6]{wotao}, through which we can establish convergence of the sequence  $\{ \vecx^t \}$  from certain starting points to some finite limit point. Finally, we use \cite[Theorem~2.8]{wotao} to show global convergence of SPARFA-M from any starting point. %

\subsubsection{Assumption 1}

We start by showing that (P) meets \cite[Assumption~1]{wotao}. Since every term in our objective function in $(\text{P})$ is non-negative, we have $F(\vecx) > -\infty$.
It is easy to verify that (P) is also block multi-convex in the variable $\vecx$, with the rows of $\bW$ and columns of~$\bC$ forming the blocks. Consequently, the problem $(\text{P})$ has at least one critical point, since $F(\vecx)$ is lower bounded by $0$. Therefore, Assumption~$1$ is met.

\subsubsection{Assumption 2}

Problem (P) also meets \cite[Assumption~2]{wotao} regarding the strong convexity of the individual subproblems. 
Due to the presence of the quadratic terms $\frac{\mu}{2}\normtwo{\vecw_i }^2$ and $\frac{\gamma}{2}\normtwo{\vecc_j}^2$, the smooth part of the objective functions of the individual subproblems $(\text{RR}_1^+)$ and $(\text{RR}_2)$ are strongly convex with parameters $\mu$ and $\gamma$, respectively.
Consequently, Assumption~2 is satisfied.

\subsubsection{Assumptions in \cite[Lem.~2.6]{wotao}}

Problem (P) also meets the assumptions in \cite[Lem.~2.6]{wotao} regarding the Lipschitz continuity of the subproblems and the Kurdyka-\L ojasiewicz inequality.
\fref{lem:biglips} shows that $f(\vecx) =  -\! \sum_{(i,j)\in\Omega_\text{obs}} \! \log p(Y_{i,j}|\vecw_i , \vecc_j) +  \frac{\mu}{2} \sum_{i}  \| \vecw_i \|_2^2 + \frac{\gamma}{2}  \sum_{j} \| \vecc_j \|_2^2$, satisfies the Lipschitz continuous requirement in \cite[Lemma~2.6]{wotao}.
As shown in \fref{lem:analytic} for the probit case and as shown in \cite[Section~2.2]{wotao} for the logit case, the negative log-likelihood term in $(\text{P})$ is real analytic, therefore also sub-analytic. 
All the regularizer functions in $F(\vecx)$ defined in \fref{thm:sparfamconv} are semi-algebraic and therefore sub-analytic, a consequence of \cite[Section~2.1]{semialg} and \cite[Section~2.2]{wotao}.
Using \cite[Theorems~1.1 and~1.2]{subsum}, the objective function $F(\vecx)$ is also sub-analytic, since all of its parts are sub-analytic and bounded below (non-negative), therefore satisfying the Kurdyka-\L{}ojasiewicz inequality at any point $\vecx$, as shown in \cite[Theorem~3.1]{semialg}. 
Finally, the SPARFA-M algorithm uses $\omega_i^{k-1} \equiv 0$ and $\ell = \text{min}\{\mu,\gamma\}$ where $\omega_i^{k-1}$ and $\ell$ as defined in \cite[Lemma~2.6]{wotao}.

Up to this point, we have shown that $(\text{P})$ satisfies all assumptions and requirements in \cite[Lemma~2.6]{wotao}. 
Now, SPARFA-M follows \cite[Lemma~2.6]{wotao} in the sense that, if $\vecx^0$ is sufficiently close to some critical point $\hat{\vecx}$ of $(\text{P})$, (more specifically, $\vecx^0 \in \mathcal{B}$ for some $\mathcal{B} \subset \mathcal{U}$ where $\mathcal{U}$ is a neighborhood of $\hat{\vecx}$ in which Kurdyka-\L{}ojasiewicz inequality holds), then $\{ \vecx^k \}$ converges to a point in $\mathcal{B}$. This establishes the convergence of SPARFA-M to a local minimum point from certain starting points.

\subsubsection{Global Convergence}

Finally, we can use \cite[Lemma~2.6]{wotao} to establish global convergence of \mbox{SPARFA-M}. It is obvious that the objective function $(\text{P})$ is bounded on any bounded set.
Hence, the sequence $\{ \vecx^k \}$ will always have a finite limit point and meet the assumptions in \cite[Theorem~2.8]{wotao}. The final statement of \fref{thm:sparfamconv} now directly follows from \cite[Theorem~2.8]{wotao}. 
Moreover, if the starting point is in close proximity to a global minimum, then SPARFA-M is guaranteed to converge to a global minimum.
This is a consequence of \cite[Corollary~2.7]{wotao}.

\section{Proof of Theorem \ref{thm:wpost}}
\label{app:sparfabd}

\begin{proof}
To prove Theorem \ref{thm:wpost}, we first define some notation. Let $\mathcal{N}(x | m,s) = \frac{1}{\sqrt{2 \pi s}} e^{-(x-m)^2/2s}$ define the normal PDF with mean $m$ and variance $s$. Furthermore, let $\textit{Exp}(m|\lambda) = \lambda e^{- \lambda m}$, $m \geq 0$ define the PDF of the exponential distribution with rate parameter $\lambda$. 

We are ultimately concerned with identifying whether the factor $W_{i,k}$ is active given our current beliefs about all other parameters in the model. Given the probit model, this is equivalent to determining whether or not an exponential random variable is present in Gaussian noise.  Let $x|m,s \sim \mathcal{N}(0|m,s)$ with $m \sim r \, \textit{Exp}(m|\lambda) + (1-r) \, \delta_0$ and $\delta_0$ the Dirac delta function located at $0$.
The posterior distribution $p(m = 0 | x)$ can be derived via Bayes' rule as follows:
\begin{align}
p(m = 0 | x) &= \frac{\mathcal{N}(x|m=0,s) (1-r)}{\mathcal{N}(x|m=0,s) (1-r) + r \int\mathcal{N}(x|m,s) \, \textit{Exp}(m|\lambda) \text{d}m}, \notag \\
 &=\frac{ \frac{\mathcal{N}(x|0,s)}{\int \mathcal{N}(x|m,s) \, \textit{Exp}(m|\lambda) dm} (1-r)}{ \frac{\mathcal{N}(x|0,s)}{\int \mathcal{N}(x|m,s) \, \textit{Exp}(m|\lambda) \text{d}m} (1-r) + r },  \notag \\
 &=\frac{ \frac{ \textit{Exp}(0|\lambda)\mathcal{N}(x|0,s)}{\int \mathcal{N}(x|m,s) \, \textit{Exp}(m|\lambda) \text{d}m} (1-r)}{ \frac{\textit{Exp}(0|\lambda)\mathcal{N}(x|0,s)}{\int \mathcal{N}(x|m,s) \, \textit{Exp}(m|\lambda) \text{d}m} (1-r) + r\textit{Exp}(0|\lambda)}. \label{eq:lastthing}
\end{align}
Here, it is important to recognize that $\frac{\textit{Exp}(m|\lambda) \mathcal{N}(x|m,s)}{\int \mathcal{N}(x|m,s) \, \textit{Exp}(m|\lambda) d m }$ denotes the posterior under the continuous portion of the prior (i.e. $m \neq 0$).  Since the exponential prior we have chosen is not conjugate to the normal likelihood, we must compute this distribution in closed form.  
To this end, let $\mathcal{N}^r(x|m,s,\lambda) \propto \mathcal{N}(x|m,s) \textit{Exp}(m|\lambda) = C_0 e^{-(x-m)^2/2s - \lambda m}$ denote a rectified normal distribution with normalization constant $C_0$.   Completing the square and carrying out the integration, we find $C_0 =  \frac{e^{\lambda m - \lambda^2 s / 2}}{\sqrt{2 \pi s} \Phi\bigl(\frac{m-\lambda s}{\sqrt{s}} \bigr)}$, which leads to 
\begin{align*}
\mathcal{N}^r(x|m,s,\lambda) =  \frac{e^{\lambda m - \lambda^2 s / 2}}{\sqrt{2 \pi s} \Phi\bigl(\frac{m-\lambda s}{\sqrt{s}} \bigr)} e^{-(x-m)^2/2s - \lambda m}.
\end{align*}
We can now rewrite \eqref{eq:lastthing} as
\begin{align*} 
p(m=0|x) = \frac{\frac{\mathcal{N}^r(0|\hat{m},\hat{s},\lambda)}{\textit{Exp}(0|\lambda)} (1-r)}{\frac{\mathcal{N}^r(0|\hat{m},\hat{s},\lambda)}{\textit{Exp}(0|\lambda)} (1-r) + r} 
\end{align*}
or, alternatively, as
\begin{eqnarray} \label{eq: selectprobalt}
\hat{r} = p(m \neq 0 | x) = 1-p(m = 0 | x) = \frac{\frac{\textit{Exp}(0|\lambda)}{\mathcal{N}^r(0|\hat{m},\hat{s},\lambda)}}{\frac{\textit{Exp}(0|\lambda)}{\mathcal{N}^r(0|\hat{m},\hat{s},\lambda)} + \frac{1-r}{r}}.
\end{eqnarray}

All that remains now is to determine $\hat{m}$ and $\hat{s}$ in \fref{eq: selectprobalt} for our full factor analysis scenario. Recall that our probabilistic model corresponds to \mbox{$\bZ = \bW \bC + \bM$}. Further recall our definition of the observation set $\Omega_{\text{obs}} = \{(i,j) : Y_{i,j} \;\; \text{is observed} \}$. We can now calculate the posterior on each coefficient $W_{i,k} \neq 0$ as follows: 
\begin{align}
p(W_{i,k}| \bZ, \bC, \bimud) \!&\propto p(W_{i,k}) \, p(\bZ | \bW_{-(i,k)}, \bC, \bimud) \nonumber \\
&\propto e^{-\lambda W_{i,k}} \, e^{-\frac{1}{2 \sigma^2} \sum_{\{j : (i,j) \in \Omega_{\text{obs}} \}} \left((Z_{i,j} - \mu_i) - \sum_{k = 1}^{K} W_{i,k} C_{k,j} \right)^2} \nonumber \\
&= e^{-\lambda W_{i,k}} \, e^{-\frac{1}{2 \sigma^2} \sum_{\{j : (i,j) \in \Omega_{\text{obs}} \}} \left((Z_{i,j} - \mu_i) - \sum_{k^\prime \neq k} W_{i,k^\prime} C_{k^\prime, j} - W_{i,k} C_{k,j} \right)^2} \nonumber\\
&\propto e^{-\lambda W_{i, k}} \, e^{-\frac{1}{2 \sigma^2} \sum_{\{j : (i,j) \in \Omega_{\text{obs}} \}} \left( W_{i, k}^2 C_{k,j}^2 - 2 ((Z_{i,j} - \mu_i)  - \sum_{k^\prime \neq k} W_{i,k^\prime} C_{k^\prime, j} )  W_{i,k} C_{k,j}\right)} \nonumber \\
&\propto  e^{-\lambda W_{i, k}} e^{\!-\frac{\sum_{\{j : (i,j) \in \Omega_{\text{obs}} \}} \!C_{k, j}^2}{2 \sigma^2} \!\Big( \!W_{i, k}  -   \frac{\sum_{\{j : (i,j) \in \Omega_{\text{obs}} \}}\! \big( \! (Z_{i,j} - \mu_i)  - \sum_{k^\prime \neq k} \!W_{i, k^\prime} C_{k^\prime, j} \!\big) \!C_{k, j}}{\sum_{\{j : (i,j) \in \Omega_{\text{obs}} \}} C_{k, j}^2}\! \Big)^2\!,  \label{eq:finalform}} 
\end{align}
where the last step is obtained by completing the square in $W_{i, k}$. 

The final result in \fref{eq:finalform} implies that $W_{i, k} \sim \mathcal{N}^\textit{r}(\hat{m},\hat{s},\lambda)$, where 
\begin{align*}
\hat{m} = \frac{\sum_{\{j : (i,j) \in \Omega_{\text{obs}} \}} \bigl( (Z_{i,j} - \mu_i)  - \sum_{k^\prime \neq k} W_{i, k^\prime} C_{k^\prime, j} \bigr) C_{k, j}}{\sum_{\{j : (i,j) \in \Omega_{\text{obs}} \}} C_{k, j}^2}
\end{align*} 
and $\hat{s} = \frac{\sigma^2}{\sum_{\{j : (i,j) \in \Omega_{\text{obs}} \}} C_{k, j}^2}$. Combining the results of \fref{eq: selectprobalt} and \fref{eq:finalform}, recognizing that $\sigma^2 = 1$ in the standard probit model, and adopting the notation $\widehat{R}_{i,k}$, $\widehat{M}_{i,k}$ and $\widehat{S}_{i,k}$ for the values of $\hat{r}$, $\hat{m}$ and $\hat{s}$ corresponding to each $\widehat{W}_{i,k}$, furnishes the final sampling result.
\end{proof}

\section*{Acknowledgments}

Thanks to Wotao Yin and Yangyang Xu for helpful discussions on the convergence proof of \mbox{SPARFA-M}, Genevera Allen for insights into probit regression, Marina Vannucci for helpful discussions regarding Bayesian factor analysis, Daniel Calder\'{o}n for organizing and administering the Amazon Mechanical Turk experiments in \fref{sec:realmturk}, and Carlos Monroy and Reid Whitaker for providing the STEMscopes data. We furthermore thank the anonymous reviewers for their valuable comments which improved the exposition of our results.

This work was supported by the National Science Foundation under Cyberlearning grant IIS-1124535, the Air Force Office of Scientific Research under grant FA9550-09-1-0432, the Google Faculty Research Award program.

Please see our website \url{www.sparfa.com}, where you can learn more about the project and purchase SPARFA t-shirts and other merchandise.

\bibliography{sparfaclustbib}

\begin{thebibliography}{86}
\providecommand{\natexlab}[1]{#1}
\providecommand{\url}[1]{\texttt{#1}}
\expandafter\ifx\csname urlstyle\endcsname\relax
  \providecommand{\doi}[1]{doi: #1}\else
  \providecommand{\doi}{doi: \begingroup \urlstyle{rm}\Url}\fi

\bibitem[Aharon et~al.(2005)Aharon, Elad, and Bruckstein]{nnksvd}
M.~Aharon, M.~Elad, and A.~M. Bruckstein.
\newblock K-{SVD} and its non-negative variant for dictionary design.
\newblock In \emph{Proc. SPIE Conf. on Wavelets}, volume 5914, pages 327--339,
  July 2005.

\bibitem[Aharon et~al.(2006)Aharon, Elad, and Bruckstein]{ksvd}
M.~Aharon, M.~Elad, and A.~M. Bruckstein.
\newblock {K-SVD}: An algorithm for designing overcomplete dictionaries for
  sparse representation.
\newblock \emph{IEEE Transactions on Signal Processing}, 54\penalty0
  (11):\penalty0 4311--4322, Dec. 2006.

\bibitem[{Amazon Mechanical Turk}(2012)]{mechturkwebsite}
{Amazon Mechanical Turk}, Sep. 2012.
\newblock URL \url{https://www.mturk.com/mturk/welcome}.

\bibitem[Bachrach et~al.(2012)Bachrach, Minka, Guiver, and Graepel]{minkak1}
Y.~Bachrach, T.~P. Minka, J.~Guiver, and T.~Graepel.
\newblock How to grade a test without knowing the answers -- a {B}ayesian
  graphical model for adaptive crowdsourcing and aptitude testing.
\newblock In \emph{Proc. 29th Intl. Conf. on Machine Learning}, pages
  1183--1190, June 2012.

\bibitem[Baker and Kim(2004)]{irtest}
F.~B. Baker and S.~H. Kim.
\newblock \emph{Item Response Theory: Parameter Estimation Techniques}.
\newblock Marcel Dekker Inc., 2nd edition, 2004.

\bibitem[Baker and Yacef(2009)]{edm2009}
R.~Baker and K.~Yacef.
\newblock The state of educational data mining in 2009: A review and future
  visions.
\newblock \emph{Journal of Educational Data Mining}, 1\penalty0 (1):\penalty0
  3--17, Oct. 2009.

\bibitem[Barnes(2005)]{qmatrix}
T.~Barnes.
\newblock The {Q}-matrix method: Mining student response data for knowledge.
\newblock In \emph{Proc. AAAI Workshop Educational Data Mining}, July 2005.

\bibitem[Beck and Teboulle(2009)]{fista}
A.~Beck and M.~Teboulle.
\newblock A fast iterative shrinkage-thresholding algorithm for linear inverse
  problems.
\newblock \emph{SIAM Journal on Imaging Science}, 2\penalty0 (1):\penalty0
  183--202, Mar. 2009.

\bibitem[Beheshti et~al.(2012)Beheshti, Desmarais, and Naceur]{predfact}
B.~Beheshti, M.~Desmarais, and R.~Naceur.
\newblock Methods to find the number of latent skills.
\newblock In \emph{Proc. 5th Intl. Conf. on Educational Data Mining}, pages
  81--86, June 2012.

\bibitem[Bergner et~al.(2012)Bergner, Droschler, Kortemeyer, Rayyan, Seaton,
  and Pritchard]{logitfa}
Y.~Bergner, S.~Droschler, G.~Kortemeyer, S.~Rayyan, D.~Seaton, and
  D.~Pritchard.
\newblock Model-based collaborative filtering analysis of student response
  data: Machine-learning item response theory.
\newblock In \emph{Proc. 5th Intl. Conf. on Educational Data Mining}, pages
  95--102, June 2012.

\bibitem[Bolte et~al.(2006)Bolte, Daniilidis, and Lewis]{semialg}
J.~Bolte, A.~Daniilidis, and A.~Lewis.
\newblock The {{\L}}ojasiewicz inequality for nonsmooth subanalytic functions
  with applications to subgradient dynamical systems.
\newblock \emph{SIAM Journal on Optimization}, 17\penalty0 (4):\penalty0
  1205--1223, Dec. 2006.

\bibitem[Boufounos and Baraniuk(2008)]{petros1bit}
P.~T. Boufounos and R.~G. Baraniuk.
\newblock 1-bit compressive sensing.
\newblock In \emph{Proc. Conf. on Information Science and Systems (CISS)}, Mar.
  2008.

\bibitem[Boyd and Vandenberghe(2004)]{boydbook}
S.~Boyd and L.~Vandenberghe.
\newblock \emph{Convex Optimization}.
\newblock Cambridge University Press, 2004.

\bibitem[Bruckstein et~al.(2008)Bruckstein, Elad, and Zibulevsky]{omp}
A.~M. Bruckstein, M.~Elad, and M.~Zibulevsky.
\newblock On the uniqueness of nonnegative sparse solutions to underdetermined
  systems of equations.
\newblock \emph{IEEE Transactions on Information Theory}, 54\penalty0
  (11):\penalty0 4813--4820, Nov. 2008.

\bibitem[Brusilovsky and Peylo(2003)]{pittcmu}
P.~Brusilovsky and C.~Peylo.
\newblock Adaptive and intelligent web-based educational systems.
\newblock \emph{Intl. Journal of Artificial Intelligence in Education},
  13\penalty0 (2-4):\penalty0 159--172, Apr. 2003.

\bibitem[Butz et~al.(2006)Butz, Hua, and Maguire]{butzregina}
C.~J. Butz, S.~Hua, and R.~B. Maguire.
\newblock A web-based {B}ayesian intelligent tutoring system for computer
  programming.
\newblock \emph{Web Intelligence and Agent Systems}, 4\penalty0 (1):\penalty0
  77--97, Nov. 2006.

\bibitem[Chang and Ying(2009)]{seqtest}
H.~Chang and Z.~Ying.
\newblock Nonlinear sequential designs for logistic item response theory models
  with applications to computerized adaptive tests.
\newblock \emph{The Annals of Statistics}, 37\penalty0 (3):\penalty0
  1466--1488, June 2009.

\bibitem[Chen et~al.(1998)Chen, Donoho, and Saunders]{bpdn}
S.~S. Chen, D.~L. Donoho, and M.~A. Saunders.
\newblock Atomic decomposition by basis pursuit.
\newblock \emph{SIAM Journal on Scientific Computing}, 20\penalty0
  (1):\penalty0 33--61, Mar. 1998.

\bibitem[Chu(1955)]{phibound}
J.~T. Chu.
\newblock On bounds for the normal integral.
\newblock \emph{IEEE Transactions on Signal Processing}, 42\penalty0
  (1/2):\penalty0 263--265, June 1955.

\bibitem[Desmarais(2011)]{qprobit}
M.~Desmarais.
\newblock Conditions for effectively deriving a {Q}-matrix from data with
  non-negative matrix factorization.
\newblock In \emph{Proc. 4th Intl. Conf. on Educational Data Mining}, pages
  41--50, July 2011.

\bibitem[{Dijksman} and {Khan}(2011)]{Khan}
J.~A. {Dijksman} and S.~{Khan}.
\newblock {Khan Academy: the world's free virtual school}.
\newblock In \emph{APS Meeting Abstracts}, page 14006, Mar. 2011.

\bibitem[{ELEC 301, Rice University}(2011)]{301website}
{ELEC 301, Rice University}.
\newblock Introduction to signals and systems, May 2011.
\newblock URL \url{http://dsp.rice.edu/courses/elec301}.

\bibitem[Fischer(2008)]{subsum}
A.~Fischer.
\newblock On sums of subanalytic functions.
\newblock \emph{Preprint}, 2008.

\bibitem[Fokoue(2004)]{fokoue2004stochastic}
E.~Fokoue.
\newblock Stochastic determination of the intrinsic structure in {Bayesian}
  factor analysis.
\newblock Technical report, Statistical and Applied Mathematical Sciences
  Institute, June 2004.

\bibitem[Fronczyk et~al.(2013, submitted)Fronczyk, Waters, Guindani, Baraniuk,
  and Vannucci]{nonparamsparfab}
K.~Fronczyk, A.~E. Waters, M.~Guindani, R.~G. Baraniuk, and M.~Vannucci.
\newblock A {B}ayesian infinite factor model for learning and content
  analytics.
\newblock \emph{Computational Statistics and Data Analysis}, June 2013,
  submitted.

\bibitem[Gonz{\'a}lez-Brenes and Mostow(2012)]{dynamickt}
J.~P. Gonz{\'a}lez-Brenes and J.~Mostow.
\newblock Dynamic cognitive tracing: Towards unified discovery of student and
  cognitive models.
\newblock In \emph{Proc. 5th Intl. Conf. on Educational Data Mining}, June
  2012.

\bibitem[Goodfellow et~al.(2012)Goodfellow, Courville, and Bengio]{bengioss}
I.~Goodfellow, A.~Courville, and Y.~Bengio.
\newblock Large-scale feature learning with spike-and-slab sparse coding.
\newblock In \emph{Proc. 29th Intl. Conf. on Machine Learning}, pages
  1439--1446, July 2012.

\bibitem[Guisan et~al.(2002)Guisan, Edwards~Jr, and Hastie]{responseglm}
A.~Guisan, T.~C. Edwards~Jr, and T.~Hastie.
\newblock Generalized linear and generalized additive models in studies of
  species distributions: setting the scene.
\newblock \emph{Ecological Modelling}, 157\penalty0 (2--3):\penalty0 89--100,
  Nov. 2002.

\bibitem[Hahn et~al.(2012)Hahn, Carvalho, and Scott]{hahn2010sparse}
P.~R. Hahn, C.~M. Carvalho, and J.~G. Scott.
\newblock A sparse factor-analytic probit model for congressional voting
  patterns.
\newblock \emph{Journal of the Royal Statistical Society}, 61\penalty0
  (4):\penalty0 619--635, Aug. 2012.

\bibitem[Harman(1976)]{factoran}
H.~H. Harman.
\newblock \emph{Modern Factor Analysis}.
\newblock The University of Chicago Press, 1976.

\bibitem[Hastie et~al.(2010)Hastie, Tibshirani, and Friedman]{tibsbook}
T.~Hastie, R.~Tibshirani, and J.~Friedman.
\newblock \emph{The Elements of Statistical Learning}.
\newblock Springer, 2010.

\bibitem[Herlocker et~al.(2004)Herlocker, Konstan, Terveen, and Riedl]{cfeval}
J.~L. Herlocker, J.~A. Konstan, L.~G. Terveen, and J.~T. Riedl.
\newblock Evaluating collaborative filtering recommender systems.
\newblock \emph{ACM Transactions on Information Systems}, 22\penalty0
  (1):\penalty0 5--53, Jan. 2004.

\bibitem[Horn and Johnson(1991)]{hornjohnson}
R.~A. Horn and C.~R. Johnson.
\newblock \emph{Topics in Matrix Analysis}.
\newblock Cambridge University Press, 1991.

\bibitem[Hu(2011)]{huweb}
D.~Hu.
\newblock How {K}han academy is using machine learning to assess student
  mastery.
\newblock \emph{online: {http://david-hu.com/}}, Nov. 2011.

\bibitem[Ishwaran and Rao(2005)]{ishwaran2005spike}
Hemant Ishwaran and J~Sunil Rao.
\newblock Spike and slab variable selection: frequentist and {Bayesian}
  strategies.
\newblock \emph{Annals of Statistics}, pages 730--773, Apr. 2005.

\bibitem[Jacques et~al.(2011, submitted)Jacques, Laska, Boufounos, and
  Baraniuk]{jason1bit}
L.~Jacques, J.~N. Laska, P.~T. Boufounos, and R.~G. Baraniuk.
\newblock Robust 1-bit compressive sensing via binary stable embeddings of
  sparse vectors.
\newblock Apr. 2011, submitted.

\bibitem[Knewton(2012)]{knewton}
Knewton.
\newblock Knewton adaptive learning: Building the world's most powerful
  recommendation engine for education.
\newblock \emph{online:
  {http://www.knewton.com/adaptive-learning-white-paper/}}, June 2012.

\bibitem[Koedinger et~al.(1997)Koedinger, Anderson, Hadley, and Mark]{rule}
K.~R. Koedinger, J.~R. Anderson, W.~H. Hadley, and M.~A. Mark.
\newblock Intelligent tutoring goes to school in the big city.
\newblock \emph{Intl. Journal of Artificial Intelligence in Education},
  8:\penalty0 30--43, 1997.

\bibitem[Koh et~al.(2007)Koh, Kim, and Boyd]{kohboyd}
K.~Koh, S.~Kim, and S.~Boyd.
\newblock An interior-point method for large-scale $\ellone$-regularized
  logistic regression.
\newblock \emph{Journal of Machine Learning Research}, 8:\penalty0 1519--1555,
  2007.

\bibitem[Koren and Sill(2011)]{ordpred}
Y.~Koren and J.~Sill.
\newblock {OrdRec}: an ordinal model for predicting personalized item rating
  distributions.
\newblock In \emph{Proc. of the 5th ACM Conf. on Recommender Systems}, pages
  117--124, Oct. 2011.

\bibitem[Koren et~al.(2009)Koren, Bell, and Volinsky]{svdplus}
Y.~Koren, R.~Bell, and C.~Volinsky.
\newblock Matrix factorization techniques for recommender systems.
\newblock \emph{Computer}, 42\penalty0 (8):\penalty0 30--37, Aug. 2009.

\bibitem[Krantz and Parks(2002)]{ra}
S.~G. Krantz and H.~R. Parks.
\newblock \emph{A Primer of Real Analytic Functions}.
\newblock Birkhauser, 2002.

\bibitem[Krudysz and McClellan(2011)]{gregk1}
G.~A. Krudysz and J.~H. McClellan.
\newblock Collaborative system for signal processing education.
\newblock In \emph{2011 IEEE Intl. Conf. on Acoustics, Speech and Signal
  Processing (ICASSP)}, pages 2904 --2907, May 2011.

\bibitem[Krudysz et~al.(2006)Krudysz, Li, and McClellan]{gregk2}
G.~A. Krudysz, J.~S. Li, and J.~H. McClellan.
\newblock Web-based {B}ayesian tutoring system.
\newblock In \emph{12th Digital Signal Processing Workshop - 4th Signal
  Processing Education Workshop}, pages 129 --134, Sep. 2006.

\bibitem[Lan et~al.(2013{\natexlab{a}})Lan, Studer, Waters, and
  Baraniuk]{sparfatag}
A.~S. Lan, C.~Studer, A.~E. Waters, and R.~G. Baraniuk.
\newblock Tag-aware ordinal sparse factor analysis for learning and content
  analytics.
\newblock In \emph{Proc. 6th Intl. Conf. on Educational Data Mining}, July
  2013{\natexlab{a}}.

\bibitem[Lan et~al.(2013{\natexlab{b}})Lan, Studer, Waters, and
  Baraniuk]{sparfatop}
A.~S. Lan, C.~Studer, A.~E. Waters, and R.~G. Baraniuk.
\newblock Joint topic modeling and factor analysis of textual information and
  graded response data.
\newblock In \emph{Proc. 6th Intl. Conf. on Educational Data Mining}, July
  2013{\natexlab{b}}.

\bibitem[Lee et~al.(2006)Lee, Lee, Abbeel, and Ng]{ng}
S.~Lee, H.~Lee, P.~Abbeel, and A.~Y. Ng.
\newblock Efficient $\ellone$ regularized logistic regression.
\newblock In \emph{Proc. National Conf. on Artificial Intelligence}, volume~21,
  pages 401--408, 2006.

\bibitem[Lee et~al.(2010)Lee, Huang, and Hu]{binopca}
S.~Lee, J.~Z. Huang, and J.~Hu.
\newblock Sparse logistic principal components analysis for binary data.
\newblock \emph{Annals of Applied Statistics}, 4\penalty0 (3):\penalty0
  1579--1601, Sept. 2010.

\bibitem[Li et~al.(2011)Li, Cohen, and Koedinger]{cohen}
N.~Li, W.~W. Cohen, and K.~R. Koedinger.
\newblock A machine learning approach for automatic student model discovery.
\newblock In \emph{Proc. 4th Intl. Conf. on Educational Data Mining}, pages
  31--40, July 2011.

\bibitem[Linacre(1999)]{raschest}
J.~M. Linacre.
\newblock Understanding {R}asch measurement: Estimation methods for {R}asch
  measures.
\newblock \emph{Journal of Outcome Measurement}, 3\penalty0 (4):\penalty0
  382--405, 1999.

\bibitem[Linden et~al.(2003)Linden, Smith, and York]{amaz}
G.~Linden, B.~Smith, and J.~York.
\newblock Amazon.com recommendations: Item-to-item collaborative filtering.
\newblock \emph{Internet Computing, IEEE}, 7\penalty0 (1):\penalty0 76--80,
  Jan. 2003.

\bibitem[Linden and Glas(2000)]{adaptivetest}
W.~J. V.~D. Linden and editors Glas, C. A.~W.
\newblock \emph{Computerized Adaptive Testing: Theory and Practice}.
\newblock Kluwer Academic Publishers, 2000.

\bibitem[Lord(1980)]{lordirt}
F.~M. Lord.
\newblock \emph{Applications of Item Response Theory to Practical Testing
  Problems}.
\newblock Erlbaum Associates, 1980.

\bibitem[Mairal et~al.(2010)Mairal, Bach, Ponce, and Sapiro]{bach}
J.~Mairal, F.~Bach, J.~Ponce, and G.~Sapiro.
\newblock Online learning for matrix factorization and sparse coding.
\newblock \emph{Journal of Machine Learning Research}, 11:\penalty0 19--60,
  2010.

\bibitem[McDonald(2000)]{mcdonald}
R.~P. McDonald.
\newblock A basis for multidimensional item response theory.
\newblock \emph{Applied Psychological Measurement}, 247\penalty0 (2):\penalty0
  99--114, June 2000.

\bibitem[Meng et~al.(2010)Meng, Zhang, Qi, Chen, and Huang]{meng2010uncovering}
J.~Meng, J.~Zhang, Y.~Qi, Y.~Chen, and Y.~Huang.
\newblock Uncovering transcriptional regulatory networks by sparse {Bayesian}
  factor model.
\newblock \emph{EURASIP Journal on Advances in Signal Processing},
  2010\penalty0 (3):\penalty0 1--18, Mar. 2010.

\bibitem[Minka(2003)]{minka}
T.~P. Minka.
\newblock A comparison of numerical optimizers for logistic regression.
\newblock Technical report, 2003.
\newblock http://citeseerx.ist.psu.edu/viewdoc/download?doi=10.1.1.85.7017
  \&rep=rep1\&type=pdf.

\bibitem[Mohamed et~al.(2012)Mohamed, Heller, and Ghahramani]{l1vsbayes}
S.~Mohamed, K.~Heller, and Z.~Ghahramani.
\newblock {Bayesian} and $\ell_1$ approaches for sparse unsupervised learning.
\newblock In \emph{Proc. 29th Intl. Conf. on Machine Learning}, pages 751--758,
  July 2012.

\bibitem[Murray et~al.(2004)Murray, VanLehn, and Mostow]{lookaheaddecision}
R.~C. Murray, K.~VanLehn, and J.~Mostow.
\newblock Looking ahead to select tutorial actions: A decision-theoretic
  approach.
\newblock \emph{Intl. Journal of Artificial Intelligence in Education},
  14\penalty0 (3--4):\penalty0 235--278, 2004.

\bibitem[Nesterov(2007)]{nest}
Y.~Nesterov.
\newblock Gradient methods for minimizing composite objective function.
\newblock Technical report, Universit\'e catholique de Louvain, Sep. 2007.

\bibitem[Norvick(1966)]{ctt}
M.~R. Norvick.
\newblock The axioms and principal results of classical test theory.
\newblock \emph{Journal of Mathematical Psychology}, 3\penalty0 (1):\penalty0
  1--18, Feb. 1966.

\bibitem[Olver(2010)]{nist}
F.~W.~J. Olver, editor.
\newblock \emph{NIST Handbook of Mathematical Functions}.
\newblock Cambridge University Press, 2010.

\bibitem[Pardos and Heffernan(2010)]{assistment}
Z.~A. Pardos and N.~T. Heffernan.
\newblock Modeling individualization in a {bayesian} networks implementation of
  knowledge tracing.
\newblock In \emph{User Modeling, Adaptation, and Personalization}, pages
  255--266. Springer, June 2010.

\bibitem[Park and Hastie(2008)]{parkhastie}
M.~Y. Park and T.~Hastie.
\newblock Penalized logistic regression for detecting gene interactions.
\newblock \emph{Biostatistics}, 9\penalty0 (1):\penalty0 30--50, Jan. 2008.

\bibitem[Plan and Vershynin(2012, submitted)]{plan1bit}
Y.~Plan and R.~Vershynin.
\newblock Robust 1-bit compressed sensing and sparse logistic regression: A
  convex programming approach.
\newblock Feb. 2012, submitted.

\bibitem[Pournara and Wernisch(2007)]{pournara2007factor}
I.~Pournara and L.~Wernisch.
\newblock Factor analysis for gene regulatory networks and transcription factor
  activity profiles.
\newblock \emph{BMC Bioinformatics}, 8\penalty0 (1):\penalty0 61, Feb. 2007.

\bibitem[Psotka et~al.(1988)Psotka, Massey, and Mutter]{itslesson}
J.~Psotka, L.~D. Massey, and editors Mutter, S.~A.
\newblock \emph{Intelligent Tutoring Systems: Lessons Learned}.
\newblock Lawrence Erlbaum Associates, 1988.

\bibitem[Rafferty et~al.(2011)Rafferty, Brunskill, Griffiths, and
  Shafto]{pomdplearn}
A.~N. Rafferty, E.~Brunskill, T.~L. Griffiths, and P.~Shafto.
\newblock Faster teaching by {POMDP} planning.
\newblock In \emph{Proc. 15th Intl. Conf. on Artificial Intelligence in
  Education}, pages 280--287, June 2011.

\bibitem[Rasch(1993)]{rasch}
G.~Rasch.
\newblock \emph{Probabilistic Models for Some Intelligence and Attainment
  Tests}.
\newblock MESA Press, 1993.

\bibitem[Rasmussen and Williams(2006)]{gpml}
C.~E. Rasmussen and C.~K.~I. Williams.
\newblock \emph{Gaussian Process for Machine Learning}.
\newblock MIT Press, 2006.

\bibitem[Reckase(2009)]{mirt}
M.~D. Reckase.
\newblock \emph{Multidimensional Item Response Theory}.
\newblock Springer Publishing Company, Incorporated, 1st edition, 2009.

\bibitem[Romero and Ventura(2007)]{edm}
C.~Romero and S.~Ventura.
\newblock Educational data mining: A survey from 1995 to 2005.
\newblock \emph{Expert Systems with Applications}, 33\penalty0 (1):\penalty0
  135--146, July 2007.

\bibitem[Schmidt et~al.(2009)Schmidt, Winther, and Hansen]{schmidt2009bayesian}
M.~N. Schmidt, O.~Winther, and L.~K. Hansen.
\newblock {Bayesian} non-negative matrix factorization.
\newblock In \emph{Independent Component Analysis and Signal Separation},
  volume 5441, pages 540--547, Mar. 2009.

\bibitem[Stamper et~al.(2007)Stamper, Barnes, and Croy]{stampersm}
J.~C. Stamper, T.~Barnes, and M.~Croy.
\newblock Extracting student models for intelligent tutoring systems.
\newblock In \emph{Proc. National Conf. on Artificial Intelligence}, volume~22,
  pages 113--147, July 2007.

\bibitem[STEMscopes(2012)]{stemwebsite}
STEMscopes.
\newblock {STEM}scopes science education, Sep. 2012.
\newblock URL \url{http://stemscopes.com/}.

\bibitem[Studer and Baraniuk(2012)]{studi}
C.~Studer and R.~G. Baraniuk.
\newblock Dictionary learning from sparsely corrupted or compressed signals.
\newblock In \emph{IEEE Intl. Conf. on Acoustics, Speech and Signal Processing
  (ICASSP)}, pages 3341--3344, Mar. 2012.

\bibitem[Thai-Nghe et~al.(2011{\natexlab{a}})Thai-Nghe, Drumond, Horvath, and
  Schmidt-Thieme]{viet}
N.~Thai-Nghe, L.~Drumond, T.~Horvath, and L.~Schmidt-Thieme.
\newblock Multi-relational factorization models for predicting student
  performance.
\newblock \emph{KDD Workshop on Knowledge Discovery in Educational Data
  (KDDinED)}, Aug. 2011{\natexlab{a}}.

\bibitem[Thai-Nghe et~al.(2011{\natexlab{b}})Thai-Nghe, Horvath, and
  Schmidt-Thieme]{viettwo}
N.~Thai-Nghe, T.~Horvath, and L.~Schmidt-Thieme.
\newblock Factorization models for forecasting student performance.
\newblock In \emph{Proc. 4th Intl. Conf. on Educational Data Mining}, pages
  11--20, July 2011{\natexlab{b}}.

\bibitem[Thompson(2009)]{abilityest}
N.~A. Thompson.
\newblock Item selection in computerized classification testing.
\newblock \emph{Educational and Psychological Measurement}, 69\penalty0
  (5):\penalty0 778--793, Oct. 2009.

\bibitem[Tipping(2001)]{tipping2001sparse}
M.~E. Tipping.
\newblock Sparse {B}ayesian learning and the relevance vector machine.
\newblock \emph{Journal of Machine Learning Research}, 1:\penalty0 211--244,
  2001.

\bibitem[VanLehn et~al.(2005)VanLehn, Lynch, Schulze, Shapiro, Shelby, Taylor,
  Treacy, Weinstein, and Wintersgill]{andes2005}
K.~VanLehn, C.~Lynch, K.~Schulze, J.~A. Shapiro, R.~Shelby, L.~Taylor,
  D.~Treacy, A.~Weinstein, and M.~Wintersgill.
\newblock The {A}ndes physics tutoring system: Lessons learned.
\newblock \emph{Intl. Journal of Artificial Intelligence in Education},
  15\penalty0 (3):\penalty0 147--204, 2005.

\bibitem[Vats et~al.(2013)Vats, Studer, Lan, Carin, and Baraniuk]{tesr}
D.~Vats, C.~Studer, A.~S. Lan, L.~Carin, and R.~G. Baraniuk.
\newblock Test size reduction for concept estimation.
\newblock In \emph{Proc. 6th Intl. Conf. on Educational Data Mining}, July
  2013.

\bibitem[West(2003)]{west2003bayesian}
M.~West.
\newblock {Bayesian} factor regression models in the ``large $p$, small $n$''
  paradigm.
\newblock \emph{{Bayesian} Statistics}, 7\penalty0 (2003):\penalty0 723--732,
  Sep. 2003.

\bibitem[Woolf(2008)]{woolf08}
B.~P. Woolf.
\newblock \emph{Building Intelligent Interactive Tutors: Student-centered
  Strategies for Revolutionizing E-learning}.
\newblock Morgan Kaufman Publishers, 2008.

\bibitem[Xu and Yin(2012)]{wotao}
Y.~Xu and W.~Yin.
\newblock A block coordinate descent method for multi-convex optimization with
  applications to nonnegative tensor factorization and completion.
\newblock Technical report, Rice University CAAM, Sep. 2012.

\bibitem[Yao(2003)]{yao}
L.~Yao.
\newblock \emph{BMIRT: {Bayesian} Multivariate Item Response Theory}.
\newblock CTb/McGraw-Hill, 2003.

\end{thebibliography}

\end{document}